\documentclass{article}

\usepackage{booktabs}
\usepackage{multirow}
\usepackage[table]{xcolor}
\usepackage{colortbl}
\usepackage{arydshln}

\usepackage{wrapfig}
\usepackage{amsmath}
\usepackage{amssymb}
\usepackage{mathtools}
\usepackage{amsthm}
\usepackage{algorithm}
\usepackage{algpseudocode}

\usepackage{graphicx}
\usepackage{subcaption}
\usepackage{caption}

\usepackage{xcolor}  
\usepackage[dvipsnames, svgnames, x11names]{xcolor}  
\definecolor{inkblue}{RGB}{0, 51, 102}  

\definecolor{identityrow}{RGB}{234, 239, 247}
\definecolor{saliencyrow}{RGB}{210, 222, 239}
\usepackage{float}

\usepackage[preprint]{style_2026}

\usepackage[utf8]{inputenc} 
\usepackage[T1]{fontenc}    
\usepackage{hyperref}       
\usepackage{cleveref}
\usepackage{url}            
\usepackage{amsfonts}       
\usepackage{nicefrac}       
\usepackage{microtype}      
\usepackage{xcolor}         

\theoremstyle{plain}
\newtheorem{theorem}{Theorem}[section]

\newtheorem{corollary}[theorem]{Corollary}
\theoremstyle{definition}
\newtheorem{definition}[theorem]{Definition}

\theoremstyle{remark}

\usepackage{float}

\newenvironment{talign*}
{\csname align*\endcsname}
{\endalign}
\newenvironment{talign}
{\align}
{\endalign}

\let\oldthanks\thanks
\renewcommand{\thanks}[1]{%
    \bgroup
    \hypersetup{pdfborder={0 0 0}, linkbordercolor=white}%
    \oldthanks{#1}%
    \egroup%
}

\title{Saliency-Aware Regularized Quantization Calibration for Large Language Models}

\author{
 Yanlong Zhao$^{1}$\thanks{Equal Contribution: \href{mailto:zhaoyanlong@mail.ustc.edu.cn}{Yanlong Zhao}, \href{mailto:ucesxc4@ucl.ac.uk}{Xiaoyuan Cheng} and \href{mailto:liuhuihang@mail.sufe.edu.cn}{Huihang Liu}.} \quad
 Xiaoyuan Cheng$^{2}$\footnotemark[1] \quad 
   \textbf{Huihang Liu}$^{3}$\footnotemark[1] \\
  \textbf{Baihua He}$^{1}$ \quad 
  \textbf{Xinyu Zhang}$^{1,4}$ \quad
  \textbf{Harrison Bo Hua Zhu}$^{5,6,7}$ \\
   \textbf{Wenlong Chen}$^{6}$ \quad
 \textbf{Li Zeng}$^{8}$ \quad
\textbf{Zhuo Sun}$^{3,6}$\thanks{Correspondence Author. Correspondence to \href{mailto:zhuosunreid@outlook.com}{Zhuo Sun}.}\\
$^{1}$University of Science and Technology of China, 
$^{2}$University College London,\\
$^{3}$Shanghai University of Finance and Economics\\
$^{4}$Academy of Mathematics and Systems Science, Chinese Academy of Sciences, \\
$^{5}$University of Copenhagen,
$^{6}$Imperial College London,\\
$^{7}$Technical University of Denmark
$^{8}$Peking University
}

\begin{document}

\maketitle

\begin{abstract}
Post-training quantization (PTQ) is an effective approach for deploying large language models (LLMs) under memory and latency constraints. Most existing PTQ methods determine quantization parameters by minimizing a layer-wise reconstruction error on a predetermined calibration dataset, typically optimized via either scale search or Gram-based methods. However, from the perspective of generalization risk, existing PTQ calibration objectives based solely on empirical reconstruction error over limited or unrepresentative calibration data may move the quantized weights away from the original floating-point weights, potentially degrading downstream performance. To address this issue, we propose \emph{Regularized Quantization Calibration} (RQC), a unified framework that augments standard PTQ objectives with a regularizer that explicitly controls weight deviation from the original weights. We further generalize this framework to incorporate a saliency-aware regularizer, resulting in \emph{Saliency-Aware Regularized Quantization Calibration} (SARQC). The proposed regularization encourages quantized weights to remain close to the original weights during calibration, leading to improved generalization at inference time. SARQC integrates seamlessly into existing PTQ pipelines and enhances both scale-search-based and Gram-based methods under a unified formulation. Extensive experiments on dense and Mixture-of-Experts LLMs demonstrate consistent improvements in perplexity and zero-shot accuracy, without introducing additional inference overhead.
\end{abstract}

\footnotetext{Project Page: \url{https://github.com/Riceormice/SARQC}.}

\section{Introduction}

Large language models (LLMs) have scaled to tens of billions of parameters and beyond, delivering strong capabilities across tasks ranging from instruction following and knowledge-intensive QA to code generation and multi-step reasoning \citep{gpt4,jiang2023mistral7b,llama3,qwen3}. However, deploying these models remains costly. Their enormous parameter size leads to high memory usage, and autoregressive decoding is often limited by memory bandwidth.

To deploy LLMs efficiently, post-training quantization (PTQ) has become a widely adopted technique for memory-constrained scenarios. Existing PTQ algorithms aim to convert a floating-point (FP) LLM into a quantized model \citep{banner2019post,gholami2021survey}. Such a compressed model typically stores its weights in low-bit integer formats (e.g., INT4). Most PTQ pipelines determine quantization parameters, such as scales, zero points, clipping thresholds, and rounding decisions, by minimizing layer- or block-level output reconstruction error between the FP model and the quantized model over a small calibration dataset \citep{nagel2020adaround,li2021brecq,frantar2022optimal}. In contrast to quantization-aware training (QAT), which typically incurs substantial data and training cost \citep{liu2023llmqatdatafreequantizationaware}, post-training quantization provides a more efficient alternative by quantizing LLMs using only a small calibration set and minimal computation \citep{wei2022outlier,xiao2023smoothquant,frantar2023gptq,lin2024awq,tian2025pocketllm}. In this work, we mainly focus on weight-only PTQ techniques, which quantize FP weights to low-bit formats while keeping activations in floating-point precision. These methods significantly reduce memory usage and improve throughput in memory-bound serving scenarios \citep{lin2024awq,paroquant}.

\begin{figure*}[t]
    \centering
    \begin{subfigure}[t]{0.6\textwidth}
        \centering
        \includegraphics[width=\linewidth]{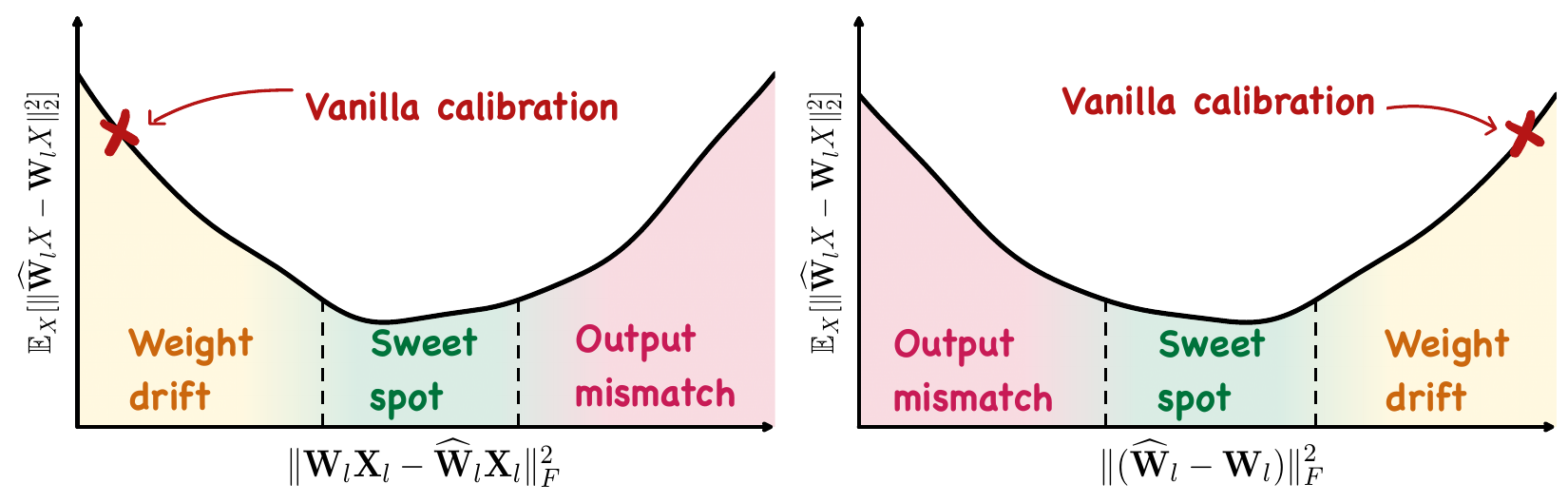}
        \caption{}
        \label{fig:row_a}
    \end{subfigure}\hfill
    \begin{subfigure}[t]{0.39\textwidth}
        \centering
        \includegraphics[width=\linewidth]{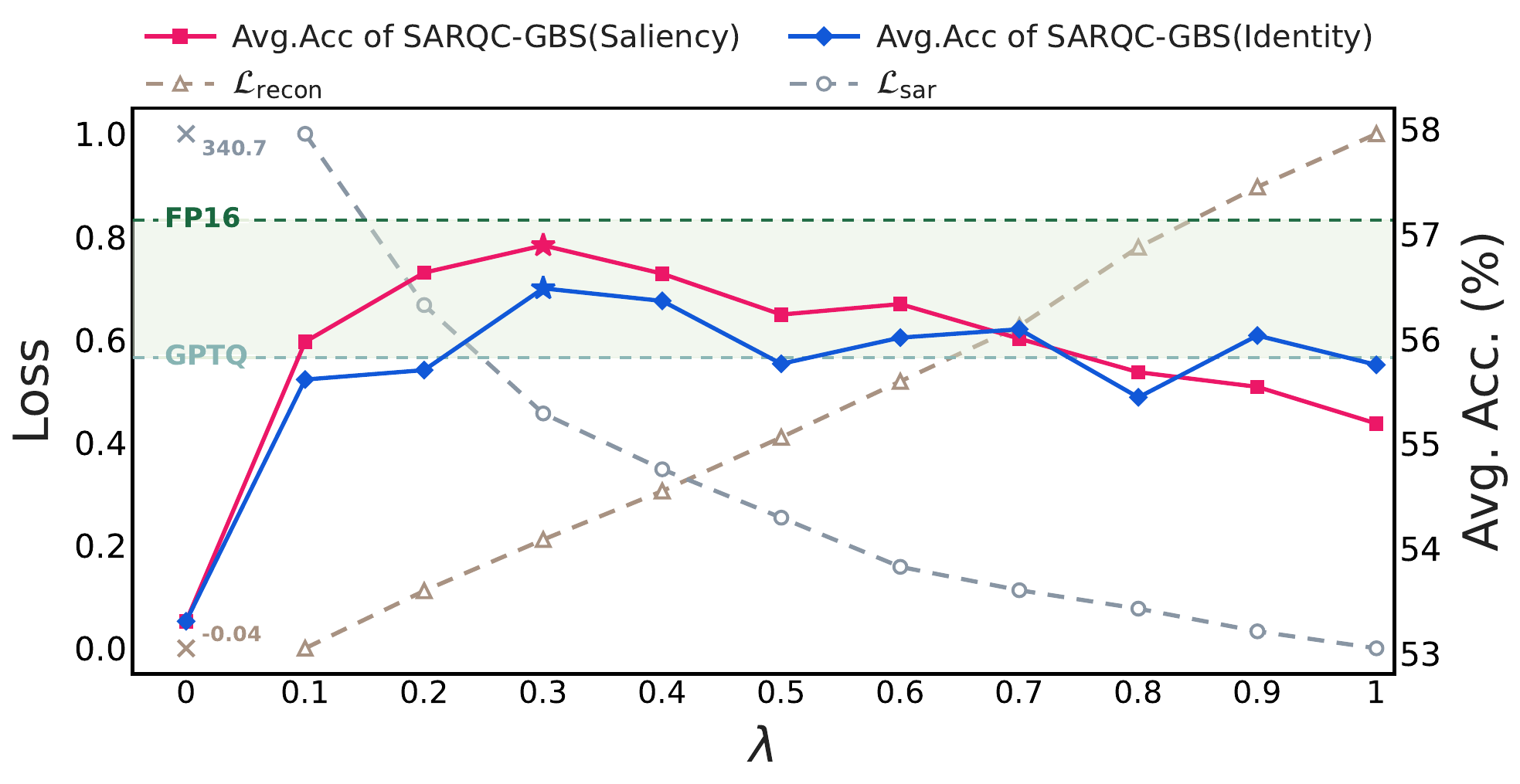}
        \caption{}
        \label{fig:row_b}
    \end{subfigure}
    \caption{Illustration and validation of our motivation.
\emph{(a) Conceptual illustration}: A smaller
$\mathbb{E}_{X}[\|\widehat{\mathbf{W}}_lX-\mathbf{W}_lX\|_2^2]$
generally implies better downstream performance. Vanilla calibration minimizes only the reconstruction loss, which can induce \emph{weight drift} and degrade performance. The optimal solution lies in a \emph{sweet spot} that balances \emph{output mismatch} and \emph{weight drift}.
\emph{(b) Empirical evidence of our motivation}: For \texttt{SARQC-GBS} on LLaMA2-7B under W4A16, the best downstream accuracy is achieved with a moderate regularization toward the original FP weights. The average accuracy over the eight evaluation tasks described in \Cref{sec:exp setup} is reported here. $\mathcal{L}_{\mathrm{recon}}$ and $\mathcal{L}_{\mathrm{sar}}$ are the reconstruction error and saliency-aware regularizer defined in \Cref{eq:rqc_layer}. See \Cref{fig:3d_weight_diff} (in Appendix~\ref{app:3d_weight_diff}) for the visualization of weight drift. See Appendix~\ref{appx:detail_of_figure1} for more details.}
    \label{fig:illustration}
\vspace{-18pt}
\end{figure*}

\paragraph{Motivation} Recent weight-only PTQ methods for LLMs select quantized weights by minimizing layer-wise reconstruction error on a predetermined, often small, calibration set \citep{lin2024awq,frantar2023gptq,li2025gptqv2}. Specifically, for a linear layer with input activations $\mathbf{X}_l$ and FP weights $\mathbf{W}_l$, a typical PTQ calibration process selects quantized weights $\widehat{\mathbf{W}}_l$ by minimizing $\|\mathbf{W}_l \mathbf{X}_l- \widehat{\mathbf{W}}_l \mathbf{X}_l \|_\mathrm{F}^2$. However, reconstruction-based PTQ calibration does not explicitly constrain the dequantized weights to remain close to the original FP weights, which can lead to undesired weight drift. During inference, weight-only quantized LLMs still rely on dequantized weights to interact with floating-point activations, so substantial deviation from the original FP weights may degrade downstream performance. As illustrated in \Cref{fig:illustration} and \Cref{fig:3d_weight_diff} (in Appendix~\ref{app:3d_weight_diff}), calibration that solely minimizes reconstruction error can induce undesired weight drift even when the calibration loss is small. In particular, a smaller reconstruction error does not guarantee a smaller weight drift $\|\mathbf{W}_l-\widehat{\mathbf{W}}_l\|_\mathrm{F}^2$; in fact, the two objectives can conflict. As formalized later in \Cref{thm:generalization risk}, this may place the model in a regime with an enlarged generalization risk. By introducing an explicit discrepancy regularizer, this trade-off becomes more explicit and controllable. As the regularization strength increases, weight drift decreases while reconstruction error rises, and the best downstream performance is achieved at an intermediate point along this trade-off curve, as shown in \Cref{fig:illustration}, which is formalized in Corollary~\ref{cor:constraint_penalty_supported} later.

\paragraph{Our Approach}
Motivated by this observation, we propose \emph{Regularized Quantization Calibration}, a general calibration framework for post-training quantization that augments the standard layer-wise reconstruction objective with an explicit regularizer on weight drift. The original FP model serves as a natural reference point, and the added regularizer constrains the quantized solution from drifting too far away from it. This leads to a more balanced calibration objective that better preserves the behavior of the original FP model and is more robust to limited or unrepresentative calibration data. It is then further extended to \emph{Saliency-Aware Regularized Quantization Calibration} (SARQC). Importantly, SARQC can be seamlessly integrated into existing PTQ algorithms and applies to the two dominant PTQ paradigms, namely grid-search methods \citep{xiao2023smoothquant,lin2024awq} and Gram-based methods \citep{frantar2023gptq,li2025gptqv2}.

\paragraph{Main Contributions} The main contributions of this work are as follows: (1) We find that controlling the deviation to the original weights of LLMs is critical to weight-only post-training quantized LLMs, and provide both empirical evidence and theoretical analysis from the perspective of generalization risk and constrained optimization. (2) Motivated by these findings, we propose a regularized quantization calibration framework for weight-only post-training quantization of LLMs that explicitly controls weight drift from the original floating-point (FP) weights. We further extend the framework with a saliency-aware regularization term. (3) Through extensive experiments, we demonstrate that the proposed framework is broadly applicable across various PTQ paradigms and consistently achieves superior performance across a wide range of LLM families and model sizes.

\section{Background}
\label{sec:background}

In this section, we briefly review quantization and post-training quantization for LLMs. See Appendix~\ref{appendix: More Comprehensive Literature Review} and Appendix~\ref{app:technical_def} for more preliminaries.

\vspace{-6pt}
\paragraph{Quantization and Dequantization}
Quantization maps high-precision floating-point values (e.g., BF16/FP16) to discrete integer values (e.g., INT2/INT4), thereby reducing memory cost and improving throughput in memory-bound scenarios. For weight-only quantization considered in this work, a floating-point weight value $w$ is quantized as follows:
\begin{talign}
\label{eq:Q operation}
    w_{\text{INT-N}} = \mathrm{round}\!\left(\frac{w_{\text{FP16}}}{\eta}\right), \quad\quad \eta = \frac{\max |w|}{2^{N-1}-1} ~~,
\end{talign}
where $N$ is the number of bits (e.g., $N=4$ for INT4) and $\eta$ is the quantization step size. The corresponding dequantization process reconstructs a floating-point approximation of the quantized value as $\widehat{w} = \eta \cdot w_{\text{INT-N}}$. Such dequantization operations are performed during inference for weight-only quantized LLMs. To simplify notation, we assume a symmetric quantization scheme centered at zero. The asymmetric case can be handled similarly by introducing a zero-point; further details are provided in Appendix~\ref{app:prelim_quant_and_dequant}. Weight quantization can be applied at various granularities, including per-tensor, group-wise, and per-channel quantization; see \cite{xiao2023smoothquant} for a detailed discussion.

\vspace{-6pt}
\paragraph{Challenges in PTQ of LLMs} Although quantization is appealing to use for memory-bound deployment, naively applying it as described above can lead to significant performance degradation in downstream tasks. This is due to the "outliers" in activations and weights of LLMs. These "outliers" are (potentially extreme) large values that can dominate the limited dynamic range of low-bit formats and degrade quantization fidelity \citep{4bitinteger,xiao2023smoothquant,lin2024awq}. To alleviate these challenges, one line of work applies extra scaling factors to both weights and activations \citep{xiao2023smoothquant, lin2024awq}. That is, for a linear layer with weight $\mathbf{W}_l \in\mathbb{R}^{d_{\text{out}}\times d_{\text{in}}}$, given a scaling factor matrix $\widetilde{\mathbf{S}}_l \in\mathbb{R}^{d_{\text{in}} \times d_{\text{in}}}$ and the input $\mathbf{X}_l \in\mathbb{R}^{d_{\text{in}}\times n}$ conditioned on the calibration data, the output of this layer can be rewritten as
\begin{talign}
\label{eq:scale-ptq}
    \mathbf{Y}_l := \mathbf{W}_l \mathbf{X}_l
    =
    \left(\mathbf{W}_l \widetilde{\mathbf{S}}_l\right)
    \left(\widetilde{\mathbf{S}}_l^{-1}\mathbf{X}_l\right)~~,
\end{talign}
where $\widetilde{\mathbf{S}}_l^{-1}$ denotes the inverse of $\widetilde{\mathbf{S}}_l$. This transformation is function-preserving, since it leaves the layer output unchanged in the absence of discretization error. A common choice is to take $\widetilde{\mathbf{S}}_l := \mathrm{diag}(\tilde{s}_l)$, where $\tilde{s}_l \in \mathbb{R}^{d_{\text{in}}}$ is a channel-wise scaling vector constructed from summary statistics of the weights and activations. The optimal quantization operations $Q$ including $\widetilde{\mathbf{S}}_l$ and the associated quantized weights $\widehat{\mathbf{W}}_l$ are selected by minimizing $\|\widehat{\mathbf{W}}_l \mathbf{X}_l - \mathbf{W}_l \mathbf{X}_l\|_\mathrm{F}^2$. For instance, AWQ \citep{lin2024awq} uses activation statistics to guide per-channel scaling factors $\tilde{s}_l$ to better preserve salient channels for weight-only PTQ. This approach alleviates the impact of activation outliers by applying $\mathrm{diag}(\tilde{s}_l)^{-1}$ to activations, while the scaled weights $\mathbf{W}_l \mathrm{diag}(\tilde{s}_l)$ can be absorbed into the offline weights.

An alternative line of work applies orthogonal or structured rotations to redistribute outliers across dimensions before quantization, aiming to make both activations and weights more quantization-friendly. For instance, QuIP \citep{Chee2023quip} introduces orthogonal transformations to reduce the impact of outliers, QuaRot \citep{Ashkboos2024quarot} uses Hadamard-style transformations to suppress outliers while preserving the model function, and SpinQuant \citep{liuspinquant} further optimizes the rotation parameters to better match quantization constraints. Despite their effectiveness, rotation-based methods may introduce additional calibration, optimization, or inference costs, especially when the transformations are learned \citep{4bitinteger,liuspinquant}.

\vspace{-6pt}
\paragraph{Grid Search over Scaling Factors}
As discussed above, a common strategy in PTQ is to apply tensor-, group-, or channel-wise scaling, which rebalances the quantization difficulty between activations and weights and often improves downstream performance. In practice, the scaling factors $\tilde{s}_l$ in \Cref{eq:scale-ptq} are derived from simple tensor-, group-, or channel-level statistics that serve as heuristic proxies for activation or weight saliency. Different scaling-based PTQ methods instantiate these factors in different ways. A widely adopted class of scaling factors is
\begin{talign}
\label{eq: widely used scaling factors}
    \tilde{s}_{l}^{(j)}
    \propto
    \frac{\mathrm{stat}_{X}(|\mathbf{X}_l^{(j)}|)^{\alpha}}
    {\mathrm{stat}_{W}(|\mathbf{W}_l^{(j)}|)^{1-\alpha}},
\end{talign}
where $\alpha\in[0,1]$ and $\mathrm{stat}_{X}(\cdot)$ and $\mathrm{stat}_{W}(\cdot)$ denote pre-defined activation/weight summary statistics. This form covers the scaling factors used in \cite{xiao2023smoothquant}, where maximum statistics are used to migrate activation outliers into weights. In contrast, \cite{lin2024awq} only uses channel-wise activation summary statistics as scaling factors e.g., $\tilde{s}_{l}^{(j)} \propto \mathrm{mean}(|\mathbf{X}_l^{(j)}|)^\alpha$. Under such parameterizations, PTQ can be simplified to a lightweight grid search over scaling parameters such as $\alpha$, with the goal of minimizing the reconstruction error between the original FP layer output and the quantized layer output conditioning on the calibration data.

\vspace{-6pt}
\paragraph{Gram-based PTQ via Optimal Brain Compression}
As an alternative to grid search over scaling factors, Gram-based methods are also widely adopted in practice \citep{frantar2023gptq,li2025gptqv2}. Second-order compression dates back to Optimal Brain Surgeon (OBS) \citep{yannoptimal1989,NIPS1993_b056eb15}, which uses curvature information to estimate how parameter perturbations affect the loss. Optimal Brain Compression (OBC) connects this idea to calibration objectives by introducing a tractable curvature proxy. When mean squared error is used as the layer-wise calibration objective with inputs $\mathbf{X}_l$, the resulting quadratic form involves the Gram matrix $\mathbf{X}_l\mathbf{X}_l^\top$, which serves as a surrogate curvature and penalizes output distortion induced by weight perturbations \citep{frantar2022optimal}. 
\cite{frantar2023gptq} further adapt OBC to weight-only PTQ by performing efficient sequential or block-wise updates using cached activation statistics, which do not require end-to-end backpropagation. That is, 
\begin{talign}
    \min_{\Delta \mathbf{W}_l} \| (\mathbf{W}_l + \Delta \mathbf{W}_l) \mathbf{X}_l - \mathbf{W}_l \mathbf{X}_l\|_\mathrm{F}^2 \quad, \, \, \mathrm{s.t.} \, \Delta \mathbf{W}_l =  \widehat{\mathbf{W}}_l - \mathbf{W}_l.
\end{talign}
where $\widehat{\mathbf{W}}_l$ denotes the quantized weights. Note that this is actually equivalent to select $\widehat{\mathbf{W}}_l$ by minimizing $ \| \widehat{\mathbf{W}}_l \mathbf{X}_l - \mathbf{W}_l \mathbf{X}_l \|_\mathrm{F}^2 $.
Several follow-up works further improve the effectiveness of OBC-style methods \citep{li2025gptqv2,vanbaalen2025gptvqblessingdimensionalityllm}.

\section{Methodology}
\label{sec:method}

In this section, we formally introduce \emph{Regularized Quantization Calibration} and its extension \emph{Saliency-Aware Regularized Quantization Calibration} (SARQC) in \Cref{sec:rqc}. We then show how to optimize the resulting regularized objectives in \Cref{sec: optimize sarqc}. 

\subsection{Saliency-Aware Regularized Quantization Calibration}
\label{sec:rqc}

We begin by reviewing the standard calibration objective used in post-training quantization for LLMs, which measures the output discrepancy induced by quantization. For the $l$-th linear layer with FP weight $\mathbf{W}_l\in\mathbb{R}^{d_{\text{out}}\times d_{\text{in}}}$ and cached calibration input $\mathbf{X}_l\in\mathbb{R}^{d_{\text{in}}\times n}$ conditioning on the calibration data, PTQ seeks the quantized weights $\widehat{\mathbf{W}}_l$ by minimizing the reconstruction error
\begin{talign}
\label{eq:ptq_recon}
\min_{\widehat{\mathbf{W}}_l\in\mathcal{Q}}
\;\; \|\mathbf{W}_l \mathbf{X}_l - \widehat{\mathbf{W}}_l \mathbf{X}_l\|_\mathrm{F}^2,
\end{talign}
where $\mathcal{Q}$ denotes the feasible set of dequantized-quantized weights induced by the underlying quantization scheme, possibly together with additional reparameterizations such as scaling by $\widetilde{\mathbf{S}}_l$. Note that each quantization operation $Q$ induces a Dirac measure supported at the associated quantized weight $\widehat{\mathbf{W}}_l$. The collection of all quantization operations therefore induces a family of Dirac measures $\mathcal{Q}$ over the space of quantized weights.

\vspace{-6pt}
\paragraph{Why It May Fail} Existing calibration-based quantization methods aim to minimize the reconstruction error $\|\mathbf{W}_l\mathbf{X}_l-\widehat{\mathbf{W}}_l\mathbf{X}_l\|_\mathrm{F}^2$ conditioning on the input $\mathbf{X}_l$ from a pre-determined calibration dataset to preserve the capacity of LLMs. However, minimizing the reconstruction error using only the calibration data does not explicitly control the deviation between $\widehat{\mathbf{W}}_l$ and $\mathbf{W}_l$, which is critical to the generalization performance of quantized LLMs on downstream tasks, as shown in \Cref{thm:generalization risk}.

\begin{theorem}
\label{thm:generalization risk}
Consider a restricted finite quantization class $\mathcal Q_R :=\{ \widehat{\mathbf W}_l \in \mathcal Q : \| \widehat{\mathbf W}_l - \mathbf W_l \|_\mathrm{F} \le R\}$. Define \(\Delta \mathbf W_l := \widehat{\mathbf W}_l - \mathbf W_l\), and assume that
\(\| X\|_2 \le M_X\) almost surely. Let the true downstream reconstruction risk be $\mathcal R(\widehat{\mathbf W}_l)
:=
\mathbb E_{X \sim p_X}\!\left[
\|
\Delta \mathbf W_l X
\|_2^2
\right].$ Let the calibration set at layer \(l\) be $\mathcal D_{\mathrm{cal},l} := \{X_{l,i}\}_{i=1}^n,
X_{l,1},\dots,X_{l,n}\stackrel{\mathrm{i.i.d.}}{\sim} p_X,$ and define the empirical calibration risk as $\widehat{\mathcal R}_{\mathrm{cal}}(\widehat{\mathbf W}_l)
:=
\frac{1}{n}\sum_{i=1}^n
\|
\Delta \mathbf W_l X_{l,i}
\|_2^2.$ Then, for any \(\delta\in(0,1)\), with probability at least \(1-\delta\), the following holds simultaneously for all \(\widehat{\mathbf W}_l\in\mathcal Q_R\):
\begin{talign*}
\left|
\mathcal R(\widehat{\mathbf W}_l)
-
\widehat{\mathcal R}_{\mathrm{cal}}(\widehat{\mathbf W}_l)
\right|
\le
R^2 M_X^2
\sqrt{\frac{
\log \frac{2|\mathcal Q_R|}{\delta}
}{2n}}.
\end{talign*}
\end{theorem}
See Appendix~\ref{append: proof of generalization risk} for proof. The above theorem shows that the true downstream risk can be controlled by two terms:
the empirical calibration risk and a generalization term depending on the radius of the weight drift $R$. Large weight drift tends to increase the generalization risk. Therefore, minimizing only
$\widehat{\mathcal R}_{\mathrm{cal}}(\widehat{\mathbf W}_l)$ may lead to poor downstream
generalization if the selected quantized weights have large drift from the original FP weights.

This is also supported by the empirical results in \Cref{fig:illustration} and \Cref{fig:3d_weight_diff} (in Appendix~\ref{app:3d_weight_diff}): undesired deviation to the original FP weights could potentially lower the performance of the weight-only quantized models. In this sense, the degradation in performance of vanilla calibration can be understood as moving away from the \emph{sweet spot} that balances output mismatch and weight drift. These motivate the following regularized calibration objective for weight-only post-training quantization.

\vspace{-6pt}
\paragraph{Regularized Quantization Calibration}  \Cref{thm:generalization risk} implies that it might be better to control the differences between $\mathbf{W}_l$ and $\widehat{\mathbf{W}}_l$ for each quantized layer indexed by $l$. To control this deviation explicitly, we introduce an extra regularization term into the quantization calibration that allows control to such deviation, given by $\mathcal{D}(\mathbf{W}_l, \widehat{\mathbf{W}}_l)$ where $\mathcal{D}(\cdot,\cdot)$ is some discrepancy one may specify. That is,
\begin{talign}
\label{eq:rqc_layer_most_genreal}
\min_{\widehat{\mathbf{W}}_l \in\mathcal{Q}}\|\mathbf{W}_l \mathbf{X}_l - \widehat{\mathbf{W}}_l \mathbf{X}_l\|_\mathrm{F}^2
+
\lambda \mathcal{D} (\widehat{\mathbf{W}}_l, \mathbf{W}_l) ~~.
\end{talign}
where $\lambda>0$ controls the strength of regularization. Note here we regard $\widehat{\mathbf{W}}_l$ and $\mathbf{W}_l$ as vectors in $\mathcal{D} (\widehat{\mathbf{W}}_l, \mathbf{W}_l)$, i.e., $\mathcal{D} (\textrm{vec}(\widehat{\mathbf{W}}_l), \textrm{vec}(\mathbf{W}_l))$ given some discrepancy or distance metric $\mathcal{D}$. To simplify notation, we drop $\textrm{vec}(\cdot)$ when it is clear from the context. A natural choice for $\mathcal{D}$ is the Kullback--Leibler (KL) divergence. However, it could lead to ill-posed discrepancy measurements as $\mathrm{KL}(\delta_{\mathbf{W}_l}, \delta_{\widehat{\mathbf{W}}_l})=0$ if $\mathbf{W}_l=\widehat{\mathbf{W}}_l$ otherwise $+\infty$ where $\delta_{\mathbf W_l}$ and $\delta_{\widehat{\mathbf W}_l}$ denote the Dirac measures supported at $\mathbf W_l$ and $\widehat{\mathbf W}_l$, respectively. While the Wasserstein-2 distance remains well-defined in this setting,
and reduces to the squared distance between the weights $
W_2^2(\delta_{\mathbf{W}_l}, \delta_{\widehat{\mathbf{W}}_l})
= \|\mathbf{W}_l - \widehat{\mathbf{W}}_l\|_\mathrm{F}^2 = \|\textrm{vec}(\mathbf{W}_l) - \textrm{vec}(\widehat{\mathbf{W}}_l)\|_2^2$. That is, we select the quantized weights $\widehat{\mathbf{W}}_l \in \mathcal{Q}$ by minimizing:
\begin{talign}
\label{eq:rqc_layer_without_saliency}
\min_{\widehat{\mathbf{W}}_l \in\mathcal{Q}} \|\mathbf{W}_l \mathbf{X}_l - \widehat{\mathbf{W}}_l \mathbf{X}_l\|_\mathrm{F}^2
+
\lambda \| (\widehat{\mathbf{W}}_l - \mathbf{W}_l)\|_\mathrm{F}^2 ~~.
\end{talign} 
The reconstruction error encourages the quantized layer to preserve output-relevant behavior on inputs drawn from the calibration distribution, while the regularization term discourages unnecessary deviation from the original FP models. The following corollary establishes the connection between the objective of \emph{Regularized Quantization Calibration} (RQC) in \Cref{eq:rqc_layer_without_saliency} and \Cref{thm:generalization risk}.

\begin{corollary} 
\label{cor:constraint_penalty_supported}
    Fix a layer indexed by $l$, assume that $\mathcal Q$ is finite and the constrained feasible set $\mathcal Q_R:=\{\widehat{\mathbf W}^{\prime}\in\mathcal Q: \|\widehat{\mathbf W}^{\prime}-\mathbf W_l\|_{\mathrm F}^2  \le R^2\}$ is nonempty. 
    Let $\widehat{\mathbf W}_l\in\arg\min_{\widehat{\mathbf W}^{\prime}\in\mathcal Q_R}\widehat{\mathcal R}_{\mathrm{cal}} (\widehat{\mathbf W}^{\prime})$. 
    If the supportedness condition $\lambda_{\min}\le \lambda_{\max}$ holds, then for every finite $\lambda_R$ satisfying $\lambda_{\min}\le \lambda_R\le \lambda_{\max}$,
    \begin{talign*}
        \widehat{\mathbf W}_l
        \in \arg\min_{\widehat{\mathbf W}^{\prime}\in\mathcal Q}
        \left\{ \widehat{\mathcal R}_{\mathrm{cal}} (\widehat{\mathbf W}^{\prime}) + \lambda_R \|\widehat{\mathbf W}^{\prime}-\mathbf W_l\|_{\mathrm F}^2 \right\},
    \end{talign*}
    where $\lambda_{\min}$ and $\lambda_{\max}$ are defined in Appendix~\ref{append:cor_constraint_penalty}.
    Conversely, if there exists a finite $\lambda_R\ge 0$ such that $\widehat{\mathbf W}_l \in \arg\min_{\widehat{\mathbf W}^{\prime}\in\mathcal Q} \{ \widehat{\mathcal R}_{\mathrm{cal}} (\widehat{\mathbf W}^{\prime})+\lambda_R \|\widehat{\mathbf W}^{\prime}-\mathbf W_l\|_{\mathrm F}^2 \}$, then necessarily $\lambda_{\min}\le \lambda_R\le \lambda_{\max}$.
\end{corollary}

See Appendix~\ref{append:cor_constraint_penalty} for proof. Corollary~\ref{cor:constraint_penalty_supported} clarifies that the constrained calibration problem can be represented by a regularized calibration objective over a finite quantization set. In particular, when $\lambda_{\min}\le \lambda_{\max}$ holds, then one can choose a penalty strength $\lambda_R\in[\lambda_{\min},\lambda_{\max}]$ such that $\widehat{\mathbf W}_l$ remains optimal after replacing the hard radius constraint $D(\widehat{\mathbf W}^{\prime})\le R^2$ by the soft penalty $\lambda_R D(\widehat{\mathbf W}^{\prime})$.
Here $\lambda_{\min}\le\lambda_{\max}$ simply states that the penalty must be large enough to suppress farther, possibly infeasible, low-risk candidates, but not so large that it favors closer, higher-risk feasible candidates. This is the exact condition under which the bias introduced by the penalty does not change the selected minimizer. See Appendix~\ref{append:cor_constraint_penalty} for more details, interpretation of $\lambda_{\min}$ and $\lambda_{\max}$, reasonableness of the supportedness condition and the connection to Lagrangian relaxation.

\vspace{-3pt}
\paragraph{Saliency-Aware Regularized Quantization Calibration} As discussed in \Cref{sec:background}, a key observation in previous works is that introducing scaling factors $\widetilde{\mathbf{S}}_l$ is essential for maintaining performance, as it effectively transfers the difficulty of quantization from activations to weights \citep{xiao2023smoothquant, lin2024awq}. These scaling factors also serve to quantify the saliency of the weights. This naturally leads to saliency-aware constraints on the weights. These motivate augmenting the reconstruction error with a saliency-weighted penalty. We therefore quantize each layer of LLMs by minimizing the following regularized quantization calibration objective, namely \emph{Saliency-Aware Regularized Quantization Calibration} (SARQC), given by:
\begin{talign}
\label{eq:rqc_layer}
\min_{\widehat{\mathbf{W}}_l \in\mathcal{Q}}
\underbrace{\|\mathbf{W}_l \mathbf{X}_l - \widehat{\mathbf{W}}_l \mathbf{X}_l\|_\mathrm{F}^2}_{\mathcal{L}_{\text{recon}}:=\text{Reconstruction Error}}
+
\lambda \underbrace{\| (\widehat{\mathbf{W}}_l - \mathbf{W}_l) \mathbf{S}_l\|_\mathrm{F}^2}_{\mathcal{L}_{\text{sar}}:=\text{Saliency-Aware Regularization}} ~~,
\end{talign}
where $\mathbf{S}_l$ encodes the saliency of the weights; see also \Cref{eq: widely used scaling factors}. Note that $\widetilde{\mathbf{S}}_l$ and $\mathbf{S}_l$ may both be present, in which case additional handling is required. In this setting, we find that encouraging $\mathbf{S}_l$ to be close to $\widetilde{\mathbf{S}}_l$ and removing the tuning hyperparameter associated with $\mathbf{S}_l$ works well in practice; see \Cref{sec: optimize sarqc} for details. The first term $\mathcal{L}_{\text{recon}}$ matches layer outputs between FP weights and quantized weights conditioned on the calibration data, while the second term $\mathcal{L}_{\text{sar}}$ penalizes deviations from the original FP weights in a saliency-aware manner. Since $\mathbf{S}_l=\mathrm{diag}(s_l)$ reweights different channels, larger saliency values impose stronger penalties on the corresponding weight deviations. In this way, the regularizer more strongly protects important channels that are more influential to the layer output under low-bit quantization. When $\mathbf{S}_l=I$, the penalty reduces to a uniformly weighted regularizer and recovers the objective of RQC in \Cref{eq:rqc_layer_without_saliency}. \emph{SARQC constrains information loss primarily along directions that are most relevant to the saliency.} We later show that this saliency-aware regularizer $\mathcal{L}_{\text{sar}}$ with $\mathbf{S}_l \neq I$ further helps maintain the performance of LLMs after post-training quantization, as supported by experimental results in \Cref{sec:result analysis}.

\subsection{Practical Strategies to Optimize RQC and SARQC}
\label{sec: optimize sarqc}

We now describe two practical strategies for optimizing the proposed RQC objective in \Cref{eq:rqc_layer_without_saliency} and the SARQC objective in \Cref{eq:rqc_layer}. Since SARQC reduces to RQC when $\mathbf{S}_l = I$ (identity matrix), these strategies are directly applicable to both objectives simultaneously.

\vspace{-8pt}
\paragraph{Optimize via Grid Search} 
Grid search over scaling factors $\widetilde{\mathbf{S}}_l$ is commonly used in PTQ methods to capture the saliency of weights. Therefore, selecting an optimal quantization operation for each layer amounts to choosing a small set of parameters associated with these scaling factors. Let $\widetilde{\mathbf{S}}_l := \mathrm{diag} (\tilde{s}_l(\alpha))$ with $\tilde{s}_l(\alpha) \in \mathbb{R}_{+}^{d_{\text{in}}}$ being a scaling vector constructed from activation/weight statistics (e.g., channel-wise) with $\alpha$; see \Cref{eq: widely used scaling factors} for the widely-adopted ones. In this case, the quantized weights $\widehat{\mathbf{W}}_l(\alpha)$ are functions of $\alpha$ through the scaling matrix $\widetilde{\mathbf{S}}_l(\alpha)$.
This naturally reframes the optimization process of SARQC as a grid-search process over $\alpha \in [0,1]$ in the scaling factor $\widetilde{\mathbf{S}}_l(\alpha)$. Therefore, to preserve LLM performance after PTQ, SARQC selects $\widehat{\mathbf{W}}_l(\alpha)$ by optimizing
\begin{talign}
\label{eq:rqc_grid}
\min_{\alpha \in \{\alpha_k\}_{k=0}^{K-1}}
\;\;
\|\mathbf{W}_l\mathbf{X}_l-\widehat{\mathbf{W}}_l(\alpha)\mathbf{X}_l\|_\mathrm{F}^2
+
\lambda\|(\widehat{\mathbf{W}}_l(\alpha)-\mathbf{W}_l)\mathbf{S}_l\|_\mathrm{F}^2 ~~,
\end{talign}
where, following the grid-search design in \citep{lin2024awq}, we consider a discrete grid
$\{\alpha_k\}_{k=0}^{K-1}\subset[0,1]$, with $K \in \mathbb{N}^+$, $K\ge 2$, and
$\alpha_k = k/(K-1)$ for $k=0,1,\dots,K-1$.
In the first term, the candidate quantized weights $\widehat{\mathbf{W}}_l(\alpha)$ depends on the scaling matrix $\widetilde{\mathbf{S}}_l(\alpha)$, for which we use the widely adopted choice in \Cref{eq: widely used scaling factors}. For SARQC, we find that setting $\mathbf{S}_l:=\mathrm{diag}(s_l)$, with the $j$-th element
$s_l^{(j)} := \nicefrac{\mathrm{mean}(|\mathbf{X}_l^{(j)}|)}{\mathrm{mean}(|\mathbf{W}_l^{(j)}|)}$,
performs well in practice, also inspired by \citep{xiao2023smoothquant}. To balance the scale between the two terms, we apply min--max normalization and select the best candidate by minimizing the normalized joint objective. This also aligns with Corollary~\ref{cor:constraint_penalty_supported} to ensure that the range of $\lambda$ is proper. We refer to this approach as \texttt{SARQC-GS(Saliency)} when $\mathbf{S}_l\neq I$. When $\mathbf{S}_l:=I$, we denote the resulting variant as \texttt{SARQC-GS(Identity)}, which recovers the RQC case. This enables the seamless integration of RQC and SARQC into grid-search-based PTQ methods without requiring backpropagation through quantization operations or scaling factors. See \Cref{alg:sarqc_grid} (in Appendix~\ref{appendix: pseudo algorithms}) and Appendix~\ref{alg:Implementation_Details} for more details.

\vspace{-8pt}
\paragraph{Optimize via Gram-Based Solvers}
An alternative approach for selecting the optimal quantized weights $\widehat{\mathbf{W}}_l$ is to employ second-order Gram-based solvers to minimize the SARQC objective as \citep{frantar2023gptq, li2025gptqv2}. Let $\Delta \mathbf{W}_l := \widehat{\mathbf{W}}_l - \mathbf{W}_l$. Since $\|\mathbf{W}_l\mathbf{X}_l-\widehat{\mathbf{W}}_l \mathbf{X}_l\|_\mathrm{F}^2 = \|\Delta \mathbf{W}_l \mathbf{X}_l\|_\mathrm{F}^2$, we can rewrite \eqref{eq:rqc_layer} as
\begin{talign}
\label{eq:rqc_quadratic}
\min_{\widehat{\mathbf{W}}_l\in\mathcal{Q}}
\;\; \mathrm{Tr}\ \!\Big(\Delta \mathbf{W}_l\, \mathbf{G}_l \,\Delta \mathbf{W}_l^\top\Big) ~~,
\end{talign}
where the regularized curvature matrix used in \texttt{SARQC-GBS} is $\mathbf{G}_l:=\mathbf{X}_l\mathbf{X}_l^\top+\lambda\mathbf{S}_l\mathbf{S}_l^\top$. Due to page limits, please refer to Appendix~\ref{alg:Implementation_Details} for detailed discussions on the choice of $\mathbf{S}_l$. Compared with GPTQ \citep{frantar2023gptq}, RQC and SARQC preserve the same layer-wise calibration structure while replacing the Gram matrix $\mathbf{X}_l \mathbf{X}_l^\top$ with the modified curvature $\mathbf{G}_l$. This allows RQC/SARQC to be optimized in a GPTQ-style sequential manner by using $\mathbf{G}_l$ and its inverse or Cholesky factor in place of $\mathbf{X}_l \mathbf{X}_l^\top$, while retaining the same compensation mechanism. We refer to this approach as \texttt{SARQC-GBS(Saliency)}. If set $\mathbf{S}_l:=I$, we denote this as \texttt{SARQC-GBS(Identity)}, which also recovers the RQC case. See Appendix~\ref{app:derivation_rqc_gbs} for detailed derivations and further discussions on its connection to GPTQ and see \Cref{alg:sarqc_gram} (in Appendix~\ref{appendix: pseudo algorithms}) and Appendix~\ref{alg:Implementation_Details} for more implementation details.

\section{Experimental Results}
In this section, we compare the proposed method, RQC and SARQC, with widely-adopted weight-only post-training quantization methods across a diverse range of LLMs and downstream tasks. Note that as SARQC covers RQC as discussed in \Cref{sec:method}, we will drop the notation of RQC in this section. We report experimental setup in \Cref{sec:exp setup}, and analyze results in \Cref{sec:result analysis}. See extra details in Appendix~\ref{app:pseudo algorithms and implementation details} and additional experimental results in Appendix~\ref{appx:extra experiments}.

\subsection{Experimental Setup}
\label{sec:exp setup}

\paragraph{Baselines}
We compare our method with representative layer-wise PTQ baselines based on linear quantization. Specifically, we consider \textsc{AWQ} \citep{lin2024awq} which uses statistics-guided scaling with grid search as the baseline for \texttt{SARQC-GS}. For the baselines of \texttt{SARQC-GBS}, we consider \textsc{GPTQ}~\citep{frantar2023gptq}, which performs layer-wise reconstruction using second-order Hessian information, and \textsc{GPTAQ}~\citep{li2025gptqv2}, a refined variant of \textsc{GPTQ}. Note that the GPTAQ baseline results are omitted for all MoE models due to a lack of MoE architectural support in implementation.

\vspace{-10pt}
\paragraph{Models}
We consider both \emph{dense LLMs} and \emph{Mixture-of-Experts (MoE) LLMs}. For dense models, we use \textsc{LLaMA2} (7B and 13B) \citep{llama2} and \textsc{LLaMA} (7B, 13B and 30B, reported in Appendix~\ref{appx:detail_of_llama1}). For MoE models, we use \textsc{DeepSeek-MoE-16B-Base} \citep{deepseek}, \textsc{Qwen3-MoE-30B} \citep{qwen3}, and \textsc{Mixtral-8x7B} \citep{mixtral}.

\vspace{-10pt}
\paragraph{Evaluations}
We evaluate quantization performance using two types of metrics: (i) \emph{perplexity} on the test split of \textsc{WikiText2} \citep{wikitext}; (ii) \emph{zero-shot accuracy} on a suite of commonsense reasoning and knowledge benchmarks including: \textsc{PIQA}, \textsc{HellaSwag} \citep{HellaSwag}, \textsc{MMLU} \citep{mmlu}, \textsc{HumanEval} \citep{humaneval}, \textsc{BoolQ} \citep{boolq}, \textsc{ARC-Challenge}, \textsc{ARC-Easy} \citep{arce}, and \textsc{WinoGrande} \citep{winogrande}.

\vspace{-10pt}
\paragraph{Implementation}
We investigate weight-only post-training quantization at low bit widths, ranging from  $4$-bit, $3$-bit, to $2$-bit. We apply group-wise quantization to all linear layers. We use a group size of $64$ for MoE models and $128$ for dense models. Unless otherwise stated in the ablation studies, all methods use the same fixed set of $128$ calibration samples from the training split of \textsc{WikiText2} \citep{wikitext} to ensure a controlled comparison. We do not introduce additional transformations or reordering of weights or activations, so that the results reflect the intrinsic performance of each method. For hyperparameter selection, see Appendix~\ref{alg:Implementation_Details} for more details. All experiments are conducted on NVIDIA A100 80GB GPUs.

\vspace{-10pt}
\begin{table*}[htbp]
\centering
\caption{Perplexity of quantized models under different bit-widths.} 
\label{tab:ppl_all_wxa16}
\setlength{\tabcolsep}{4.8pt}
\renewcommand{\arraystretch}{1.08}
\definecolor{rqcrow}{RGB}{244,220,217}

\resizebox{0.98\textwidth}{!}{%
\begin{tabular}{l l c c c c c c c c}
\toprule
Quant. & Method & LLaMA-7B & LLaMA-13B & LLaMA-30B & LLaMA2-7B & LLaMA2-13B & DeepSeek-MoE-16B & Qwen3-MoE-30B & Mixtral-8x7B \\
\midrule

\multirow{1}{*}{FP16/BF16}
& FP16/BF16
& 5.68 & 5.09 & 4.10 & 5.47 & 4.88 & 6.51 & 6.09 & 3.84 \\
\midrule

\multirow{4}{*}{W2A16}
& GPTQ
& 1378.39 & 237.32 & 27.33 & 1889.54 & 792.32 & 44.51 & 38.03 & 1079.31 \\
& GPTAQ
& 467.79 & 95.52 & 32.86 & 358.46 & 176.53 & -- & -- & -- \\
& \cellcolor[RGB]{234, 239, 247}SARQC-GBS(Identity)
& \cellcolor[RGB]{234, 239, 247}212.32 & \cellcolor[RGB]{234, 239, 247}21.39 & \cellcolor[RGB]{234, 239, 247}24.99 & \cellcolor[RGB]{234, 239, 247}396.52 & \cellcolor[RGB]{234, 239, 247}190.36 & \cellcolor[RGB]{234, 239, 247}27.29 & \cellcolor[RGB]{234, 239, 247}17.41 & \cellcolor[RGB]{234, 239, 247}34.02 \\
& \cellcolor[RGB]{210, 222, 239}SARQC-GBS(Saliency)
& \cellcolor[RGB]{210, 222, 239}\textbf{62.64} & \cellcolor[RGB]{210, 222, 239}\textbf{19.89} & \cellcolor[RGB]{210, 222, 239}\textbf{14.70} & \cellcolor[RGB]{210, 222, 239}\textbf{225.01} & \cellcolor[RGB]{210, 222, 239}\textbf{117.67} & \cellcolor[RGB]{210, 222, 239}\textbf{22.83} & \cellcolor[RGB]{210, 222, 239}\textbf{16.69} & \cellcolor[RGB]{210, 222, 239}\textbf{31.48} \\
\midrule

\multirow{4}{*}{W3A16}
& GPTQ
& 7.03 & 5.77 & 4.99 & 6.80 & 6.71 & 7.21 & 9.66 & 4.86 \\
& GPTAQ
& \textbf{6.68} & 5.82 & 4.96 & 6.65 & 6.65 & -- & -- & -- \\
& \cellcolor[RGB]{234, 239, 247}SARQC-GBS(Identity)
& \cellcolor[RGB]{234, 239, 247}6.95 & \cellcolor[RGB]{234, 239, 247}\textbf{5.65} & \cellcolor[RGB]{234, 239, 247}4.77 & \cellcolor[RGB]{234, 239, 247}7.49 & \cellcolor[RGB]{234, 239, 247}\textbf{6.36} & \cellcolor[RGB]{234, 239, 247}7.15 & \cellcolor[RGB]{234, 239, 247}6.91 & \cellcolor[RGB]{234, 239, 247}4.72 \\
& \cellcolor[RGB]{210, 222, 239}SARQC-GBS(Saliency)
& \cellcolor[RGB]{210, 222, 239}6.70 & \cellcolor[RGB]{210, 222, 239}5.71 & \cellcolor[RGB]{210, 222, 239}\textbf{4.72} & \cellcolor[RGB]{210, 222, 239}\textbf{6.63} & \cellcolor[RGB]{210, 222, 239}6.40 & \cellcolor[RGB]{210, 222, 239}\textbf{7.08} & \cellcolor[RGB]{210, 222, 239}\textbf{6.88} & \cellcolor[RGB]{210, 222, 239}\textbf{4.70} \\
\midrule

\multirow{7}{*}{W4A16}
& AWQ
& 5.81 & 5.20 & 4.21 & 5.62 & 4.97 & 7.42 & 6.80 & 4.03 \\
& GPTQ
& 6.73 & 5.27 & 4.25 & 5.72 & 5.19 & 6.63 & 8.86 & 4.08 \\
& GPTAQ
& 5.95 & 5.24 & 4.25 & 5.68 & 10.63 & -- & -- & -- \\
& \cellcolor[RGB]{234, 239, 247}SARQC-GS(Identity)
& \cellcolor[RGB]{234, 239, 247}5.80 & \cellcolor[RGB]{234, 239, 247}5.18 & \cellcolor[RGB]{234, 239, 247}4.21 & \cellcolor[RGB]{234, 239, 247}5.61 & \cellcolor[RGB]{234, 239, 247}4.96 & \cellcolor[RGB]{234, 239, 247}7.17 & \cellcolor[RGB]{234, 239, 247}6.75 & \cellcolor[RGB]{234, 239, 247}4.02 \\
& \cellcolor[RGB]{210, 222, 239}SARQC-GS(Saliency)
& \cellcolor[RGB]{210, 222, 239}\textbf{5.77} & \cellcolor[RGB]{210, 222, 239}\textbf{5.17} & \cellcolor[RGB]{210, 222, 239}\textbf{4.20} & \cellcolor[RGB]{210, 222, 239}\textbf{5.60} & \cellcolor[RGB]{210, 222, 239}\textbf{4.96} & \cellcolor[RGB]{210, 222, 239}6.93 & \cellcolor[RGB]{210, 222, 239}6.71 & \cellcolor[RGB]{210, 222, 239}\textbf{4.00} \\
& \cellcolor[RGB]{234, 239, 247}SARQC-GBS(Identity)
& \cellcolor[RGB]{234, 239, 247}6.01 & \cellcolor[RGB]{234, 239, 247}5.23 & \cellcolor[RGB]{234, 239, 247}4.22 & \cellcolor[RGB]{234, 239, 247}5.97 & \cellcolor[RGB]{234, 239, 247}5.16 & \cellcolor[RGB]{234, 239, 247}6.65 & \cellcolor[RGB]{234, 239, 247}6.30 & \cellcolor[RGB]{234, 239, 247}4.06 \\
& \cellcolor[RGB]{210, 222, 239}SARQC-GBS(Saliency)
& \cellcolor[RGB]{210, 222, 239}5.88 & \cellcolor[RGB]{210, 222, 239}5.22 & \cellcolor[RGB]{210, 222, 239}4.21 & \cellcolor[RGB]{210, 222, 239}5.65 & \cellcolor[RGB]{210, 222, 239}5.05 & \cellcolor[RGB]{210, 222, 239}\textbf{6.62} & \cellcolor[RGB]{210, 222, 239}\textbf{6.27} & \cellcolor[RGB]{210, 222, 239}4.02 \\
\bottomrule
\end{tabular}%
}
\end{table*}

\subsection{Results Analysis}
\label{sec:result analysis}

\emph{(1) When does vanilla reconstruction error-based calibration become unreliable?}
Vanilla calibration based solely on reconstruction error becomes increasingly unreliable in challenging quantization regimes, especially under more aggressive low-bit settings such as W2A16 and W3A16. As shown in \Cref{tab:ppl_all_wxa16}, this failure can manifest as a dramatic increase in perplexity, indicating that preserving calibration-set reconstruction alone is insufficient to maintain the original generation behavior of LLMs. Even under the milder W4A16 setting, the same tendency is reflected in downstream zero-shot performance in \Cref{tab:rqc_zeroshot_acc_w4a16_others}, where reconstruction-only calibration often leads to noticeable accuracy degradation relative to the original LLMs. This issue becomes even more pronounced when the calibration set is small. As illustrated in \Cref{fig:main_and_trend}\emph{(b)}, GPTQ degrades substantially under both W3A16 and W2A16 when only limited calibration data are available. Taken together, these results suggest that standard calibration objectives are vulnerable to limited or unrepresentative calibration data.

\vspace{-8pt}
\begin{table*}[h]
\centering
\caption{Zero-shot accuracy (\%) on multiple benchmarks under W4A16.}
\label{tab:rqc_zeroshot_acc_w4a16_others}
\setlength{\tabcolsep}{5.0pt}
\renewcommand{\arraystretch}{1.06}
\definecolor{rqcrow}{RGB}{244,220,217}

\resizebox{1.\textwidth}{!}{%
\begin{tabular}{l l c c c c c c c c c}
\toprule
Model & Method & PIQA & HellaSwag & MMLU & HumanEval & BoolQ & WinoGrande & ARC-E & ARC-C & Avg. \\
\midrule

\multirow{8}{*}{LLaMA2-7B}
& FP16      & 78.13 & 57.12 & 41.80 & 12.80 & 79.27 & 69.46 & 75.46 & 43.00 & 57.13 \\
& AWQ       & 77.58 & 56.58 & 40.43 & 11.59 & 79.94 & 68.19 & 74.96 & 41.72 & 56.37 \\
& GPTQ      & 77.48 & 56.07 & 38.40 & 12.20 & 78.10 & 69.77 & 73.86 & 40.70 & 55.82 \\
& GPTAQ     & 76.88 & 56.03 & 39.14 & 10.37 & 76.97 & 68.67 & 74.16 & 40.61 & 55.35 \\
& \cellcolor[RGB]{234, 239, 247}SARQC-GS(Identity)  & \cellcolor[RGB]{234, 239, 247}77.80 & \cellcolor[RGB]{234, 239, 247}56.36 & \cellcolor[RGB]{234, 239, 247}42.02 & \cellcolor[RGB]{234, 239, 247}15.24 & \cellcolor[RGB]{234, 239, 247}79.82 & \cellcolor[RGB]{234, 239, 247}68.11 & \cellcolor[RGB]{234, 239, 247}74.33 & \cellcolor[RGB]{234, 239, 247}41.81 & \cellcolor[RGB]{234, 239, 247}56.94 \\
& \cellcolor[RGB]{210, 222, 239}SARQC-GS(Saliency)  & \cellcolor[RGB]{210, 222, 239}78.02 & \cellcolor[RGB]{210, 222, 239}56.33 & \cellcolor[RGB]{210, 222, 239}40.89 & \cellcolor[RGB]{210, 222, 239}16.46 & \cellcolor[RGB]{210, 222, 239}80.15 & \cellcolor[RGB]{210, 222, 239}69.22 & \cellcolor[RGB]{210, 222, 239}75.76 & \cellcolor[RGB]{210, 222, 239}42.32 & \cellcolor[RGB]{210, 222, 239}\textbf{57.39} \\
& \cellcolor[RGB]{234, 239, 247}SARQC-GBS(Identity) & \cellcolor[RGB]{234, 239, 247}77.58 & \cellcolor[RGB]{234, 239, 247}56.16 & \cellcolor[RGB]{234, 239, 247}39.21 & \cellcolor[RGB]{234, 239, 247}14.63 & \cellcolor[RGB]{234, 239, 247}76.70 & \cellcolor[RGB]{234, 239, 247}69.77 & \cellcolor[RGB]{234, 239, 247}75.63 & \cellcolor[RGB]{234, 239, 247}42.24 & \cellcolor[RGB]{234, 239, 247}56.49 \\
& \cellcolor[RGB]{210, 222, 239}SARQC-GBS(Saliency) & \cellcolor[RGB]{210, 222, 239}78.18 & \cellcolor[RGB]{210, 222, 239}56.11 & \cellcolor[RGB]{210, 222, 239}39.90 & \cellcolor[RGB]{210, 222, 239}14.63 & \cellcolor[RGB]{210, 222, 239}79.08 & \cellcolor[RGB]{210, 222, 239}70.32 & \cellcolor[RGB]{210, 222, 239}74.71 & \cellcolor[RGB]{210, 222, 239}43.52 & \cellcolor[RGB]{210, 222, 239}57.06 \\
\midrule

\multirow{8}{*}{LLaMA2-13B}
& FP16      & 79.49 & 60.21 & 52.41 & 18.90 & 82.11 & 72.53 & 78.96 & 47.35 & 61.49 \\
& AWQ       & 78.35 & 59.74 & 51.44 & 14.63 & 81.99 & 72.45 & 78.79 & 45.39 & 60.35 \\
& GPTQ      & 77.58 & 58.14 & 46.25 & 14.02 & 78.96 & 69.46 & 76.81 & 44.97 & 58.27 \\
& GPTAQ     & 77.64 & 58.63 & 46.99 & 14.02 & 79.30 & 71.11 & 77.27 & 45.39 & 58.79  \\
& \cellcolor[RGB]{234, 239, 247}SARQC-GS(Identity)  & \cellcolor[RGB]{234, 239, 247}78.56 & \cellcolor[RGB]{234, 239, 247}60.07 & \cellcolor[RGB]{234, 239, 247}51.64 & \cellcolor[RGB]{234, 239, 247}16.46 & \cellcolor[RGB]{234, 239, 247}81.80 & \cellcolor[RGB]{234, 239, 247}72.77 & \cellcolor[RGB]{234, 239, 247}78.96 & \cellcolor[RGB]{234, 239, 247}46.25 & \cellcolor[RGB]{234, 239, 247}60.81 \\
& \cellcolor[RGB]{210, 222, 239}SARQC-GS(Saliency)  & \cellcolor[RGB]{210, 222, 239}79.11 & \cellcolor[RGB]{210, 222, 239}59.94 & \cellcolor[RGB]{210, 222, 239}51.65 & \cellcolor[RGB]{210, 222, 239}16.46 & \cellcolor[RGB]{210, 222, 239}81.99 & \cellcolor[RGB]{210, 222, 239}73.64 & \cellcolor[RGB]{210, 222, 239}79.17 & \cellcolor[RGB]{210, 222, 239}45.99 & \cellcolor[RGB]{210, 222, 239}\textbf{60.99} \\
& \cellcolor[RGB]{234, 239, 247}SARQC-GBS(Identity) & \cellcolor[RGB]{234, 239, 247}78.02 & \cellcolor[RGB]{234, 239, 247}58.11 & \cellcolor[RGB]{234, 239, 247}46.19 & \cellcolor[RGB]{234, 239, 247}16.46 & \cellcolor[RGB]{234, 239, 247}81.25 & \cellcolor[RGB]{234, 239, 247}70.48 & \cellcolor[RGB]{234, 239, 247}77.53 & \cellcolor[RGB]{234, 239, 247}45.90 & \cellcolor[RGB]{234, 239, 247}59.24 \\
& \cellcolor[RGB]{210, 222, 239}SARQC-GBS(Saliency) & \cellcolor[RGB]{210, 222, 239}78.13 & \cellcolor[RGB]{210, 222, 239}59.30 & \cellcolor[RGB]{210, 222, 239}48.13 & \cellcolor[RGB]{210, 222, 239}15.24 & \cellcolor[RGB]{210, 222, 239}80.34 & \cellcolor[RGB]{210, 222, 239}71.90 & \cellcolor[RGB]{210, 222, 239}78.24 & \cellcolor[RGB]{210, 222, 239}46.84 & \cellcolor[RGB]{210, 222, 239}59.77 \\
\midrule

\multirow{7}{*}{DeepSeek-MoE-16B}
& BF16      & 78.73 & 58.10 & 38.21 & 27.44 & 74.07 & 70.01 & 74.87 & 43.94 & 58.17 \\
& AWQ       & 78.24 & 57.06 & 34.38 & 20.12 & 61.41 & 67.25 & 70.96 & 40.44 & 53.73 \\
& GPTQ      & 78.67 & 57.06 & 37.55 & 24.39 & 73.43 & 69.38 & 73.70 & 42.83 & 57.13 \\
& \cellcolor[RGB]{234, 239, 247}SARQC-GS(Identity)  & \cellcolor[RGB]{234, 239, 247}78.56 & \cellcolor[RGB]{234, 239, 247}57.66 & \cellcolor[RGB]{234, 239, 247}36.55 & \cellcolor[RGB]{234, 239, 247}25.00 & \cellcolor[RGB]{234, 239, 247}66.51 & \cellcolor[RGB]{234, 239, 247}67.96 & \cellcolor[RGB]{234, 239, 247}71.51 & \cellcolor[RGB]{234, 239, 247}41.55 & \cellcolor[RGB]{234, 239, 247}55.66 \\
& \cellcolor[RGB]{210, 222, 239}SARQC-GS(Saliency)  & \cellcolor[RGB]{210, 222, 239}78.94 & \cellcolor[RGB]{210, 222, 239}58.33 & \cellcolor[RGB]{210, 222, 239}36.45 & \cellcolor[RGB]{210, 222, 239}23.78 & \cellcolor[RGB]{210, 222, 239}67.68 & \cellcolor[RGB]{210, 222, 239}68.90 & \cellcolor[RGB]{210, 222, 239}72.90 & \cellcolor[RGB]{210, 222, 239}42.58 & \cellcolor[RGB]{210, 222, 239}56.20 \\
& \cellcolor[RGB]{234, 239, 247}SARQC-GBS(Identity) & \cellcolor[RGB]{234, 239, 247}78.24 & \cellcolor[RGB]{234, 239, 247}57.59 & \cellcolor[RGB]{234, 239, 247}36.23 & \cellcolor[RGB]{234, 239, 247}23.78 & \cellcolor[RGB]{234, 239, 247}73.52 & \cellcolor[RGB]{234, 239, 247}70.24 & \cellcolor[RGB]{234, 239, 247}74.54 & \cellcolor[RGB]{234, 239, 247}43.34 & \cellcolor[RGB]{234, 239, 247}57.19 \\
& \cellcolor[RGB]{210, 222, 239}SARQC-GBS(Saliency) & \cellcolor[RGB]{210, 222, 239}78.89 & \cellcolor[RGB]{210, 222, 239}57.54 & \cellcolor[RGB]{210, 222, 239}37.32 & \cellcolor[RGB]{210, 222, 239}25.00 & \cellcolor[RGB]{210, 222, 239}74.04 & \cellcolor[RGB]{210, 222, 239}69.77 & \cellcolor[RGB]{210, 222, 239}74.33 & \cellcolor[RGB]{210, 222, 239}43.43 & \cellcolor[RGB]{210, 222, 239}\textbf{57.54} \\
\midrule

\multirow{7}{*}{Qwen3-MoE-30B}
& BF16      & 79.82 & 62.35 & 78.76 & 50.61 & 80.61 & 72.14 & 80.05 & 54.01 & 69.79 \\
& AWQ       & 79.54 & 60.05 & 75.73 & 50.00 & 79.63 & 70.56 & 79.50 & 49.74 & 68.09 \\
& GPTQ      & 78.62 & 58.74 & 76.99 & 50.00 & 79.69 & 69.93 & 78.24 & 51.37 & 67.95 \\
& \cellcolor[RGB]{234, 239, 247}SARQC-GS(Identity)  & \cellcolor[RGB]{234, 239, 247}78.89 & \cellcolor[RGB]{234, 239, 247}60.19 & \cellcolor[RGB]{234, 239, 247}75.64 & \cellcolor[RGB]{234, 239, 247}51.22 & \cellcolor[RGB]{234, 239, 247}80.49 & \cellcolor[RGB]{234, 239, 247}70.80 & \cellcolor[RGB]{234, 239, 247}79.63 & \cellcolor[RGB]{234, 239, 247}48.46 & \cellcolor[RGB]{234, 239, 247}68.16 \\
& \cellcolor[RGB]{210, 222, 239}SARQC-GS(Saliency)  & \cellcolor[RGB]{210, 222, 239}79.76 & \cellcolor[RGB]{210, 222, 239}60.43 & \cellcolor[RGB]{210, 222, 239}75.59 & \cellcolor[RGB]{210, 222, 239}52.44 & \cellcolor[RGB]{210, 222, 239}80.73 & \cellcolor[RGB]{210, 222, 239}71.11 & \cellcolor[RGB]{210, 222, 239}80.09 & \cellcolor[RGB]{210, 222, 239}50.34 & \cellcolor[RGB]{210, 222, 239}68.81 \\
& \cellcolor[RGB]{234, 239, 247}SARQC-GBS(Identity) & \cellcolor[RGB]{234, 239, 247}80.85 & \cellcolor[RGB]{234, 239, 247}60.31 & \cellcolor[RGB]{234, 239, 247}77.46 & \cellcolor[RGB]{234, 239, 247}53.66 & \cellcolor[RGB]{234, 239, 247}79.02 & \cellcolor[RGB]{234, 239, 247}72.77 & \cellcolor[RGB]{234, 239, 247}80.39 & \cellcolor[RGB]{234, 239, 247}53.07 & \cellcolor[RGB]{234, 239, 247}69.69 \\
& \cellcolor[RGB]{210, 222, 239}SARQC-GBS(Saliency) & \cellcolor[RGB]{210, 222, 239}81.18 & \cellcolor[RGB]{210, 222, 239}61.67 & \cellcolor[RGB]{210, 222, 239}77.01 & \cellcolor[RGB]{210, 222, 239}58.54 & \cellcolor[RGB]{210, 222, 239}81.10 & \cellcolor[RGB]{210, 222, 239}73.09 & \cellcolor[RGB]{210, 222, 239}81.10 & \cellcolor[RGB]{210, 222, 239}52.65 & \cellcolor[RGB]{210, 222, 239}\textbf{70.79} \\
\midrule

\multirow{7}{*}{Mixtral-8x7B}
& BF16      & 82.54 & 65.10 & 68.12 & 39.63 & 85.90 & 77.19 & 83.71 & 56.74 & 69.87 \\
& AWQ       & 81.99 & 64.21 & 67.32 & 35.98 & 84.22 & 76.09 & 83.42 & 56.48 & 68.71 \\
& GPTQ      & 80.96 & 63.68 & 66.05 & 35.98 & 85.20 & 76.09 & 82.32 & 55.46 & 68.22 \\
& \cellcolor[RGB]{234, 239, 247}SARQC-GS(Identity)  & \cellcolor[RGB]{234, 239, 247}82.15 & \cellcolor[RGB]{234, 239, 247}64.28 & \cellcolor[RGB]{234, 239, 247}66.93 & \cellcolor[RGB]{234, 239, 247}34.15 & \cellcolor[RGB]{234, 239, 247}84.10 & \cellcolor[RGB]{234, 239, 247}76.80 & \cellcolor[RGB]{234, 239, 247}83.21 & \cellcolor[RGB]{234, 239, 247}56.91 & \cellcolor[RGB]{234, 239, 247}68.57 \\
& \cellcolor[RGB]{210, 222, 239}SARQC-GS(Saliency)  & \cellcolor[RGB]{210, 222, 239}82.26 & \cellcolor[RGB]{210, 222, 239}64.33 & \cellcolor[RGB]{210, 222, 239}67.51 & \cellcolor[RGB]{210, 222, 239}37.20 & \cellcolor[RGB]{210, 222, 239}84.43 & \cellcolor[RGB]{210, 222, 239}76.64 & \cellcolor[RGB]{210, 222, 239}83.29 & \cellcolor[RGB]{210, 222, 239}57.59 & \cellcolor[RGB]{210, 222, 239}\textbf{69.16} \\
& \cellcolor[RGB]{234, 239, 247}SARQC-GBS(Identity) & \cellcolor[RGB]{234, 239, 247}81.39 & \cellcolor[RGB]{234, 239, 247}64.06 & \cellcolor[RGB]{234, 239, 247}66.21 & \cellcolor[RGB]{234, 239, 247}35.37 & \cellcolor[RGB]{234, 239, 247}85.69 & \cellcolor[RGB]{234, 239, 247}76.32 & \cellcolor[RGB]{234, 239, 247}82.49 & \cellcolor[RGB]{234, 239, 247}54.69 & \cellcolor[RGB]{234, 239, 247}68.28 \\
& \cellcolor[RGB]{210, 222, 239}SARQC-GBS(Saliency) & \cellcolor[RGB]{210, 222, 239}82.05 & \cellcolor[RGB]{210, 222, 239}64.42 & \cellcolor[RGB]{210, 222, 239}66.83 & \cellcolor[RGB]{210, 222, 239}35.98 & \cellcolor[RGB]{210, 222, 239}85.57 & \cellcolor[RGB]{210, 222, 239}77.35 & \cellcolor[RGB]{210, 222, 239}82.87 & \cellcolor[RGB]{210, 222, 239}56.06 & \cellcolor[RGB]{210, 222, 239}68.89 \\
\bottomrule
\end{tabular}
}
\end{table*}
\vspace{-8pt}

\emph{(2) Does SARQC improve robustness and overall performance, especially when vanilla calibration degrades?}
\textbf{Yes.} SARQC consistently delivers stronger robustness and better task performance precisely in the regimes where vanilla calibration becomes unstable. In terms of perplexity, SARQC remains much more stable than standard calibration methods across both dense and MoE LLMs, particularly in the hardest low-bit settings. For example, under W2A16 in \Cref{tab:ppl_all_wxa16}, SARQC reduces the perplexity of LLaMA2-13B to $117.67$, compared with $792.32$ for GPTQ and $176.53$ for GPTAQ. A similar pattern also holds for MoE models. On Qwen3-MoE-30B, SARQC reduces perplexity to $16.69$, compared with $38.03$ for GPTQ. This robustness further translates into performance gains on downstream tasks. Under W4A16, SARQC achieves higher average zero-shot accuracy than the baselines on both dense and MoE models in \Cref{tab:rqc_zeroshot_acc_w4a16_others}. For instance, \texttt{SARQC-GS(Saliency)} attains the best average accuracy on both LLaMA2-7B and LLaMA2-13B, and \texttt{SARQC-GBS(Saliency)} achieves the best average accuracy on both DeepSeek-MoE-16B and Qwen3-MoE-30B. Similar performance is also observed under W3A16, as reported in \Cref{tab:rqc_zeroshot_acc_w3a16_others}. Moreover, the ablation study on the calibration sample size in \Cref{fig:main_and_trend}\emph{(b)} shows that \texttt{SARQC-GBS} consistently outperforms GPTQ under both W3A16 and W2A16, with a larger gap at smaller sample sizes. This further confirms that SARQC performs well not only under low-bit scenarios but also under calibration data scarcity.

\begin{figure}[t]
    \centering
    \begin{subfigure}[t]{0.68\linewidth}
        \centering
        \includegraphics[width=\linewidth]{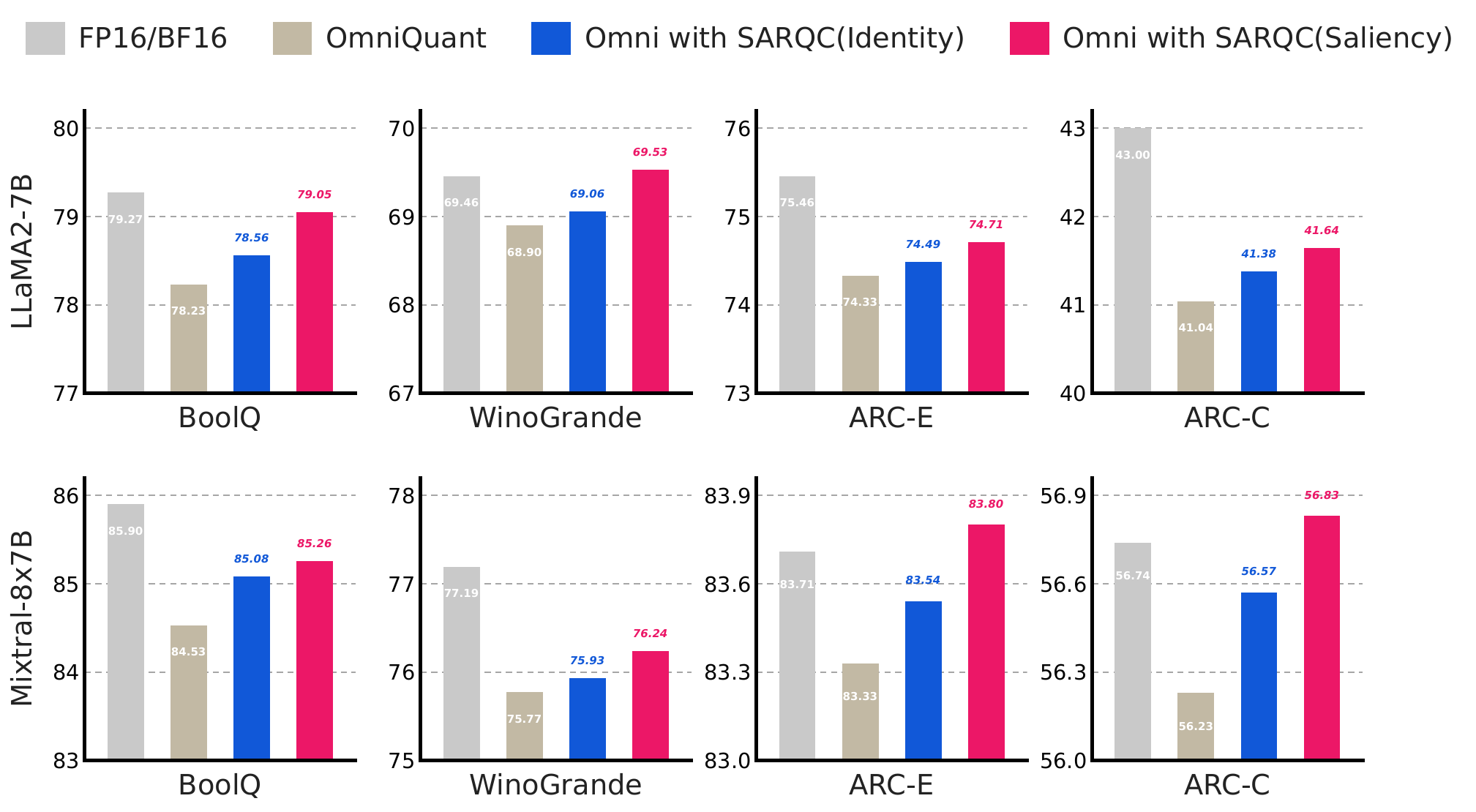}
                \caption{}
        \label{fig:main_and_trend_a}
    \end{subfigure}\hfill
    \begin{subfigure}[t]{0.31\linewidth}
        \centering
        \includegraphics[width=\linewidth]{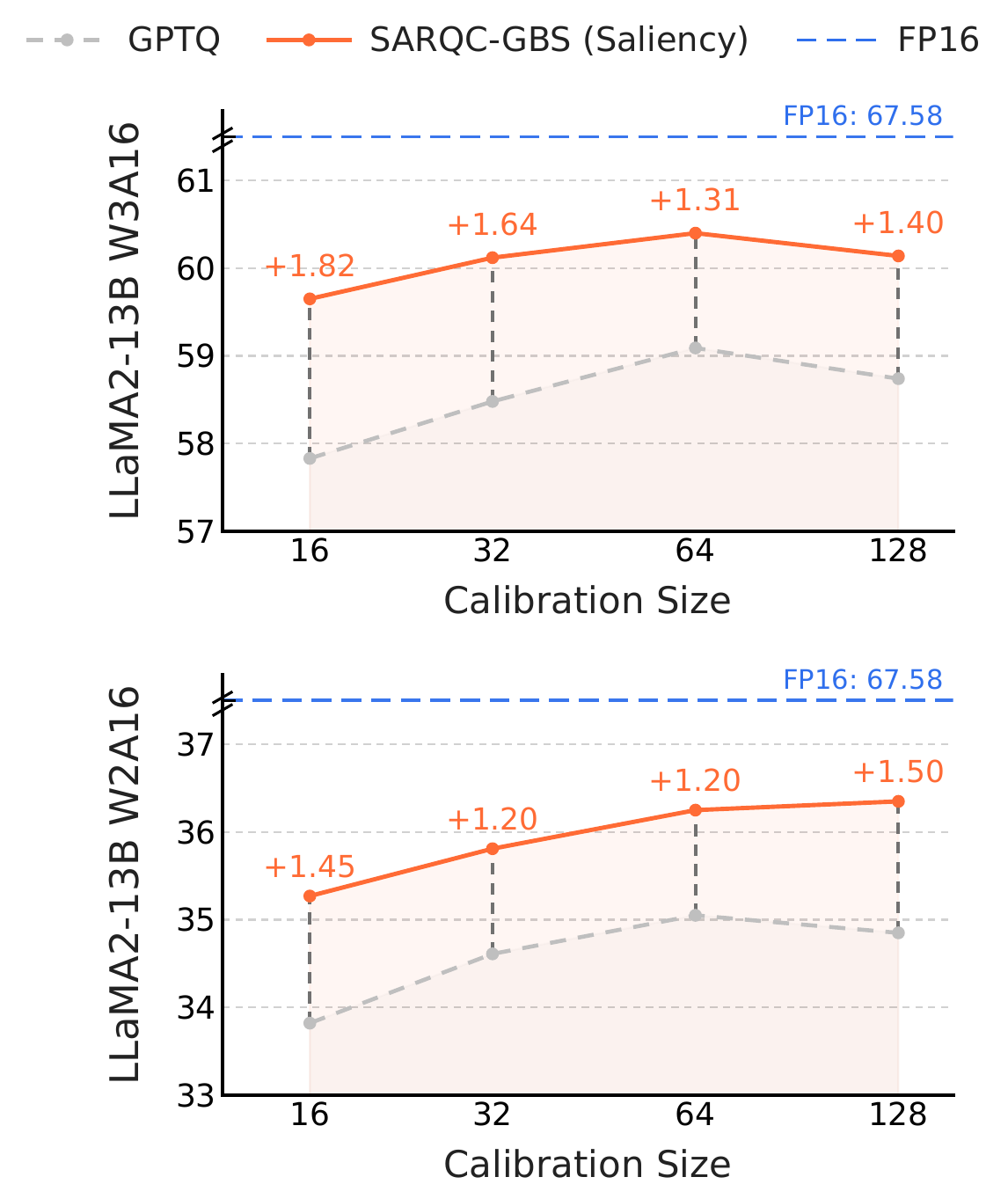}
        \caption{}
        \label{fig:main_and_trend_b}
    \end{subfigure}
\caption{Ablation study. \emph{(a) Extension to OmniQuant}: the y-axis shows downstream task accuracy, and the bars compare FP16/BF16, OmniQuant, and OmniQuant augmented with SARQC on the reported tasks for LLaMA2-7B and Mixtral-8x7B under W4A16. \emph{(b) Effect of calibration size}: the x-axis shows the number of calibration samples, and the y-axis shows average downstream accuracy. The vertical dashed segments indicate the accuracy gap between GPTQ and \texttt{SARQC-GBS}.}
    \label{fig:main_and_trend}
    \vspace{-18pt}
\end{figure}

\emph{(3) Is saliency-aware regularization more effective than its non-saliency-aware counterpart?}
\textbf{Overall, yes.} The saliency-aware variant consistently performs better in practice than the corresponding identity-based variant. From \Cref{tab:rqc_zeroshot_acc_w4a16_others}, the saliency-aware regularizer yields higher average accuracy than its identity-based counterpart on all five evaluated models for the grid-search solver, and on four out of five models for the Gram-based solver. The same trend is also supported by the perplexity results in \Cref{tab:ppl_all_wxa16}, where the saliency-aware variants are generally more competitive, especially in the more difficult low-bit settings. These results indicate that explicitly accounting for weight saliency helps preserve information better during calibration and leads to better performance.

\emph{(4) Can SARQC be extended further?}
\textbf{Yes.} SARQC is not tied to a particular solver or calibration pipeline, but instead acts as a general regularization principle that can be incorporated into other post-training quantization frameworks. This is evidenced by the results in \Cref{fig:main_and_trend}\emph{(a)}, where integrating SARQC into OmniQuant \citep{omniquant} consistently improves performance across different architectures and downstream tasks. On both LLaMA2-7B and Mixtral-8x7B, SARQC consistently outperforms the original OmniQuant baseline, and the saliency-aware variant achieves the best overall performance in all cases. These findings demonstrate that SARQC can be widely generalized as a practical plug-in module for other PTQ methods without introducing additional computational overhead during inference; see also \Cref{tab:speed_comparison} (in Appendix~\ref{appendix: speedup}) for the ablation study on the inference speedup.

\section{Conclusion}
In this work, we propose RQC and its extension SARQC, a unified framework for weight-only PTQ of LLMs that is simple, effective, and well-motivated. RQC and SARQC augment standard quantization calibration objectives with a penalty that explicitly balances reconstruction fidelity and weight drift from the original floating-point weights, arising from the generalization-risk analysis and constrained optimization formulation. Two practical optimization algorithms are proposed: one built on grid-search over scaling factors, and the other built on Gram-based solvers. RQC and SARQC are easy to implement, incur no extra computational overhead during inference, and can be seamlessly integrated into existing PTQ pipelines. Extensive experimental results show the effectiveness of the proposed framework.

This work also has several limitations. For instance, we do not evaluate RQC/SARQC on extremely large LLMs due to limited computational resources. It is possible the current design of saliency factors is not optimal as they are inspired from previous works. There are a number of possible extensions. Firstly, it would be valuable to evaluate the performance of RQC/SARQC on extremely large LLMs. Secondly, it may be worthwhile to extend our method to account for error propagation in the PTQ calibration process, since both the proposed method and the theoretical analysis are conducted in a layerwise manner. Thirdly, it would be interesting to investigate how well the method generalizes to weight–activation quantization scenarios.

\bibliographystyle{plain}
\bibliography{example_paper}

\newpage
\appendix
\onecolumn

\begin{center}
    {\Large \textbf{Appendix}} 
\end{center}

\section*{Appendix Overview}

This appendix offers a complete, self-contained elaboration of the theoretical underpinnings, proofs, implementation details, algorithms, and supplementary experimental results for the proposed RQC/SARQC framework.

Appendix~\ref{app: Notation} summarizes the main notation used throughout the paper and appendix.

Appendix~\ref{appendix: More Comprehensive Literature Review} presents a more comprehensive literature review on post-training quantization for LLMs and constrained optimization, complementing the discussion in \Cref{sec:background}.

Appendix~\ref{app:technical_def} provides additional preliminaries on quantization, most relevant scaling-based and Gram-based PTQ methods, and constrained optimization notions.

Appendix~\ref{appendix: Additional Theoretical Results} presents additional theoretical results. In particular, Appendix~\ref{app:derivation_rqc_gbs} derives the quadratic form of the SARQC objective and shows how it leads to a GPTQ-style row-wise compensation rule with a regularized curvature matrix. Appendix~\ref{append: proof of generalization risk} gives the proof of the generalization-risk bound in \Cref{thm:generalization risk}. Appendix~\ref{append:cor_constraint_penalty} discusses the relationship between hard drift constraints and quadratic penalty formulations for supported quantized solutions.

Appendix~\ref{app:pseudo algorithms and implementation details} reports the pseudo-algorithms and implementation details of SARQC. We describe both SARQC-GS, which performs regularized model selection over scaling factors, and SARQC-GBS, which modifies the Gram-based sequential quantization curvature using saliency-aware regularization. We also describe how SARQC is incorporated into OmniQuant in our ablation study.

Appendix~\ref{appx:extra experiments} provides supplementary experimental results. Appendix~\ref{appx:detail_of_figure1} gives the detailed results underlying \Cref{fig:illustration}. Appendix~\ref{app:3d_weight_diff} visualizes weight drift under different regularization strengths. The remaining subsections report additional results under W3A16, results on the LLaMA model family, sensitivity analyses with respect to calibration size and calibration corpus, and practical inference speedup comparisons.

\newpage
\section{Notation}
\label{app: Notation}
\Cref{tab:notation} summarizes the main notation used throughout the paper and appendix.

\begin{table}[hbtp]
\centering
\caption{Summary of main notation.}
\label{tab:notation}
\setlength{\tabcolsep}{6pt}
\renewcommand{\arraystretch}{1.12}
\resizebox{0.98\linewidth}{!}{%
\begin{tabular}{ll}
\toprule
Notation & Description \\
\midrule

$\mathbb{R}$ 
& Set of real numbers. \\

$\mathbb{N}^{+}$ 
& Set of positive integers. \\

$p_X$
& Downstream data distribution of the random input $X$. \\

$\mathcal{Q}$ 
& Feasible set of quantized--dequantized weights induced by a given quantization scheme. \\

$\mathcal{Q}_R$
& Restricted finite quantization class
$\{\widehat{\mathbf W}_l\in\mathcal Q:\|\widehat{\mathbf W}_l-\mathbf W_l\|_\mathrm{F}\le R\}$. \\

$\mathbf{W}_l \in \mathbb{R}^{d_{\mathrm{out}}\times d_{\mathrm{in}}}$ 
& Float-point (FP) weight matrix of the $l$-th linear layer (FP16/BF16). \\

$\widehat{\mathbf{W}}_l \in \mathcal{Q}$ 
& Quantized--dequantized weight matrix of the $l$-th linear layer. \\

$\Delta \mathbf{W}_l := \widehat{\mathbf{W}}_l-\mathbf{W}_l$ 
& Weight perturbation induced by quantization. \\
$\delta_{\mathbf{W}_l}$ 
& Dirac measure supported at the point $\mathbf{W}_l$ in the weight space. \\

$\delta_{\widehat{\mathbf{W}}_l}$ 
& Dirac measure supported at the point $\widehat{\mathbf{W}}_l$ in the weight space. \\
$\mathbf{X}_l \in \mathbb{R}^{d_{\mathrm{in}}\times n}$ 
& Calibration input matrix of the $l$-th layer, where $n$ is the number of calibration tokens or samples. \\

$X$
& Random input drawn from the downstream data distribution $p_X$. \\

$Q(\cdot)$ 
& Quantization operator together with the corresponding dequantization map. \\

$\eta$ 
& Quantization step size. \\

$\widetilde{\mathbf{S}}_l(\alpha)$ 
& Channel-wise scaling matrix used to generate candidate quantized weights in scaling-based PTQ methods. \\

$\tilde{s}_l^{(j)}(\alpha)$ 
& Channel-wise scaling factor for channel $j$. \\

$\alpha$ 
& Scaling hyperparameter in grid-search-based PTQ methods such as AWQ and SARQC-GS. \\

$\mathbf{S}_l=\mathrm{diag}(s_l)$ 
& Diagonal saliency matrix used in the SARQC regularization term. \\

$s_l^{(j)}$ 
& Saliency weight of channel $j$. \\

$s_l^{(j)}(\gamma)$ 
& Saliency weight of channel $j$, parameterized by $\gamma$ in SARQC-GBS. \\

$\gamma$ 
& Hyperparameter controlling the balance between activation and weight statistics in the saliency profile. \\

$\lambda$ 
& Regularization coefficient balancing reconstruction fidelity and weight discrepancy. \\

$R$
& Radius parameter defining the restricted class $\mathcal Q_R$. \\

$M_X$
& Upper bound on the input norm, i.e., $\|X\|_2 \le M_X$. \\

$\mathbf{G}_l := \mathbf{X}_l\mathbf{X}_l^\top + \lambda \mathbf{S}_l \mathbf{S}_l^\top$ 
& Regularized Gram-based matrix used in the SARQC objective. \\

$\mathcal{R}(\widehat{\mathbf{W}}_l)$
& True downstream reconstruction risk,
$\mathbb{E}_{X\sim p_X}\!\left[\|\Delta \mathbf{W}_l X\|_2^2\right]$. \\

$\widehat{\mathcal{R}}_{\mathrm{cal}}(\widehat{\mathbf{W}}_l)$
& Empirical calibration reconstruction risk on the calibration set. \\


$I$ 
& Identity matrix. \\

\bottomrule
\end{tabular}%
}
\end{table}

\newpage
\section{More Comprehensive Literature Review}
\label{appendix: More Comprehensive Literature Review}

\paragraph{Post-training quantization} 
Post-training quantization (PTQ)\citep{nagel2021whitepaperneuralnetwork,wu2020easyquantposttrainingquantizationscale} has evolved from a lightweight deployment technique into a central compression paradigm for large language models (LLMs). Early optimization-based PTQ methods showed that naive round-to-nearest quantization is often suboptimal. AdaRound \citep{nagel2020adaround} formulates rounding as a data-dependent local optimization, while BRECQ \citep{li2021brecq} extends this view from individual layers to block reconstruction, balancing local reconstruction fidelity with cross-layer dependency. Second-order methods further strengthened this line of work. Optimal Brain Compression revisits the Optimal Brain Surgeon framework for one-shot pruning and quantization \citep{frantar2022optimal}.

For LLMs, however, quantization is complicated by activation and weight outliers, large inter-layer dependencies, and the distinction between weight-only and weight-activation quantization. Weight-only methods such as GPTQ \citep{frantar2023gptq}, AWQ \citep{lin2024awq}, SpQR \citep{dettmersspqr}, and OWQ \citep{Lee_Jin_Kim_Kim_Park_2024} mainly reduce memory bandwidth while keeping activations in higher precision.  Recent PTQ methods have also moved beyond plain layer-wise MSE by incorporating additional optimization criteria, including prediction-difference calibration \cite{LiuEtAl2023PDQuant}, contrastive or mutual-information-based calibration \cite{ShangEtAl2024CLCalib}, output-adaptive Hessian or end-loss-guided objectives \cite{EdalatiEtAl2025OAC,KimEtAl2025GuidedQuant}, statistical pre-calibration based on distributional discrepancy \cite{GhaffariEtAl2025PreCalibration}, mixed-metric reconstruction regularization \cite{WeiYanWang2025MPPQ}, and explicit regularized calibration with successive rounding \cite{ChaEtAl2026RegularizedCalibration}.

\paragraph{Constrained optimization}

Constrained and penalized formulations are a classical mechanism for balancing data fidelity against structural preference or model complexity. 
In statistics, ridge regression and the lasso are canonical examples in which a constrained estimator can be expressed through a penalized objective with a tuning parameter controlling the trade-off between empirical fit and regularity \cite{HoerlKennard1970,Tibshirani1996}. 
In convex optimization, this connection is typically formalized through Lagrange multipliers, KKT conditions, and duality theory \cite{BoydVandenberghe2004,NocedalWright2006}. 
Recent constrained-optimization literature emphasizes that replacing constraints by penalties is an exact reformulation only when a suitable multiplier or exactness condition exists; otherwise, the penalized problem should be interpreted as a relaxation or scalarized surrogate. 
This distinction is particularly relevant for post-training quantization, since the feasible set $\mathcal Q$ is discrete and generally nonconvex. 
Recent work on exact penalties and Lagrangian relaxation in constrained and mixed-integer optimization makes this multiplier-existence issue explicit \cite{Bragin2024SurveyLR,LiaoYuanGao2024,DiouaneGollierOrban2026}.

Our supportedness condition follows the same principle from multiobjective optimization. 
The reconstruction error and the discrepancy from the floating-point weights define two competing objectives, and the penalized objective is a weighted-sum scalarization whose exact recovery requires the selected point to be supported \cite{HelfrichEtAl2024,KoenenStiglmayr2025}. 
This viewpoint is natural for quantization calibration, where one balances output fidelity against proximity to the original FP weights.

\newpage
\section{Additional Preliminaries}
\label{app:technical_def}
This section reviews several preliminaries used throughout the paper. We first summarize basic quantization formulations, together with representative scaling-based and Gram-based PTQ methods, since the proposed SARQC framework extends these standard calibration objectives and optimization strategies in \Cref{sec:background,sec:method}. We then review the constrained optimization basics used in the theoretical discussion, which provide background for part of the motivation and analysis in \Cref{sec:rqc}. These preliminaries are included to establish notation for the derivations, proofs, and implementation details in the appendix.
\subsection{Quantization and Dequantization}
\label{app:prelim_quant_and_dequant}

Quantization maps high-precision weights or activations to a discrete low-bit set to reduce memory and improve inference efficiency. 
Given a FP weight matrix $\mathbf{W}\in\mathbb{R}^{d_{\mathrm{out}}\times d_{\mathrm{in}}}$, we denote its quantized--dequantized approximation as $\widehat{\mathbf{W}}=Q(\mathbf{W})$, where $Q(\cdot)$ represents the quantization operator together with the corresponding dequantization map.

For a scalar value $w$, a standard formulation is uniform affine quantization:
\begin{equation}
q=\mathrm{clip}\!\left(\left\lfloor \frac{w}{\eta}+z \right\rceil,\; q_{\min},\; q_{\max}\right),
\qquad
\widehat{w}=\eta\,(q-z),
\end{equation}
where $\eta>0$ is the quantization step size, $z$ is the zero-point, and $\lfloor\cdot\rceil$ denotes rounding to the nearest integer. The integer $q$ lies in a finite range $[q_{\min}, q_{\max}]$ determined by the bit width. In practice, quantization is applied elementwise to vectors, matrices, or tensors, with parameters shared at different granularities such as per-tensor, per-channel, or per-group.

\subsection{Post-Training Quantization: AWQ and GPTQ}

As our method is closely related to AWQ \cite{lin2024awq} and GPTQ \cite{frantar2023gptq}, we now present details of these two methods.

\paragraph{AWQ:}
Activation-Aware Weight Quantization (AWQ) is a weight-only PTQ method that improves quantization by applying channel-wise scaling prior to quantization \citep{lin2024awq}. For the $l$-th layer with weights $\mathbf{W}_l$ and calibration inputs $\mathbf{X}_l$, AWQ rescales input channels using a positive vector $\tilde{s}_l(\alpha)$, quantizes the rescaled weights, and compensates the scaling on the input side. This leads to the reconstruction objective
\begin{align}
\min_{\alpha}
\;\;
\bigl\|
\mathbf{W}_l \mathbf{X}_l
-
Q\!\bigl(\mathbf{W}_l \mathrm{diag}(\tilde{s}_l(\alpha))\bigr)
\mathrm{diag}(\tilde{s}_l(\alpha))^{-1}\mathbf{X}_l
\bigr\|_\mathrm{F}^2.
\end{align}

The scaling factors are parameterized using simple calibration statistics:
$$
\tilde{s}_l^{(j)}(\alpha)
=
\frac{\mathrm{mean}(|\mathbf{X}_l^{(j)}|)^{\alpha}}
{\mathrm{mean}(|\mathbf{W}_l^{(j)}|)^{1-\alpha}},
$$
and the final quantized weights are selected via a lightweight grid search over $\alpha$. This mechanism effectively protects activation-salient channels.

\paragraph{GPTQ:}
GPTQ formulates PTQ as a Gram-based reconstruction problem \citep{frantar2023gptq}. It minimizes the layer-wise error
$$
\bigl\|
\mathbf{W}_l \mathbf{X}_l - \widehat{\mathbf{W}}_l \mathbf{X}_l
\bigr\|_\mathrm{F}^2,
$$
which admits a quadratic form governed by the Gram matrix $\mathbf{X}_l \mathbf{X}_l^\top.$ This matrix serves as a surrogate curvature, capturing the sensitivity of the reconstruction loss.

GPTQ proceeds by sequentially quantizing weights and compensating the induced error using inverse Gram-based information. Combined with efficient implementations (e.g., blockwise updates and Cholesky-based inversion), this yields a scalable and accurate PTQ method for large language models.

\subsection{Optimization: Constraints and Penalties}
We briefly review the relationship between constrained and penalized formulations, which is a classical tool for expressing trade-offs between empirical fidelity and structural preference.

\begin{definition}[Constrained formulation]
Let $F(x)$ be a primary objective and let $R(x)$ measure the complexity, regularity, or deviation of $x$.  
A constrained formulation takes the form
$$
\min_{x\in\mathcal X} F(x)
\quad
\mathrm{s.t.}
\quad
R(x)\le \tau,
$$
where $\tau$ specifies the admissible level of regularity or complexity.
\end{definition}

\begin{definition}[Penalized formulation]
The corresponding penalized formulation is
$$
\min_{x\in\mathcal X}
F(x)+\lambda R(x),
$$
where $\lambda\ge 0$ controls the trade-off between the primary objective and the penalty term.
\end{definition}

Such constrained--penalized pairs are standard in statistics. For example, ridge regression and the lasso can be viewed either as constrained estimators or as penalized estimators, where the tuning parameter controls the trade-off between data fitting and regularity \citep{HoerlKennard1970,Tibshirani1996}. In convex optimization, this relationship is usually formalized through the Lagrangian
$$
\mathcal L(x,\lambda)
=
F(x)+\lambda\bigl(R(x)-\tau\bigr),
\qquad
\lambda\ge 0.
$$
Under suitable regularity conditions, KKT conditions, and strong duality, a solution of the constrained problem can be recovered by minimizing a penalized Lagrangian objective with an appropriate multiplier \citep{BoydVandenberghe2004,NocedalWright2006}.

However, this equivalence is not automatic. A penalized problem is an exact reformulation of the constrained problem only when a suitable multiplier or exactness condition exists. Otherwise, the penalized objective should be interpreted as a relaxation or scalarized surrogate rather than a fully equivalent problem. This issue is particularly important in nonconvex, nonsmooth, or mixed-integer optimization, where strong duality may fail and a nonzero duality gap can appear. Recent work on exact penalties and Lagrangian relaxation makes this multiplier-existence issue explicit \citep{Bragin2024SurveyLR,LiaoYuanGao2024,DiouaneGollierOrban2026}.

The same phenomenon can also be understood from the perspective of multiobjective optimization. Consider two competing objectives,
$$
F_1(x)
\quad\text{and}\quad
F_2(x).
$$
A weighted-sum scalarization solves
$$
\min_{x\in\mathcal X}
F_1(x)+\lambda F_2(x),
\qquad
\lambda\ge 0.
$$
This scalarized problem does not necessarily recover every Pareto-optimal solution.

\begin{definition}[Supported solution]
A feasible point $x^\star\in\mathcal X$ is called supported if there exists a weight $\lambda\ge 0$ such that
$$
x^\star
\in
\arg\min_{x\in\mathcal X}
F_1(x)+\lambda F_2(x).
$$
Equivalently, the objective vector
$
(F_1(x^\star),F_2(x^\star))
$
lies on a part of the attainable objective frontier that can be touched by a supporting hyperplane.
\end{definition}

Weighted-sum scalarization can recover supported Pareto solutions, but it can miss unsupported Pareto solutions, especially when the attainable objective set is discrete or nonconvex \citep{HelfrichEtAl2024,KoenenStiglmayr2025}. Therefore, the existence of a penalty parameter that exactly recovers a constrained solution is closely related to whether the target solution is supported.

In summary, constrained and penalized formulations express the same trade-off from two different viewpoints. Their exact equivalence requires additional structural conditions, such as the existence of a suitable Lagrange multiplier or supportedness of the selected trade-off point. Without such conditions, the penalized objective remains a useful and tunable surrogate, but not necessarily an exact reformulation of the constrained problem.

\clearpage

\section{Additional Theoretical Results}
\label{appendix: Additional Theoretical Results}

\subsection{Derivations of the Update for SARQC-GBS}
\label{app:derivation_rqc_gbs}

In this subsection, we derive the quadratic form underlying SARQC and show that it leads to a GPTQ-style row-wise compensation rule under a regularized curvature matrix.

\paragraph{Quadratic reformulation of the SARQC objective}
Recall the layer-wise SARQC objective
$$
\min_{\widehat{\mathbf{W}}_l \in \mathcal{Q}}
\;
\|\mathbf{W}_l\mathbf{X}_l-\widehat{\mathbf{W}}_l\mathbf{X}_l\|_\mathrm{F}^2
+
\lambda \|(\widehat{\mathbf{W}}_l-\mathbf{W}_l)\mathbf{S}_l\|_\mathrm{F}^2 .
$$
Let
$$
\Delta \mathbf{W}_l := \widehat{\mathbf{W}}_l - \mathbf{W}_l .
$$
Then the reconstruction term can be written as
$$
\|\mathbf{W}_l\mathbf{X}_l-\widehat{\mathbf{W}}_l\mathbf{X}_l\|_\mathrm{F}^2
=
\|\Delta \mathbf{W}_l \mathbf{X}_l\|_\mathrm{F}^2
=
\mathrm{Tr}\ \!\left(\Delta \mathbf{W}_l \mathbf{X}_l\mathbf{X}_l^\top \Delta \mathbf{W}_l^\top\right),
$$
and the regularization term becomes
$$
\|(\widehat{\mathbf{W}}_l-\mathbf{W}_l)\mathbf{S}_l\|_\mathrm{F}^2
=
\|\Delta \mathbf{W}_l \mathbf{S}_l\|_\mathrm{F}^2
=
\mathrm{Tr}\ \!\left(\Delta \mathbf{W}_l \mathbf{S}_l\mathbf{S}_l^\top \Delta \mathbf{W}_l^\top\right).
$$
Therefore, the SARQC objective reduces to the quadratic form
$$
\mathrm{Tr}\ \!\left(\Delta \mathbf{W}_l \mathbf{G}_l \Delta \mathbf{W}_l^\top\right),
\qquad
\mathbf{G}_l := \mathbf{X}_l\mathbf{X}_l^\top + \lambda \mathbf{S}_l\mathbf{S}_l^\top .
$$

\paragraph{Row-wise decomposition.}
Since the trace is additive over rows, writing $\Delta \mathbf{W}_l$ row-wise yields
$$
\mathrm{Tr}\ \!\left(\Delta \mathbf{W}_l \mathbf{G}_l \Delta \mathbf{W}_l^\top\right)
=
\sum_{r=1}^{d_{\mathrm{out}}}
\Delta \mathbf{w}_{l,r}^\top \mathbf{G}_l \Delta \mathbf{w}_{l,r},
$$
where $\Delta \mathbf{w}_{l,r}$ is the $r$-th row of $\Delta \mathbf{W}_l$. Hence, the optimization decomposes across output rows, as in GPTQ, with the curvature matrix replaced by $\mathbf{G}_l$.

\paragraph{Closed-form coordinate compensation}
Consider a single output row, written as a column vector $\mathbf w$, with quantized--dequantized counterpart $\widehat{\mathbf w}$. Define the row-wise perturbation
\[
\Delta \mathbf w := \widehat{\mathbf w}-\mathbf w .
\]
Let $\mathbf G:=\mathbf G_l$, and assume that $\mathbf G$ is invertible. The row-wise quadratic objective is
\[
\min_{\Delta \mathbf w}\;
\frac{1}{2}\Delta \mathbf w^\top \mathbf G \Delta \mathbf w .
\]

Now suppose the $j$-th coordinate is committed to its quantized--dequantized value $\widehat w_j$, and define the induced quantization error by
\[
e_j := w_j-\widehat w_j,
\qquad
\Delta w_j = -e_j .
\]
This leads to the constrained problem
\[
\min_{\Delta \mathbf w}\;
\frac{1}{2}\Delta \mathbf w^\top \mathbf G \Delta \mathbf w
\quad
\text{s.t.}
\quad
\Delta w_j = -e_j .
\]

Introducing a Lagrange multiplier $\nu$, the Lagrangian is
\[
\mathcal J(\Delta \mathbf w,\nu)
=
\frac{1}{2}\Delta \mathbf w^\top \mathbf G \Delta \mathbf w
+
\nu(\Delta w_j+e_j).
\]
The first-order optimality condition gives
\[
\mathbf G\Delta \mathbf w+\nu \mathbf e_j=0,
\qquad\Longrightarrow\qquad
\Delta \mathbf w = -\nu \mathbf G^{-1}\mathbf e_j,
\]
where $\mathbf e_j$ denotes the $j$-th standard basis vector. Let
\[
\mathbf M := \mathbf G^{-1}.
\]
Enforcing the constraint yields
\[
-e_j = \Delta w_j = -\nu M_{jj},
\qquad\Longrightarrow\qquad
\nu = \frac{e_j}{M_{jj}}.
\]
Therefore, the optimal perturbation update is
\[
\Delta \mathbf w^\star
=
-\frac{e_j}{M_{jj}}\,\mathbf M_{:,j}.
\]

Substituting this back into the quadratic objective, the minimal objective value induced by fixing coordinate $j$ is
\[
\Delta \mathcal L^\star
=
\frac{1}{2}\frac{e_j^2}{M_{jj}}.
\]

\paragraph{Relation to GPTQ}
The above derivation shows that \texttt{SARQC-GBS} is a strictly minimal modification of GPTQ under the same Gram-based sequential framework. In standard GPTQ, the quadratic form is governed by the activation Gram matrix $\mathbf{X}_l\mathbf{X}_l^\top$. SARQC replaces it with the regularized matrix
$$
\mathbf{G}_l = \mathbf{X}_l\mathbf{X}_l^\top + \lambda \mathbf{S}_l\mathbf{S}_l^\top,
$$
while leaving the row-wise decomposition, coordinate-wise commitment, and inverse-curvature compensation unchanged.

Consequently, \texttt{SARQC-GBS} preserves the implementation structure and computational pattern of GPTQ, but alters the effective curvature to incorporate saliency-aware regularization. When $\lambda=0$, the objective reduces to the undamped GPTQ-style quadratic form governed only by $\mathbf{X}_l\mathbf{X}_l^\top$. In practical GPTQ implementations, however, an additional heuristic damping term is often added to the Gram matrix for numerical stability. This damping is an implementation-level stabilization and can be viewed as an isotropic curvature regularization. In contrast, SARQC introduces a saliency-aware structured regularization term through $\lambda \mathbf{S}_l\mathbf{S}_l^\top$. Therefore, the $\lambda=0$ case of \texttt{SARQC-GBS} should be understood as the undamped GPTQ-style objective. When $\mathbf{S}_l\mathbf{S}_l^\top=\mathbf{I}$, SARQC reduces to an isotropically regularized variant, which is closely related to the damping used in practical GPTQ implementations.

\subsection{Proof of \Cref{thm:generalization risk}}
\label{append: proof of generalization risk}

\textit{Proof.}
Consider a single linear layer and define $
\Delta \mathbf W_l := \widehat{\mathbf W}_l - \mathbf W_l$.
Let the true downstream reconstruction risk be
\[
\mathcal R(\widehat{\mathbf W}_l)
:=
\mathbb E_{X \sim p_X}
\left[
\|
\Delta \mathbf W_l X
\|_2^2
\right],
\]
and let the empirical calibration risk on a calibration set $\mathcal D_{\mathrm{cal},l} := \mathbf{X}_l :=\{\mathbf{X}_{l,i}\}_{i=1}^n $
be
\[
\widehat{\mathcal R}_{\mathrm{cal}}(\widehat{\mathbf W}_l)
:=
\frac{1}{n}
\sum_{i=1}^n
\|
\Delta \mathbf W_l \mathbf{X}_{l,i}
\|_2^2 .
\]

Assume that $\| X\|_2 \le M_X$. We consider a restricted quantization hypothesis class
\[
\mathcal Q_R
:=
\left\{
\widehat{\mathbf W}_l \in \mathcal Q
:
\|
\widehat{\mathbf W}_l
-
\mathbf W_l
\|_\mathrm{F}
\le R
\right\}.
\]
For any $\widehat{\mathbf W}_l \in \mathcal Q_R$, we have
\[
\|
\Delta \mathbf W_l  X \|_2 \le \| \Delta \mathbf W_l \|_\mathrm{F} \| X\|_2
\le
RM_X .
\]
Therefore,
\[
0
\le
\|
\Delta \mathbf W_l  X
\|_2^2
\le
R^2 M_X^2 .
\]
We assume that $\mathcal Q_R$ is finite. This finite-class assumption is natural for quantization, since admissible quantized
weights are selected from a discrete codebook. The resulting bound may be loose because $|\mathcal Q_R|$ can be very large, but it cleanly illustrates the role of the weight drift radius $R$.

Define the loss $
\ell_{\widehat{\mathbf W}_l}( X)
:=
\|
\Delta \mathbf W_l  X
\|_2^2$.
Then for any fixed $\widehat{\mathbf W}_l \in \mathcal Q_R$, $
\mathcal R(\widehat{\mathbf W}_l)
=
\mathbb E[
\ell_{\widehat{\mathbf W}_l}( X)
]$,
and $
\widehat{\mathcal R}_{\mathrm{cal}}(\widehat{\mathbf W}_l)
=
\frac{1}{n}
\sum_{i=1}^n
\ell_{\widehat{\mathbf W}_l}(\mathbf{X}_{l,i})$.

Since $0
\le
\ell_{\widehat{\mathbf W}_l}(X)
\le
R^2 M_X^2$, Hoeffding's inequality gives, for any fixed
$\widehat{\mathbf W}_l \in \mathcal Q_R$,
\[
\Pr
\left(
\left|
\mathcal R(\widehat{\mathbf W}_l)
-
\widehat{\mathcal R}_{\mathrm{cal}}(\widehat{\mathbf W}_l)
\right|
\ge t
\right)
\le
2
\exp
\left(
-
\frac{2nt^2}{R^4 M_X^4}
\right).
\]

Since $\mathcal Q_R$ is finite, we apply the union bound:
\[
\Pr
\left(
\exists \widehat{\mathbf W}_l \in \mathcal Q_R
:
\left|
\mathcal R(\widehat{\mathbf W}_l)
-
\widehat{\mathcal R}_{\mathrm{cal}}(\widehat{\mathbf W}_l)
\right|
\ge t
\right)
\le
2 |\mathcal Q_R|
\exp
\left(
-
\frac{2nt^2}{R^4 M_X^4}
\right).
\]
Setting the right-hand side to $\delta$, we have
\[
2 |\mathcal Q_R|
\exp
\left(
-
\frac{2nt^2}{R^4 M_X^4}
\right)
=
\delta .
\]
Taking logarithms gives
\[
\frac{2nt^2}{R^4 M_X^4}
=
\log \frac{2|\mathcal Q_R|}{\delta}.
\]
Thus,
\[
t
=
R^2 M_X^2
\sqrt{
\frac{
\log \frac{2|\mathcal Q_R|}{\delta}
}{
2n
}
}.
\]

Therefore, with probability at least $1-\delta$, for all
$\widehat{\mathbf W}_l \in \mathcal Q_R$, we have
\[
\left|
\mathcal R(\widehat{\mathbf W}_l)
-
\widehat{\mathcal R}_{\mathrm{cal}}(\widehat{\mathbf W}_l)
\right|
\le
R^2 M_X^2
\sqrt{
\frac{
\log \frac{2|\mathcal Q_R|}{\delta}
}{
2n
}
}.
\]
In particular,
\[
\mathcal R(\widehat{\mathbf W}_l)
\le
\widehat{\mathcal R}_{\mathrm{cal}}(\widehat{\mathbf W}_l)
+
R^2 M_X^2
\sqrt{
\frac{
\log \frac{2|\mathcal Q_R|}{\delta}
}{
2n
}
}.
\]
This completes the proof.
\qed

This bound shows that the true downstream risk can be controlled by two terms: the empirical calibration risk (i.e., reconstruction error in this case) and a generalization term depending on the weight drift $R$. A large weight drift increases the generalization term. Therefore, minimizing only
$\widehat{\mathcal R}_{\mathrm{cal}}(\widehat{\mathbf W}_l)$ may lead to poor downstream
generalization if the selected quantized weights have large drift from the original
FP weights.

\subsection{Proof of Corollary~\ref{cor:constraint_penalty_supported}}
\label{append:cor_constraint_penalty}

Fix a layer $l$, and define $D(\widehat{\mathbf W}^{\prime}):=\|\widehat{\mathbf W}^{\prime}-\mathbf W_l\|_{\mathrm F}^2$. 
For this fixed constrained minimizer $\widehat{\mathbf W}_l$, define
$\mathcal S_+:=\{\widehat{\mathbf W}^{\prime}\in\mathcal Q:D(\widehat{\mathbf W}^{\prime})>D(\widehat{\mathbf W}_l)\}$ and
$\mathcal S_-:=\{\widehat{\mathbf W}^{\prime}\in\mathcal Q:D(\widehat{\mathbf W}^{\prime})<D(\widehat{\mathbf W}_l)\}$. 
Set 
$$\lambda_{\min}:= \max\left\{ 0,\;\max_{\widehat{\mathbf W}^{\prime}\in\mathcal S_+} \frac{\widehat{\mathcal R}_{\mathrm{cal}} (\widehat{\mathbf W}_l)-\widehat{\mathcal R}_{\mathrm{cal}} (\widehat{\mathbf W}^{\prime})} {D(\widehat{\mathbf W}^{\prime})-D(\widehat{\mathbf W}_l)} \right\}, $$
$$\text{and} \quad \lambda_{\max}:=
\min_{\widehat{\mathbf W}^{\prime}\in\mathcal S_-} \frac{\widehat{\mathcal R}_{\mathrm{cal}} (\widehat{\mathbf W}^{\prime})-\widehat{\mathcal R}_{\mathrm{cal}} (\widehat{\mathbf W}_l)}{D(\widehat{\mathbf W}_l)-D(\widehat{\mathbf W}^{\prime})}.$$
Here we use the conventions $\max_{\emptyset}(\cdot):=-\infty$ and $\min_{\emptyset}(\cdot):=+\infty$.

\begin{proof}[Proof]
    For simplicity, write $\widehat{\mathbf W}:=\widehat{\mathbf W}_l$ and let $\widehat{\mathbf W}^{\prime}\in\mathcal Q$ be arbitrary. We also write $L$ and $D$ without the layer subscript. For a fixed finite $\lambda\ge 0$, define $J_\lambda(\mathbf W):=\widehat{\mathcal R}_{\mathrm{cal}} (\mathbf W)+\lambda D(\mathbf W)$. The condition $\widehat{\mathbf W}\in\arg\min_{\mathbf W\in\mathcal Q}J_\lambda(\mathbf W)$ is equivalent to the pairwise inequalities $J_\lambda(\widehat{\mathbf W})\le J_\lambda(\widehat{\mathbf W}^{\prime})$ for all $\widehat{\mathbf W}^{\prime}\in\mathcal Q$.

    For each $\widehat{\mathbf W}^{\prime}\in\mathcal Q$, define $\Delta \widehat{\mathcal R}_{\mathrm{cal}} (\widehat{\mathbf W}^{\prime}):=\widehat{\mathcal R}_{\mathrm{cal}} (\widehat{\mathbf W}^{\prime})-\widehat{\mathcal R}_{\mathrm{cal}} (\widehat{\mathbf W})$ and $\Delta D(\widehat{\mathbf W}^{\prime}):=D(\widehat{\mathbf W}^{\prime})-D(\widehat{\mathbf W})$. Then $J_\lambda(\widehat{\mathbf W})\le J_\lambda(\widehat{\mathbf W}^{\prime})$ is equivalent to $\Delta \widehat{\mathcal R}_{\mathrm{cal}} (\widehat{\mathbf W}^{\prime})+\lambda\Delta D(\widehat{\mathbf W}^{\prime})\ge 0$.

    We first note that infeasible points are not omitted. Since $\widehat{\mathbf W}\in\mathcal Q_R$, we have $D(\widehat{\mathbf W})\le R^2$. If $\widehat{\mathbf W}^{\prime}\notin\mathcal Q_R$, then $D(\widehat{\mathbf W}^{\prime})>R^2$, and hence $D(\widehat{\mathbf W}^{\prime})>D(\widehat{\mathbf W})$. Therefore every infeasible point belongs to $\mathcal S_+$. In particular, any infeasible point with a smaller value of $L$ contributes to the lower-bound requirement encoded in $\lambda_{\min}$.

    We now prove the sufficiency. Assume $\lambda_{\min}\le \lambda_{\max}$ and fix a finite $\lambda_R$ such that $\lambda_{\min}\le \lambda_R\le \lambda_{\max}$. We verify $J_{\lambda_R}(\widehat{\mathbf W})\le J_{\lambda_R}(\widehat{\mathbf W}^{\prime})$ by considering three cases.

    First, suppose $D(\widehat{\mathbf W}^{\prime})>D(\widehat{\mathbf W})$. Then $\widehat{\mathbf W}^{\prime}\in\mathcal S_+$. By the definition of $\lambda_{\min}$ and the choice $\lambda_R\ge\lambda_{\min}$, we have
    \begin{align*}
    \lambda_R
    \ge
    \frac{\widehat{\mathcal R}_{\mathrm{cal}} (\widehat{\mathbf W})-\widehat{\mathcal R}_{\mathrm{cal}} (\widehat{\mathbf W}^{\prime})}
    {D(\widehat{\mathbf W}^{\prime})-D(\widehat{\mathbf W})}.
    \end{align*}
    The denominator $D(\widehat{\mathbf W}^{\prime})-D(\widehat{\mathbf W})$ is positive, so multiplying by it preserves the inequality and gives $\lambda_R\{D(\widehat{\mathbf W}^{\prime})-D(\widehat{\mathbf W})\}\ge \widehat{\mathcal R}_{\mathrm{cal}} (\widehat{\mathbf W})-\widehat{\mathcal R}_{\mathrm{cal}} (\widehat{\mathbf W}^{\prime})$. This is equivalent to $\Delta \widehat{\mathcal R}_{\mathrm{cal}} (\widehat{\mathbf W}^{\prime})+\lambda_R\Delta D(\widehat{\mathbf W}^{\prime})\ge 0$, and hence $J_{\lambda_R}(\widehat{\mathbf W})\le J_{\lambda_R}(\widehat{\mathbf W}^{\prime})$. This case includes all infeasible candidates.

    Second, suppose $D(\widehat{\mathbf W}^{\prime})<D(\widehat{\mathbf W})$. Then $\widehat{\mathbf W}^{\prime}\in\mathcal S_-$. Moreover, $\widehat{\mathbf W}^{\prime}$ is feasible because $D(\widehat{\mathbf W}^{\prime})<D(\widehat{\mathbf W})\le R^2$. Since $\widehat{\mathbf W}$ minimizes $L$ over $\mathcal Q_R$, we have $\widehat{\mathcal R}_{\mathrm{cal}} (\widehat{\mathbf W})\le \widehat{\mathcal R}_{\mathrm{cal}} (\widehat{\mathbf W}^{\prime})$. Thus the ratio defining the upper bound is nonnegative. By the definition of $\lambda_{\max}$ and the choice $\lambda_R\le\lambda_{\max}$, we have
    \begin{align*}
    \lambda_R
    \le
    \frac{\widehat{\mathcal R}_{\mathrm{cal}} (\widehat{\mathbf W}^{\prime})-\widehat{\mathcal R}_{\mathrm{cal}} (\widehat{\mathbf W})}
    {D(\widehat{\mathbf W})-D(\widehat{\mathbf W}^{\prime})}.
    \end{align*}
    The denominator $D(\widehat{\mathbf W})-D(\widehat{\mathbf W}^{\prime})$ is positive, so multiplying by it preserves the inequality and gives $\lambda_R\{D(\widehat{\mathbf W})-D(\widehat{\mathbf W}^{\prime})\}\le \widehat{\mathcal R}_{\mathrm{cal}} (\widehat{\mathbf W}^{\prime})-\widehat{\mathcal R}_{\mathrm{cal}} (\widehat{\mathbf W})$. Equivalently, $\Delta \widehat{\mathcal R}_{\mathrm{cal}} (\widehat{\mathbf W}^{\prime})+\lambda_R\Delta D(\widehat{\mathbf W}^{\prime})\ge 0$, and hence $J_{\lambda_R}(\widehat{\mathbf W})\le J_{\lambda_R}(\widehat{\mathbf W}^{\prime})$. This case explains why the penalty parameter cannot be chosen too large: otherwise a feasible point with smaller distance penalty could dominate $\widehat{\mathbf W}$.

    Third, suppose $D(\widehat{\mathbf W}^{\prime})=D(\widehat{\mathbf W})$. Since $D(\widehat{\mathbf W})\le R^2$, this implies $D(\widehat{\mathbf W}^{\prime})\le R^2$, so $\widehat{\mathbf W}^{\prime}\in\mathcal Q_R$. By constrained optimality, $\widehat{\mathcal R}_{\mathrm{cal}} (\widehat{\mathbf W})\le \widehat{\mathcal R}_{\mathrm{cal}} (\widehat{\mathbf W}^{\prime})$. Since the penalty terms are equal in this case, $J_{\lambda_R}(\widehat{\mathbf W})\le J_{\lambda_R}(\widehat{\mathbf W}^{\prime})$ follows immediately.

    The three cases exhaust all $\widehat{\mathbf W}^{\prime}\in\mathcal Q$. Therefore $J_{\lambda_R}(\widehat{\mathbf W})\le J_{\lambda_R}(\widehat{\mathbf W}^{\prime})$ for every $\widehat{\mathbf W}^{\prime}\in\mathcal Q$, which proves
    \begin{align*}
    \widehat{\mathbf W}
    \in
    \arg\min_{\widehat{\mathbf W}^{\prime}\in\mathcal Q}
    \left\{
    \widehat{\mathcal R}_{\mathrm{cal}} (\widehat{\mathbf W}^{\prime})+\lambda_R D(\widehat{\mathbf W}^{\prime})
    \right\}.
    \end{align*}

    It remains to prove the converse. Suppose there exists a finite $\lambda_R\ge 0$ such that $\widehat{\mathbf W}\in\arg\min_{\widehat{\mathbf W}^{\prime}\in\mathcal Q}J_{\lambda_R}(\widehat{\mathbf W}^{\prime})$. Then for every $\widehat{\mathbf W}^{\prime}\in\mathcal Q$, we have $\Delta \widehat{\mathcal R}_{\mathrm{cal}} (\widehat{\mathbf W}^{\prime})+\lambda_R\Delta D(\widehat{\mathbf W}^{\prime})\ge 0$.

    If $D(\widehat{\mathbf W}^{\prime})>D(\widehat{\mathbf W})$, then $\Delta D(\widehat{\mathbf W}^{\prime})>0$, and the preceding inequality implies
    \begin{align*}
    \lambda_R
    \ge
    \frac{\widehat{\mathcal R}_{\mathrm{cal}} (\widehat{\mathbf W})-\widehat{\mathcal R}_{\mathrm{cal}} (\widehat{\mathbf W}^{\prime})}
    {D(\widehat{\mathbf W}^{\prime})-D(\widehat{\mathbf W})}.
    \end{align*}
    Taking the maximum over all such $\widehat{\mathbf W}^{\prime}$ and using $\lambda_R\ge 0$ gives $\lambda_R\ge\lambda_{\min}$, with the convention that this is trivial when $\mathcal S_+=\emptyset$.

    If $D(\widehat{\mathbf W}^{\prime})<D(\widehat{\mathbf W})$, then $\Delta D(\widehat{\mathbf W}^{\prime})<0$, and dividing by this negative quantity reverses the inequality. Thus
    \begin{align*}
    \lambda_R
    \le
    \frac{\widehat{\mathcal R}_{\mathrm{cal}} (\widehat{\mathbf W}^{\prime})-\widehat{\mathcal R}_{\mathrm{cal}} (\widehat{\mathbf W})}
    {D(\widehat{\mathbf W})-D(\widehat{\mathbf W}^{\prime})}.
    \end{align*}
    Taking the minimum over all such $\widehat{\mathbf W}^{\prime}$ gives $\lambda_R\le\lambda_{\max}$, with the convention that this is trivial when $\mathcal S_-=\emptyset$. Hence $\lambda_{\min}\le\lambda_R\le\lambda_{\max}$.

    This completes the proof.
\end{proof}

Corollary~\ref{cor:constraint_penalty_supported} clarifies that the constrained calibration problem can be represented by a penalized calibration objective over a finite quantization set. 
It gives a recovery condition for a fixed constrained minimizer $\widehat{\mathbf W}_l$. 
In particular, if the supportedness condition $\lambda_{\min}\le \lambda_{\max}$ holds, then one can choose a penalty strength $\lambda_R\in[\lambda_{\min},\lambda_{\max}]$ such that $\widehat{\mathbf W}_l$ remains optimal after replacing the hard radius constraint $D(\widehat{\mathbf W}^{\prime})\le R^2$ by the soft penalty $\lambda_R D(\widehat{\mathbf W}^{\prime})$.
It is the standard supported-solution viewpoint in multiobjective optimization that weighted-sum scalarizations recover supported efficient points, while finite objective sets may also contain efficient but unsupported points \citep{Zadeh1963,Geoffrion1968,Miettinen1998,Ehrgott2005,MarlerArora2010}.

\paragraph{Interpretation of $\lambda_{\min}$ and $\lambda_{\max}$}
The quantities $\lambda_{\min}$ and $\lambda_{\max}$ have a direct pairwise-comparison interpretation. The lower threshold $\lambda_{\min}$ is the smallest penalty strength needed to prevent candidates with larger distance $D(\widehat{\mathbf W}^{\prime})>D(\widehat{\mathbf W}_l)$ from improving upon $\widehat{\mathbf W}_l$ under the penalized objective. This class includes all infeasible candidates, since $\widehat{\mathbf W}_l\in\mathcal Q_R$ implies $D(\widehat{\mathbf W}_l)\le R^2$, whereas any $\widehat{\mathbf W}^{\prime}\notin\mathcal Q_R$ satisfies $D(\widehat{\mathbf W}^{\prime})>R^2\ge D(\widehat{\mathbf W}_l)$. Thus infeasible quantized weights are explicitly accounted for in $\lambda_{\min}$. 
In contrast, $\lambda_{\max}$ is the largest penalty strength allowed before a closer feasible candidate with $D(\widehat{\mathbf W}^{\prime})<D(\widehat{\mathbf W}_l)$ can overtake $\widehat{\mathbf W}_l$ due to its smaller distance penalty. Such candidates are automatically feasible because $D(\widehat{\mathbf W}^{\prime})<D(\widehat{\mathbf W}_l)\le R^2$, and the constrained optimality of $\widehat{\mathbf W}_l$ ensures $L(\widehat{\mathbf W}_l)\le L(\widehat{\mathbf W}^{\prime})$ for them. 
Hence $\lambda_{\min}\le\lambda_{\max}$ simply states that the penalty must be large enough to suppress farther, possibly infeasible, low-risk candidates, but not so large that it favors closer, higher-risk feasible candidates. This is the exact condition under which the bias introduced by the penalty does not change the selected minimizer.

\paragraph{Reasonableness of the supportedness condition}
We note that the condition $\lambda_{\min}\le\lambda_{\max}$ is not an ad-hoc regularity assumption which is the finite-set feasibility condition for a common penalty multiplier. Indeed, every competitor $\widehat{\mathbf W}^{\prime}\in\mathcal Q$ imposes one admissible interval for $\lambda_R$, that is, candidates with $D(\widehat{\mathbf W}^{\prime})>D(\widehat{\mathbf W}_l)$ impose lower bounds, candidates with $D(\widehat{\mathbf W}^{\prime})<D(\widehat{\mathbf W}_l)$ impose upper bounds, and candidates with $D(\widehat{\mathbf W}^{\prime})=D(\widehat{\mathbf W}_l)$ impose no additional constraint because constrained optimality already gives $L(\widehat{\mathbf W}_l)\le L(\widehat{\mathbf W}^{\prime})$. Hence $[\lambda_{\min},\lambda_{\max}]$ is precisely the intersection of all pairwise admissible intervals. 
Its nonemptiness means that the penalty can be chosen large enough to suppress farther, possibly infeasible, low-risk candidates, but not so large that it favors closer, higher-risk feasible candidates. 

This condition is closely aligned with the supported-solution viewpoint in multiobjective optimization. Weighted-sum scalarizations recover supported efficient solutions, whereas finite objective sets may contain efficient but unsupported solutions that cannot be recovered by any nonnegative linear weighting \citep{Geoffrion1968,Miettinen1998,Ehrgott2005,MarlerArora2010}.  In this sense, $\lambda_{\min}\le\lambda_{\max}$ says precisely that the constrained minimizer $\widehat{\mathbf W}_l$ is supported by some linear scalarization of the distance--risk trade-off.

Similar condition also appears in exact-penalty and Lagrangian-relaxation theory. 
Exact-penalty methods typically require the existence of a finite penalty parameter, often under constraint qualifications, multiplier bounds, or other exactness conditions \citep{HanMangasarian1979,DiPilloGrippo1989}. 
Similar exact-penalty ideas have also been developed for nonlinear integer programming, where additional conditions are imposed to guarantee equivalence between the original discrete problem and its penalized reformulation \citep{LucidiRinaldi2010}. 
Likewise, in Lagrangian relaxation for integer programming, a multiplier generally provides a relaxation or bound, and exact primal recovery is not automatic without an additional exactness or zero-gap condition \citep{Geoffrion1974,Fisher1981}.

\paragraph{Link to Lagrangian relaxation}
This result is closely related to classical Lagrangian relaxation. The constrained problem $\min_{\widehat{\mathbf W}^{\prime}\in\mathcal Q}L(\widehat{\mathbf W}^{\prime})$ subject to $D(\widehat{\mathbf W}^{\prime})\le R^2$ has the Lagrangian form $L(\widehat{\mathbf W}^{\prime})+\lambda\{D(\widehat{\mathbf W}^{\prime})-R^2\}$ with multiplier $\lambda\ge 0$. Since the term $-\lambda R^2$ is constant in $\widehat{\mathbf W}^{\prime}$, minimizing the Lagrangian over $\mathcal Q$ is equivalent to minimizing $L(\widehat{\mathbf W}^{\prime})+\lambda D(\widehat{\mathbf W}^{\prime})$. 
In convex optimization, suitable constraint qualifications and strong duality often justify recovery of constrained optima from Lagrange multipliers \citep{BoydVandenberghe2004}. In discrete or integer optimization, however, Lagrangian relaxation is generally a relaxation or bounding device, and exact primal recovery is not automatic \citep{Everett1963,Geoffrion1974,Fisher1981}. 
Corollary~\ref{cor:constraint_penalty_supported} can therefore be viewed as a finite-set analogue of Lagrangian primal recovery. 
It identifies when a Lagrangian-style scalarization recovers the selected constrained minimizer.

\clearpage
\section{Algorithms and Implementation Details}
\label{app:pseudo algorithms and implementation details}

\subsection{Algorithms}
\label{appendix: pseudo algorithms}
In this section, we summarize the algorithms used to optimize the proposed SARQC objective.
\begin{algorithm}[hbpt]
\caption{SARQC with Grid Search over Scaling Factors (SARQC-GS)}
\label{alg:sarqc_grid}
\begin{algorithmic}[1]
\Require FP weight matrix $\mathbf{W}_l \in \mathbb{R}^{d_{\mathrm{out}}\times d_{\mathrm{in}}}$, calibration input matrix $\mathbf{X}_l \in \mathbb{R}^{d_{\mathrm{in}}\times n}$, search set $\mathcal{A}$, saliency matrix $\mathbf{S}_l$, regularization weight $\lambda>0$

\For{each $\alpha \in \mathcal{A}$}
    \State Construct the channel-wise scaling vector $\tilde{s}_l(\alpha)$
    \State $\widetilde{\mathbf{S}}_l(\alpha)\gets \mathrm{diag}(\tilde{s}_l(\alpha))$
    \State $\widehat{\mathbf{W}}_l(\alpha)\gets Q\!\bigl(\mathbf{W}_l\,\widetilde{\mathbf{S}}_l(\alpha)\bigr)\,\widetilde{\mathbf{S}}_l(\alpha)^{-1}$
    \State $\mathcal{L}_{\mathrm{recon}}(\alpha)\gets \|\mathbf{W}_l\mathbf{X}_l-\widehat{\mathbf{W}}_l(\alpha)\mathbf{X}_l\|_\mathrm{F}^2$
    \State $\mathcal{L}_{\mathrm{sar}}(\alpha)\gets \|(\widehat{\mathbf{W}}_l(\alpha)-\mathbf{W}_l)\mathbf{S}_l\|_\mathrm{F}^2$
\EndFor

\State Compute $
\mathcal{L}_{\mathrm{recon}}^{\min},\mathcal{L}_{\mathrm{recon}}^{\max}
\gets
\min_{\alpha\in\mathcal A}\mathcal{L}_{\mathrm{recon}}(\alpha),\;
\max_{\alpha\in\mathcal A}\mathcal{L}_{\mathrm{recon}}(\alpha)
$

\State Compute $
\mathcal{L}_{\mathrm{sar}}^{\min},\mathcal{L}_{\mathrm{sar}}^{\max}
\gets
\min_{\alpha\in\mathcal A}\mathcal{L}_{\mathrm{sar}}(\alpha),\;
\max_{\alpha\in\mathcal A}\mathcal{L}_{\mathrm{sar}}(\alpha)
$

\State $\mathcal{L}^{\star}\gets +\infty$
\For{each $\alpha \in \mathcal{A}$}
    \State $\widetilde{\mathcal{L}}_{\mathrm{recon}}(\alpha)\gets
    \dfrac{\mathcal{L}_{\mathrm{recon}}(\alpha)-\mathcal{L}_{\mathrm{recon}}^{\min}}
    {\mathcal{L}_{\mathrm{recon}}^{\max}-\mathcal{L}_{\mathrm{recon}}^{\min}}$
    , $\widetilde{\mathcal{L}}_{\mathrm{sar}}(\alpha)\gets
    \dfrac{\mathcal{L}_{\mathrm{sar}}(\alpha)-\mathcal{L}_{\mathrm{sar}}^{\min}}
    {\mathcal{L}_{\mathrm{sar}}^{\max}-\mathcal{L}_{\mathrm{sar}}^{\min}}$
    \State $\mathcal{L}(\alpha)\gets \widetilde{\mathcal{L}}_{\mathrm{recon}}(\alpha)+\lambda\,\widetilde{\mathcal{L}}_{\mathrm{sar}}(\alpha)$
    \If{$\mathcal{L}(\alpha)<\mathcal{L}^{\star}$}
        \State $\mathcal{L}^{\star}\gets \mathcal{L}(\alpha)$, $\widehat{\mathbf{W}}_l\gets \widehat{\mathbf{W}}_l(\alpha)$
    \EndIf
\EndFor
\State \Return $\widehat{\mathbf{W}}_l$
\end{algorithmic}
\end{algorithm}

\begin{algorithm}[hbpt]
\caption{SARQC with Gram-based Sequential Quantization (SARQC-GBS)}
\label{alg:sarqc_gram}
\begin{algorithmic}[1]
\Require FP weight matrix $\mathbf{W}_l \in \mathbb{R}^{d_{\mathrm{out}}\times d_{\mathrm{in}}}$, calibration input matrix $\mathbf{X}_l \in \mathbb{R}^{d_{\mathrm{in}}\times n}$, saliency profile $s_l(\gamma)$, regularization weight $\lambda>0$, block size $B$

\State $\mathbf{G}_l \gets \mathbf{X}_l\mathbf{X}_l^\top
+
\lambda \bar{h}_l\,
\mathrm{diag}\!\left(
\dfrac{s_l(\gamma)^2}{\mathrm{mean}(s_l(\gamma)^2)}
\right)$, where $\bar{h}_l \gets \mathrm{mean}\!\left(\mathrm{diag}(\mathbf{X}_l\mathbf{X}_l^\top)\right)$
\State $\mathbf{M}_l \gets \mathrm{chol}(\mathbf{G}_l^{-1})^\top$
\State $\widehat{\mathbf{W}}_l \gets \mathbf{0}_{d_{\mathrm{out}}\times d_{\mathrm{in}}}$
\State $\mathbf{U} \gets \mathbf{W}_l$

\For{$i = 1, 1+B, 1+2B, \dots, d_{\mathrm{in}}$}
    \State $i_{\mathrm{end}} \gets \min(i+B-1,\; d_{\mathrm{in}})$
    \State $\mathbf{E} \gets \mathbf{0}_{d_{\mathrm{out}}\times (i_{\mathrm{end}}-i+1)}$
    \For{$j = i, i+1, \dots, i_{\mathrm{end}}$}
        \State $\widehat{\mathbf{W}}_{l,:,j} \gets Q(\mathbf{U}_{:,j})$
        \State $\mathbf{e} \gets \dfrac{\mathbf{U}_{:,j}-\widehat{\mathbf{W}}_{l,:,j}}{(\mathbf{M}_l)_{jj}}$
        \State $\mathbf{E}_{:,\,j-i+1} \gets \mathbf{e}$
        \State $\mathbf{U}_{:,\,j:i_{\mathrm{end}}} \gets \mathbf{U}_{:,\,j:i_{\mathrm{end}}} - \mathbf{e}\,(\mathbf{M}_l)_{j,\,j:i_{\mathrm{end}}}$
    \EndFor
    \If{$i_{\mathrm{end}} < d_{\mathrm{in}}$}
        \State $\mathbf{U}_{:,\,i_{\mathrm{end}}+1:d_{\mathrm{in}}}
        \gets
        \mathbf{U}_{:,\,i_{\mathrm{end}}+1:d_{\mathrm{in}}}
        -
        \mathbf{E}\,(\mathbf{M}_l)_{\,i:i_{\mathrm{end}},\,i_{\mathrm{end}}+1:d_{\mathrm{in}}}$
    \EndIf
\EndFor

\State \Return $\widehat{\mathbf{W}}_l$
\end{algorithmic}
\end{algorithm}

\clearpage

\subsection{Implementation Details}
\label{alg:Implementation_Details}
\paragraph{Implementation Details for SARQC-GS}

\texttt{SARQC-GS} is implemented as a scaling-based post-training quantization procedure with a regularized model selection rule over candidate scaling factors. For each layer, we perform a grid search over a scalar parameter $\alpha$ that controls a channel-wise reparameterization of the weight matrix.

Specifically, we consider a uniform grid
$$
\mathcal{A}=\left\{\frac{k}{20}:k=0,1,\dots,20\right\}.
$$
For each $\alpha\in\mathcal{A}$, we construct a channel-wise scaling vector $\tilde{s}_l(\alpha)\in\mathbb{R}_+^{d_{\mathrm{in}}}$ and the corresponding diagonal matrix
$$
\widetilde{\mathbf{S}}_l(\alpha)=\mathrm{diag}\!\bigl(\tilde{s}_l(\alpha)\bigr).
$$
Following the design of AWQ~\citep{lin2024awq}, the scaling factors are defined as
$$
\tilde{s}_l^{(j)}(\alpha)
=
\frac{\mathrm{mean}(|\mathbf{X}_l^{(j)}|)^{\alpha}}
{\mathrm{mean}(|\mathbf{W}_l^{(j)}|)^{1-\alpha}},
$$
which interpolates between activation-driven and weight-driven normalization. The resulting vector is further normalized by the geometric mean of its extrema to improve numerical stability.

Given $\widetilde{\mathbf{S}}_l(\alpha)$, we generate a candidate quantized weight via
$$
\widehat{\mathbf{W}}_l(\alpha)
=
Q\!\bigl(\mathbf{W}_l\,\widetilde{\mathbf{S}}_l(\alpha)\bigr)\,
\widetilde{\mathbf{S}}_l(\alpha)^{-1}.
$$

In addition, \texttt{SARQC-GS} introduces a fixed saliency matrix $\mathbf{S}_l=\mathrm{diag}(s_l)$ for discrepancy regularization, defined as
$$
s_l^{(j)}=
\frac{\mathrm{mean}(|\mathbf{X}_l^{(j)}|)}
{\mathrm{mean}(|\mathbf{W}_l^{(j)}|)}.
$$
Importantly, $\widetilde{\mathbf{S}}_l(\alpha)$ and $\mathbf{S}_l$ serve distinct roles: the former generates candidate solutions, while the latter weights channel-wise deviations in the regularization term.

For each candidate $\alpha$, we evaluate the reconstruction loss
$$
\mathcal{L}_{\mathrm{recon}}(\alpha)
=
\|\mathbf{W}_l\mathbf{X}_l-\widehat{\mathbf{W}}_l(\alpha)\mathbf{X}_l\|_\mathrm{F}^2,
$$
and the saliency-weighted discrepancy
$$
\mathcal{L}_{\mathrm{sar}}(\alpha)
=
\|(\widehat{\mathbf{W}}_l(\alpha)-\mathbf{W}_l)\mathbf{S}_l\|_\mathrm{F}^2.
$$
To balance the scale mismatch between the two terms, we apply min--max normalization across $\alpha\in\mathcal{A}$ and select the best candidate by minimizing the normalized joint objective.

At the outer level, the regularization weight $\lambda$ is selected from $\{0.1,0.2,\dots,1.0\}$ using a held-out validation split of the calibration data. For each $\lambda$, we run layer-wise \texttt{SARQC-GS} on the training subset and choose the value that minimizes the validation perplexity.

\paragraph{Implementation Details for SARQC-GBS}

\texttt{SARQC-GBS} is implemented as a second-order Gram-based sequential quantization method, where the curvature matrix is regularized using saliency-aware statistics.

We perform layer-wise hyperparameter selection over
$$
\lambda \in \{0.25, 0.5, 0.75\},
\qquad
\gamma \in \{0.1, 0.15, 0.35, 0.5\}.
$$
Here, $\lambda$ controls the strength of curvature regularization, while $\gamma$ determines the relative contribution of activation and weight statistics in the saliency profile.

For each candidate $(\lambda,\gamma)$, we construct a saliency vector
$$
s_l^{(j)}(\gamma)
=
\frac{\mathrm{mean}(|\mathbf{X}_l^{(j)}|)^\gamma}
{\mathrm{mean}(|\mathbf{W}_l^{(j)}|)^{1-\gamma}}.
$$

We then construct the regularized curvature matrix as
\begin{align*}
\mathbf{G}_l:=\mathbf{X}_l\mathbf{X}_l^\top+\lambda\mathbf{S}_l\mathbf{S}_l^\top ~.
\end{align*}
To match the practical scale of the empirical Gram matrix across layers, we use 
\begin{align*}
    \mathbf{S}_l:=\mathbf{S}_l(\gamma)=
\sqrt{\bar h_l}\,
\operatorname{diag} (
\frac{s_l(\gamma)}
{\sqrt{\operatorname{mean}(s_l(\gamma)^2)}} 
)
.
\end{align*}
with $\bar h_l:= \operatorname{mean}(\operatorname{diag}(\mathbf{X}_l\mathbf{X}_l^\top))$ and $s_l(\gamma)$ is defined channel-wise by
$s_l^{(j)}(\gamma):=
\nicefrac{\mathrm{mean}(|\mathbf{X}_l^{(j)}|)^{\gamma}}
{\mathrm{mean}(|\mathbf{W}_l^{(j)}|)^{1-\gamma}}$ where $\gamma$ controls the relative contribution of activation and weight statistics. This construction can be interpreted as a saliency-aware and scale-normalized curvature regularization of the Gram-based objective. This also aligns with Corollary~\ref{cor:constraint_penalty_supported} to ensure that the range of $\lambda$ is proper. It preserves the computational structure of GPTQ while modulating the relative importance of different channels.

For efficiency, hyperparameter selection is performed on a reduced subset of input channels. The calibration data are split into training and validation subsets. The training subset is used to construct curvature statistics and perform sequential quantization, while the validation subset is used to evaluate candidate configurations. The best $(\lambda,\gamma)$ pair is then applied to the full layer.

\paragraph{Implementation Details for SARQC with OmniQuant}

In the ablation study, we extend OmniQuant \citep{omniquant} by incorporating the SARQC objective into its differentiable quantization framework. Unlike the layer-wise linear form used in the main method, OmniQuant operates on transformer-block outputs. Accordingly, let $\mathcal F(\cdot,\cdot)$ denote the forward mapping of the corresponding transformer block. In the special case of a single linear layer, this reduces to the familiar form $\mathbf W_l \mathbf X_l$.

Let $\widehat{\mathbf W}_l(\Theta_1,\Theta_2)$ and $\widehat{\mathbf X}_l(\Theta_2)$ denote the differentiable quantized--dequantized weights and activations induced by the OmniQuant parameters $(\Theta_1,\Theta_2)$. The original OmniQuant objective can be written as
\begin{align*}
\min_{\Theta_1,\Theta_2}\;
\bigl\|
\mathcal F(\mathbf W_l,\mathbf X_l)
-
\mathcal F\bigl(\widehat{\mathbf W}_l(\Theta_1,\Theta_2),\,\widehat{\mathbf X}_l(\Theta_2)\bigr)
\bigr\|_\mathrm{F}^2.
\end{align*}

To incorporate SARQC, we augment this objective with the saliency-aware regularization term in \Cref{eq:rqc_layer}. Specifically, for each layer or block $l$, we optimize
\begin{align*}
\min_{\Theta_1,\Theta_2}\quad
\underbrace{
\bigl\|
\mathcal F(\mathbf W_l,\mathbf X_l)
-
\mathcal F\bigl(\widehat{\mathbf W}_l(\Theta_1,\Theta_2),\,\widehat{\mathbf X}_l(\Theta_2)\bigr)
\bigr\|_\mathrm{F}^2
}_{\mathcal L_{\mathrm{recon}}}
+
\lambda
\underbrace{
\bigl\|
\bigl(\widehat{\mathbf W}_l(\Theta_1,\Theta_2)-\mathbf W_l\bigr)\mathbf S_l
\bigr\|_\mathrm{F}^2
}_{\mathcal L_{\mathrm{sar}}}.
\end{align*}

Here, $\mathbf S_l=\mathrm{diag}(s_l)$ is the saliency matrix, with channel-wise saliency weights defined by
\begin{align*}
s_l^{(j)}
=
\frac{\mathrm{mean}(|\mathbf X_l^{(j)}|)^{1/2}}
{\mathrm{mean}(|\mathbf W_l^{(j)}|)^{1/2}}.
\end{align*}

The reconstruction term follows the original OmniQuant formulation and remains differentiable with respect to $(\Theta_1,\Theta_2)$ via straight-through estimators. The SARQC regularization term is likewise differentiable, since it is applied directly to the dequantized weights $\widehat{\mathbf W}_l(\Theta_1,\Theta_2)$. In practice, we optimize $(\Theta_1,\Theta_2)$ using Adam with the same learning-rate schedule as OmniQuant. The regularization weight $\lambda$ is selected from $\{0.01,0.05,0.1,0.15\}$ based on validation perplexity.

\clearpage
\section{Extra Experimental Results}
\label{appx:extra experiments}

\subsection{Details of \Cref{fig:illustration}\emph{(b)}}
\label{appx:detail_of_figure1}

\Cref{tab:sarqc_gbs_lambda_llama2_7b_w4a16} reports the exact zero-shot accuracies underlying \Cref{fig:illustration}. The two accuracy curves in \Cref{fig:illustration} are obtained from the average accuracies in the last column for \texttt{SARQC-GBS(Identity)} and \texttt{SARQC-GBS(Saliency)}, respectively. The figure also overlays the reconstruction term and the regularization term to illustrate how the trade-off controlled by $\lambda$ affects downstream performance. For visualization, $\mathcal{L}_{\mathrm{recon}}$ and $\mathcal{L}_{\mathrm{sar}}$ are min--max normalized to $[0,1]$ over $\lambda\in\{0.1,0.2,\dots,1.0\}$. Since the raw reconstruction loss at $\lambda=0$ is much larger than the others, directly including it would compress the remaining curve; we therefore plot the point at $\lambda=0$ using the same normalization rule determined from $\lambda\in\{0.1,\dots,1.0\}$. This preserves the outlier nature of $\lambda=0$ while keeping the overall trend visually interpretable. We also conduct the same study for \texttt{SARQC-GS}. The corresponding curve is shown in \Cref{fig:placeholder3}, and the exact zero-shot accuracies for different $\lambda$ are reported in \Cref{tab:sarqc_gs_lambda_llama2_7b_w4a16}.

\begin{figure}[hbtp]
    \centering
    \includegraphics[width=0.7\linewidth]{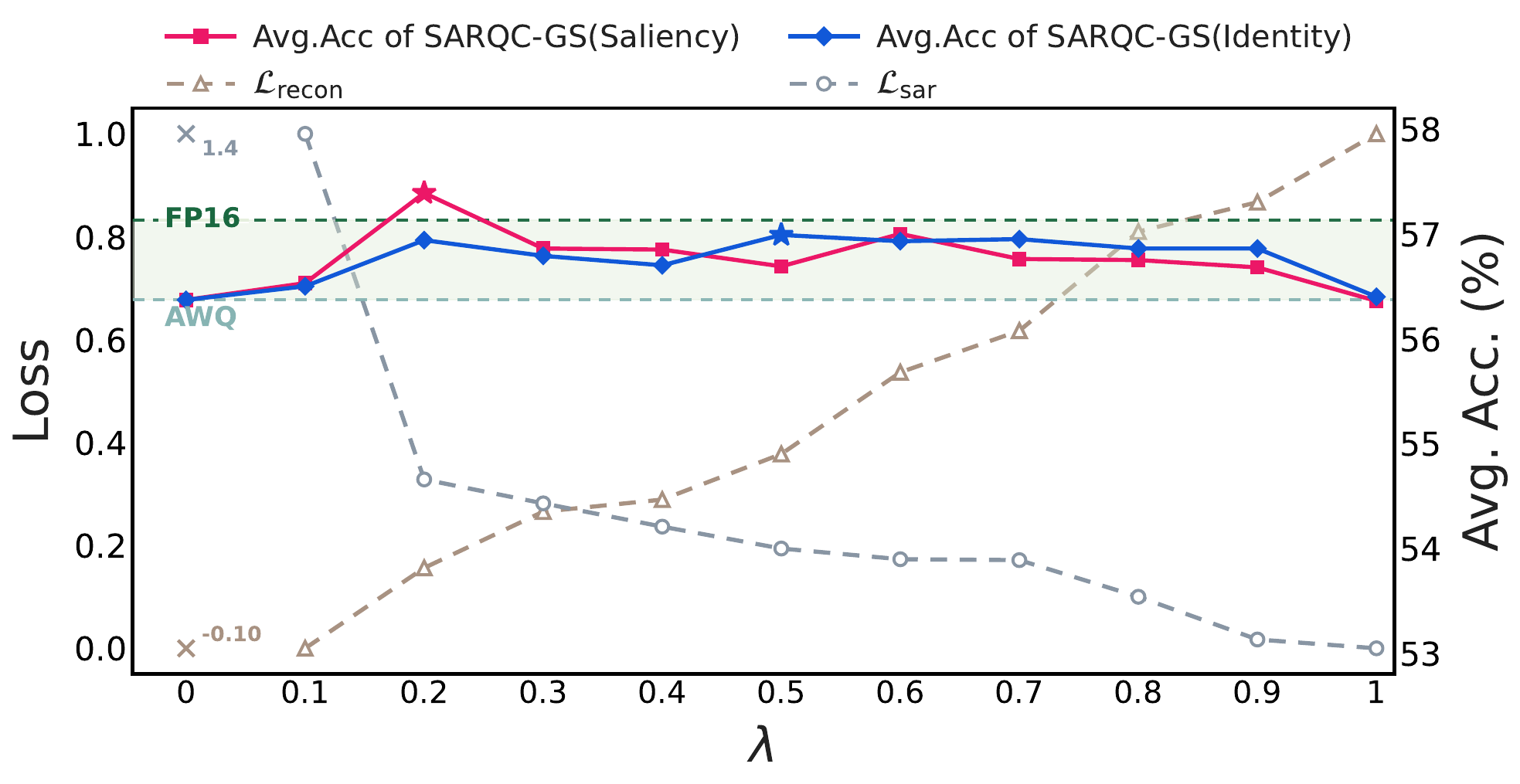}
    \caption{Effect of the regularization strength $\lambda$ for SARQC-GS on LLaMA2-7B under INT4 weight-only quantization.}
    \label{fig:placeholder3}
\end{figure}

\begin{table*}[htbp]
\centering
\caption{Zero-shot accuracy (\%) of SARQC-GBS on LLaMA2-7B under W4A16 for different $\lambda$. For SARQC-GBS(Saliency) $\gamma$ is fixed at $0.5$.}
\label{tab:sarqc_gbs_lambda_llama2_7b_w4a16}
\setlength{\tabcolsep}{4.3pt}
\renewcommand{\arraystretch}{1.05}
\resizebox{\textwidth}{!}{%
\begin{tabular}{l c c c c c c c c c c}
\toprule
Method & $\lambda$ & PIQA & HellaSwag & MMLU & HumanEval & BoolQ & WinoGrande & ARC-E & ARC-C & Avg. \\
\midrule
 FP16& & 78.13 & 57.12 & 41.80 & 12.80 & 79.27 & 69.46 & 75.46 & 43.00 & 57.13 \\
\midrule
 \multirow{11}{*}{SARQC-GBS(Identity)}
& 0 & 77.14 & 54.71 & 31.66 & 11.59 & 75.23 & 67.64 & 73.40 & 40.87 & 54.03 \\
& 0.1 & 77.15 & 56.03 & 39.84 & 10.98 & 76.73 & 67.96 & 74.24 & 41.98 & 55.61 \\
& 0.2 & 77.48 & 55.92 & 37.30 & 11.59 & 77.65 & 69.38 & 74.83 & 41.47 & 55.70 \\
& 0.3 & 77.75 & 55.95 & 40.04 & 14.02 & 77.86 & 69.53 & 75.08 & 41.64 & 56.48 \\
& 0.4 & 77.64 & 55.95 & 39.51 & 13.41 & 76.88 & 70.17 & 75.42 & 41.89 & 56.36 \\
& 0.5 & 77.37 & 56.07 & 38.39 & 12.20 & 77.06 & 68.35 & 75.00 & 41.64 & 55.76 \\
& 0.6 & 77.69 & 56.23 & 37.80 & 11.59 & 77.71 & 69.38 & 75.51 & 42.15 & 56.01 \\
& 0.7 & 77.80 & 56.10 & 39.62 & 11.59 & 77.58 & 69.06 & 75.34 & 41.64 & 56.09 \\
& 0.8 & 77.31 & 56.19 & 36.64 & 10.98 & 75.78 & 68.90 & 75.29 & 42.41 & 55.44 \\
& 0.9 & 77.48 & 56.31 & 38.26 & 12.80 & 76.79 & 69.30 & 75.00 & 42.32 & 56.03 \\
& 1.0 & 77.31 & 56.29 & 38.74 & 10.98 & 76.97 & 68.59 & 75.46 & 41.64 & 55.75 \\
\midrule
\multirow{11}{*}{SARQC-GBS(Saliency)}
& 0 & 77.14 & 54.71 & 31.66 & 11.59 & 75.23 & 67.64 & 73.40 & 40.87 & 54.03 \\
& 0.1 & 77.37 & 55.81 & 38.41 & 10.37 & 77.68 & 70.01 & 74.79 & 43.34 & 55.97 \\
& 0.2 & 77.86 & 56.17 & 40.08 & 11.59 & 78.17 & 69.85 & 75.80 & 43.52 & 56.63 \\
& 0.3 & 78.18 & 56.51 & 40.12 & 14.63 & 78.32 & 69.85 & 75.25 & 42.24 & 56.89 \\
& 0.4 & 78.13 & 56.11 & 39.85 & 14.02 & 77.65 & 69.38 & 74.75 & 43.09 & 56.62 \\
& 0.5 & 78.29 & 55.82 & 40.34 & 10.98 & 77.92 & 69.30 & 74.24 & 42.92 & 56.23 \\
& 0.6 & 77.97 & 55.92 & 39.63 & 11.59 & 77.89 & 69.46 & 75.08 & 43.09 & 56.33 \\
& 0.7 & 78.13 & 55.76 & 39.41 & 10.37 & 77.25 & 68.98 & 75.00 & 43.09 & 56.00 \\
& 0.8 & 77.26 & 56.09 & 36.95 & 9.76 & 77.92 & 69.85 & 74.92 & 42.66 & 55.68 \\
& 0.9 & 77.64 & 55.96 & 34.78 & 11.59 & 76.85 & 69.61 & 75.42 & 42.49 & 55.54 \\
& 1.0 & 77.26 & 55.92 & 34.75 & 10.37 & 76.76 & 69.46 & 75.13 & 41.89 & 55.19 \\
\bottomrule
\end{tabular}}
\end{table*}

\begin{table*}[htbp]
\centering
\caption{Zero-shot accuracy (\%) of SARQC-GS on LLaMA2-7B under W4A16 for different $\lambda$.}
\label{tab:sarqc_gs_lambda_llama2_7b_w4a16}
\setlength{\tabcolsep}{4.3pt}
\renewcommand{\arraystretch}{1.05}
\resizebox{\textwidth}{!}{%
\begin{tabular}{l c c c c c c c c c c}
\toprule
Method & $\lambda$ & PIQA & HellaSwag & MMLU & HumanEval & BoolQ & WinoGrande & ARC-E & ARC-C & Avg. \\
\midrule
 FP16& & 78.13 & 57.12 & 41.80 & 12.80 & 79.27 & 69.46 & 75.46 & 43.00 & 57.13 \\
\midrule
\multirow{11}{*}{SARQC-GS(Identity)}
& 0 & 77.58 & 56.58 & 40.43 & 11.59 & 79.94 & 68.19 & 74.96 & 41.72 & 56.37 \\
& 0.1 & 78.07 & 56.27 & 40.41 & 11.59 & 78.90 & 69.30 & 75.21 & 42.24 & 56.50 \\
& 0.2 & 77.80 & 56.36 & 42.02 & 15.24 & 79.82 & 68.11 & 74.33 & 41.81 & 56.94 \\
& 0.3 & 77.75 & 56.31 & 40.60 & 13.41 & 79.48 & 68.90 & 75.46 & 42.41 & 56.79 \\
& 0.4 & 77.75 & 56.34 & 40.48 & 13.41 & 79.20 & 69.06 & 75.21 & 42.15 & 56.70 \\
& 0.5 & 77.86 & 56.34 & 40.82 & 14.02 & 79.54 & 69.38 & 75.59 & 42.41 & 56.99 \\
& 0.6 & 78.07 & 56.33 & 41.09 & 14.02 & 79.39 & 68.98 & 75.59 & 41.98 & 56.93 \\
& 0.7 & 77.97 & 56.35 & 40.88 & 14.02 & 79.57 & 69.30 & 75.42 & 42.06 & 56.95 \\
& 0.8 & 77.91 & 56.32 & 41.05 & 13.41 & 79.48 & 68.98 & 75.55 & 42.15 & 56.86 \\
& 0.9 & 77.97 & 56.30 & 40.57 & 12.80 & 79.51 & 70.09 & 75.25 & 42.41 & 56.86 \\
& 1.0 & 77.80 & 56.28 & 40.28 & 11.59 & 79.17 & 69.46 & 74.79 & 41.81 & 56.40 \\
\midrule
\multirow{10}{*}{SARQC-GS(Saliency)}
& 0 & 77.58 & 56.58 & 40.43 & 11.59 & 79.94 & 68.19 & 74.96 & 41.72 & 56.37 \\
& 0.1 & 77.86 & 56.26 & 40.44 & 12.20 & 78.96 & 69.06 & 75.25 & 42.24 & 56.53 \\
& 0.2 & 78.02 & 56.33 & 40.89 & 16.46 & 80.15 & 69.22 & 75.76 & 42.32 & 57.39 \\
& 0.3 & 77.80 & 56.29 & 40.67 & 14.02 & 79.42 & 69.06 & 75.38 & 42.24 & 56.86 \\
& 0.4 & 77.91 & 56.35 & 40.59 & 14.63 & 79.20 & 68.75 & 75.21 & 42.15 & 56.85 \\
& 0.5 & 77.80 & 56.28 & 40.89 & 12.20 & 79.66 & 69.22 & 75.51 & 41.98 & 56.69 \\
& 0.6 & 78.13 & 56.35 & 41.00 & 14.63 & 79.33 & 68.82 & 75.67 & 42.06 & 57.00 \\
& 0.7 & 77.86 & 56.28 & 40.91 & 12.80 & 79.51 & 69.22 & 75.42 & 42.06 & 56.76 \\
& 0.8 & 77.69 & 56.32 & 41.05 & 12.80 & 79.48 & 68.98 & 75.55 & 42.15 & 56.75 \\
& 0.9 & 77.80 & 56.28 & 40.44 & 11.59 & 79.57 & 70.09 & 75.25 & 42.41 & 56.68 \\
& 1.0 & 77.91 & 56.26 & 40.17 & 10.98 & 79.24 & 69.61 & 74.92 & 41.81 & 56.36 \\
\bottomrule
\end{tabular}}
\end{table*}

\subsection{Visualization of Weight Drift}
\label{app:3d_weight_diff}

To further support the motivation in \Cref{fig:illustration}, \Cref{fig:3d_weight_diff} visualizes the absolute differences, $|\mathbf{W}_l-\widehat{\mathbf{W}}_l|$, for a representative layer of the INT4-quantized LLaMA2-7B model. The two panels correspond to the $\lambda=0$ and $\lambda=0.3$ cases in \Cref{fig:illustration}\emph{(b)}. When $\lambda=0$, corresponding to vanilla calibration, the quantized weights exhibit larger deviation from the original FP weights, indicating more severe weight drift. With moderate regularization at $\lambda=0.3$, the discrepancy becomes visibly more controlled. This directly illustrates in weight space that the proposed regularizer can constrain $\widehat{\mathbf{W}}_l$ from drifting too far away from $\mathbf{W}_l$.

\begin{figure}[hbtp]
    \centering
    \begin{subfigure}[t]{0.44\linewidth}
        \centering
        \includegraphics[width=\linewidth]{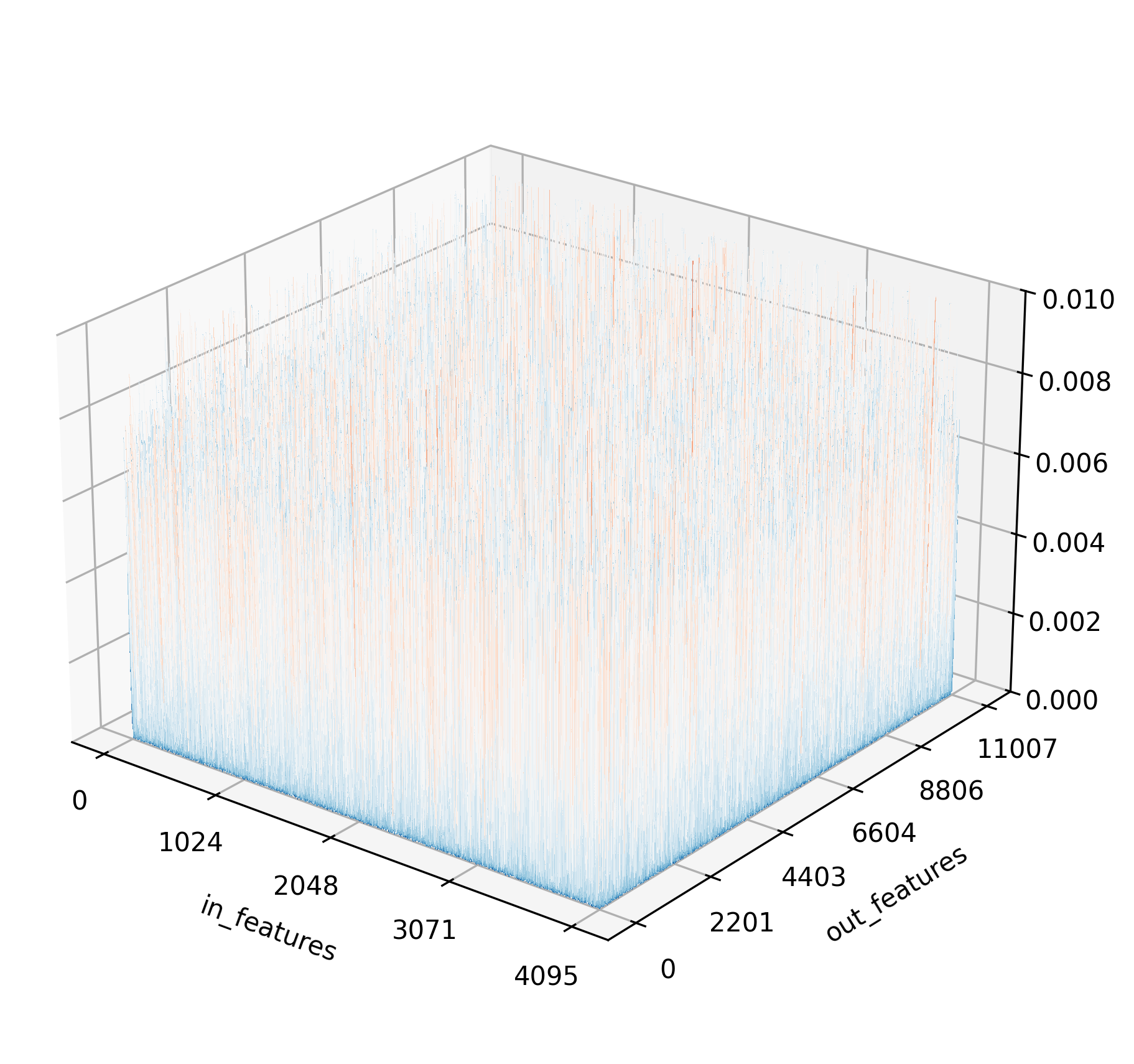}
        \caption{$\lambda=0$}
    \end{subfigure}\hfill
    \begin{subfigure}[t]{0.44\linewidth}
        \centering
        \includegraphics[width=\linewidth]{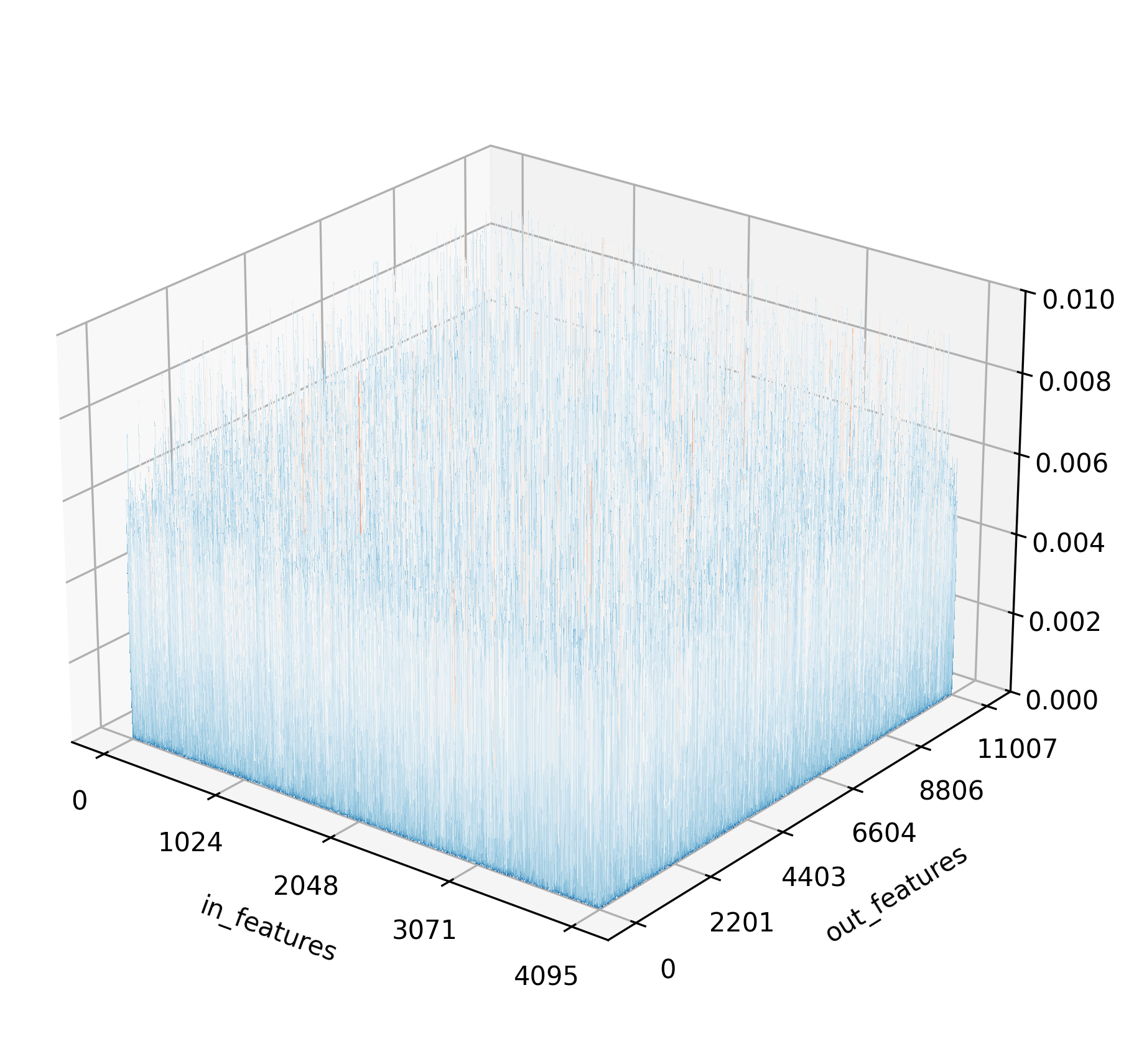}
        \caption{$\lambda=0.3$}
    \end{subfigure}
    \caption{Visualization of the element-wise weight discrepancy $|\mathbf{W}_l-\widehat{\mathbf{W}}_l|$ for a representative layer of INT4-quantized LLaMA2-7B under different regularization strengths. The two panels correspond to the $\lambda=0$ and $\lambda=0.3$ cases in \Cref{fig:illustration}\emph{(b)}.}
    \label{fig:3d_weight_diff}
\end{figure}

\subsection{Results under W3A16}

\Cref{tab:rqc_zeroshot_acc_w3a16_others} reports zero-shot accuracy under W3A16 for both dense and MoE models. For W3A16, we focus on seven non-code zero-shot benchmarks and exclude HumanEval from the averaged score, since code-generation performance becomes highly unstable and substantially degraded under this aggressive quantization setting. Overall, \texttt{SARQC-GBS} consistently improves over strong Gram-based PTQ baselines under this more aggressive quantization setting. In particular, \texttt{SARQC-GBS(Saliency)} achieves the best average accuracy among quantized methods on four out of the five evaluated models, while \texttt{SARQC-GBS(Identity)} performs best on the remaining model. Comparing the two SARQC variants, the saliency-aware version tends to outperform the identity-based one on most dense models and on two of the three MoE models, indicating that saliency-aware regularization is generally helpful under low-bit weight-only quantization.

\begin{table*}[htbp]
\centering
\caption{Zero-shot accuracy (\%) on multiple benchmarks under W3A16.}
\label{tab:rqc_zeroshot_acc_w3a16_others}
\setlength{\tabcolsep}{5.0pt}
\renewcommand{\arraystretch}{1.06}
\definecolor{rqcrow}{RGB}{244,220,217}

\resizebox{1.\textwidth}{!}{%
\begin{tabular}{l l c c c c c c c c}
\toprule
Model & Method & PIQA & HellaSwag & MMLU & BoolQ & WinoGrande & ARC-E & ARC-C & Avg. \\
\midrule

\multirow{5}{*}{LLaMA2-7B}
& FP16      & 78.13 & 57.12 & 41.80 & 79.27 & 69.46 & 75.46 & 43.00 & 63.46 \\
& GPTQ      & 73.94 & 51.48 & 29.55 & 69.42 & 64.88 & 68.69 & 35.49 & 56.21 \\
& GPTAQ     & 74.54 & 52.17 & 29.01 & 70.76 & 66.14 & 67.55 & 34.81 & 56.43 \\
& \cellcolor[RGB]{234, 239, 247}SARQC-GBS(Identity)  & \cellcolor[RGB]{234, 239, 247}75.19 & \cellcolor[RGB]{234, 239, 247}52.65 & \cellcolor[RGB]{234, 239, 247}32.32 & \cellcolor[RGB]{234, 239, 247}72.94 & \cellcolor[RGB]{234, 239, 247}67.56 & \cellcolor[RGB]{234, 239, 247}70.50 & \cellcolor[RGB]{234, 239, 247}37.88 & \cellcolor[RGB]{234, 239, 247}58.43 \\
& \cellcolor[RGB]{210, 222, 239}SARQC-GBS(Saliency) & \cellcolor[RGB]{210, 222, 239}76.33 & \cellcolor[RGB]{210, 222, 239}53.18 & \cellcolor[RGB]{210, 222, 239}33.60 & \cellcolor[RGB]{210, 222, 239}69.66 & \cellcolor[RGB]{210, 222, 239}68.51 & \cellcolor[RGB]{210, 222, 239}70.71 & \cellcolor[RGB]{210, 222, 239}39.16 & \cellcolor[RGB]{210, 222, 239}\textbf{58.74} \\
\midrule

\multirow{5}{*}{LLaMA2-13B}
& FP16      & 79.49 & 60.21 & 52.41 & 82.11 & 72.53 & 78.96 & 47.35 & 67.58 \\
& GPTQ & 73.88 & 52.54 & 37.85 & 75.84 & 62.27 & 70.37 & 38.40 & 58.74 \\
& GPTAQ & 74.27 & 52.71 & 38.07 & 75.72 & 62.90 & 70.88 & 38.65 & 59.03 \\
& \cellcolor[RGB]{234, 239, 247}SARQC-GBS(Identity)
& \cellcolor[RGB]{234, 239, 247}74.37
& \cellcolor[RGB]{234, 239, 247}52.77
& \cellcolor[RGB]{234, 239, 247}39.21
& \cellcolor[RGB]{234, 239, 247}75.78
& \cellcolor[RGB]{234, 239, 247}63.14
& \cellcolor[RGB]{234, 239, 247}71.13
& \cellcolor[RGB]{234, 239, 247}38.99
& \cellcolor[RGB]{234, 239, 247}59.34 \\
& \cellcolor[RGB]{210, 222, 239}SARQC-GBS(Saliency) & \cellcolor[RGB]{210, 222, 239}74.97 & \cellcolor[RGB]{210, 222, 239}53.35 & \cellcolor[RGB]{210, 222, 239}40.24 & \cellcolor[RGB]{210, 222, 239}76.02 & \cellcolor[RGB]{210, 222, 239}64.48 & \cellcolor[RGB]{210, 222, 239}72.18 & \cellcolor[RGB]{210, 222, 239}39.76 & \cellcolor[RGB]{210, 222, 239}\textbf{60.14} \\
\midrule

\multirow{4}{*}{DeepSeek-MoE-16B}
& BF16      & 78.73 & 58.10 & 38.21 & 74.07 & 70.01 & 74.87 & 43.94 & 62.56 \\
& GPTQ      & 77.80 & 54.64 & 31.31 & 70.86 & 68.82 & 69.78 & 36.60 & 58.54 \\
& \cellcolor[RGB]{234, 239, 247}SARQC-GBS(Identity) & \cellcolor[RGB]{234, 239, 247}77.31 & \cellcolor[RGB]{234, 239, 247}54.92 & \cellcolor[RGB]{234, 239, 247}30.69 & \cellcolor[RGB]{234, 239, 247}72.23 & \cellcolor[RGB]{234, 239, 247}69.06 & \cellcolor[RGB]{234, 239, 247}71.97 & \cellcolor[RGB]{234, 239, 247}40.02 & \cellcolor[RGB]{234, 239, 247}\textbf{59.46} \\
& \cellcolor[RGB]{210, 222, 239}SARQC-GBS(Saliency) & \cellcolor[RGB]{210, 222, 239}77.91 & \cellcolor[RGB]{210, 222, 239}55.52 & \cellcolor[RGB]{210, 222, 239}32.84 & \cellcolor[RGB]{210, 222, 239}72.66 & \cellcolor[RGB]{210, 222, 239}67.80 & \cellcolor[RGB]{210, 222, 239}71.30 & \cellcolor[RGB]{210, 222, 239}37.20 & \cellcolor[RGB]{210, 222, 239}59.32 \\
\midrule

\multirow{4}{*}{Qwen3-MoE-30B}
& BF16      & 79.82 & 62.35 & 78.76 & 80.61 & 72.14 & 80.05 & 54.01 & 72.53 \\
& GPTQ      & 77.09 & 57.76 & 62.14 & 81.25 & 70.64 & 74.28 & 45.65 & 66.97 \\
& \cellcolor[RGB]{234, 239, 247}SARQC-GBS(Identity) & \cellcolor[RGB]{234, 239, 247}78.51 & \cellcolor[RGB]{234, 239, 247}58.35 & \cellcolor[RGB]{234, 239, 247}74.27 & \cellcolor[RGB]{234, 239, 247}80.12 & \cellcolor[RGB]{234, 239, 247}71.82 & \cellcolor[RGB]{234, 239, 247}78.32 & \cellcolor[RGB]{234, 239, 247}47.53 & \cellcolor[RGB]{234, 239, 247}69.85 \\
& \cellcolor[RGB]{210, 222, 239}SARQC-GBS(Saliency) & \cellcolor[RGB]{210, 222, 239}79.00 & \cellcolor[RGB]{210, 222, 239}59.56 & \cellcolor[RGB]{210, 222, 239}73.36 & \cellcolor[RGB]{210, 222, 239}81.90 & \cellcolor[RGB]{210, 222, 239}68.67 & \cellcolor[RGB]{210, 222, 239}79.21 & \cellcolor[RGB]{210, 222, 239}50.00 & \cellcolor[RGB]{210, 222, 239}\textbf{70.24} \\
\midrule

\multirow{4}{*}{Mixtral-8x7B}
& BF16      & 82.54 & 65.10 & 68.12 & 85.90 & 77.19 & 83.71 & 56.74 & 74.19 \\
& GPTQ      & 80.36 & 59.74 & 57.21 & 82.11 & 73.88 & 77.99 & 48.12 & 68.49 \\
& \cellcolor[RGB]{234, 239, 247}SARQC-GBS(Identity) & \cellcolor[RGB]{234, 239, 247}79.87 & \cellcolor[RGB]{234, 239, 247}59.96 & \cellcolor[RGB]{234, 239, 247}60.83 & \cellcolor[RGB]{234, 239, 247}82.35 & \cellcolor[RGB]{234, 239, 247}73.88 & \cellcolor[RGB]{234, 239, 247}79.00 & \cellcolor[RGB]{234, 239, 247}48.21 & \cellcolor[RGB]{234, 239, 247}69.16 \\
& \cellcolor[RGB]{210, 222, 239}SARQC-GBS(Saliency) & \cellcolor[RGB]{210, 222, 239}81.45 & \cellcolor[RGB]{210, 222, 239}61.29 & \cellcolor[RGB]{210, 222, 239}61.29 & \cellcolor[RGB]{210, 222, 239}82.60 & \cellcolor[RGB]{210, 222, 239}74.66 & \cellcolor[RGB]{210, 222, 239}80.60 & \cellcolor[RGB]{210, 222, 239}49.91 & \cellcolor[RGB]{210, 222, 239}\textbf{70.26} \\
\bottomrule
\end{tabular}
}
\end{table*}

\subsection{Results on LLaMA-7B, LLaMA-13B, and LLaMA-30B}
\label{appx:detail_of_llama1}
\Cref{tab:rqc_zeroshot_acc_llama1_merged} reports zero-shot accuracy for LLaMA-7B, LLaMA-13B, and LLaMA-30B under both W4A16 and W3A16. Overall, SARQC yields consistent improvements over standard PTQ baselines across model scales, with the gains being more pronounced under the more challenging W3A16 setting.
\begin{table*}[htbp]
\centering
\caption{Zero-shot accuracy (\%) on multiple benchmarks for LLaMA models under W4A16 and W3A16.}
\label{tab:rqc_zeroshot_acc_llama1_merged}
\setlength{\tabcolsep}{4.6pt}
\renewcommand{\arraystretch}{1.05}
\definecolor{rqcrow}{RGB}{244,220,217}

\resizebox{\textwidth}{!}{%
\begin{tabular}{l l l c c c c c c c c c}
\toprule
Model & Precision & Method & PIQA & HellaSwag & MMLU & HumanEval & BoolQ & WinoGrande & ARC-E & ARC-C & Avg. \\
\midrule

\multirow{13}{*}{LLaMA-7B}
& \multirow{8}{*}{W4A16}
& FP16      & 78.73 & 57.10 & 32.22 & 13.41 & 76.79 & 70.17 & 75.55 & 42.41 & 55.80 \\
&& AWQ       & 77.58 & 56.32 & 32.00 & 8.54  & 76.57 & 68.82 & 74.96 & 41.55 & 54.54 \\
&& GPTQ      & 77.97 & 55.87 & 28.66 & 7.93  & 76.02 & 68.51 & 72.56 & 41.30 & 53.60 \\
&& GPTAQ     & 78.84 & 56.12 & 29.78 & 7.93  & 75.72 & 68.19 & 73.65 & 41.04 & 53.91 \\
&& \cellcolor{identityrow}SARQC-GS(Identity)   & \cellcolor{identityrow}78.73 & \cellcolor{identityrow}56.41 & \cellcolor{identityrow}32.52 & \cellcolor{identityrow}12.20 & \cellcolor{identityrow}76.39 & \cellcolor{identityrow}68.59 & \cellcolor{identityrow}75.04 & \cellcolor{identityrow}41.81 & \cellcolor{identityrow}55.21 \\
&& \cellcolor{saliencyrow}SARQC-GS(Saliency)   & \cellcolor{saliencyrow}78.62 & \cellcolor{saliencyrow}56.47 & \cellcolor{saliencyrow}32.47 & \cellcolor{saliencyrow}10.98 & \cellcolor{saliencyrow}76.79 & \cellcolor{saliencyrow}69.38 & \cellcolor{saliencyrow}75.51 & \cellcolor{saliencyrow}42.15 & \cellcolor{saliencyrow}55.30 \\
&& \cellcolor{identityrow}SARQC-GBS(Identity)  & \cellcolor{identityrow}79.16 & \cellcolor{identityrow}56.26 & \cellcolor{identityrow}32.21 & \cellcolor{identityrow}8.54 & \cellcolor{identityrow}76.51 & \cellcolor{identityrow}68.35 & \cellcolor{identityrow}73.86 & \cellcolor{identityrow}41.72 & \cellcolor{identityrow}54.58 \\
&& \cellcolor{saliencyrow}SARQC-GBS(Saliency)  & \cellcolor{saliencyrow}80.03 & \cellcolor{saliencyrow}56.82 & \cellcolor{saliencyrow}33.02 & \cellcolor{saliencyrow}10.98 & \cellcolor{saliencyrow}77.06 & \cellcolor{saliencyrow}69.77 & \cellcolor{saliencyrow}73.70 & \cellcolor{saliencyrow}41.47 & \cellcolor{saliencyrow}\textbf{55.36} \\
\cmidrule(lr){3-12}
& \multirow{5}{*}{W3A16}
& FP16      & 78.73 & 57.10 & 32.22 & -- & 76.79 & 70.17 & 75.55 & 42.41 & 61.85 \\
&& GPTQ      & 75.35 & 51.80 & 26.46 & -- & 72.60 & 65.11 & 71.59 & 38.31 & 57.32 \\
&& GPTAQ     & 76.28 & 53.02 & 25.49 & -- & 71.93 & 67.88 & 70.71 & 38.57 & 57.70 \\
&& \cellcolor{identityrow}SARQC-GBS(Identity)  & \cellcolor{identityrow}75.90 & \cellcolor{identityrow}53.33 & \cellcolor{identityrow}27.60 & \cellcolor{identityrow}-- & \cellcolor{identityrow}71.83 & \cellcolor{identityrow}67.17 & \cellcolor{identityrow}70.62 & \cellcolor{identityrow}38.05 & \cellcolor{identityrow}57.79 \\
&& \cellcolor{saliencyrow}SARQC-GBS(Saliency)  & \cellcolor{saliencyrow}77.15 & \cellcolor{saliencyrow}53.81 & \cellcolor{saliencyrow}26.74 & \cellcolor{saliencyrow}-- & \cellcolor{saliencyrow}72.14 & \cellcolor{saliencyrow}67.64 & \cellcolor{saliencyrow}71.13 & \cellcolor{saliencyrow}38.74 & \cellcolor{saliencyrow}\textbf{58.19} \\
\midrule

\multirow{13}{*}{LLaMA-13B}
& \multirow{8}{*}{W4A16}
& FP16      & 79.65 & 60.04 & 44.02 & 15.24 & 79.82 & 73.56 & 77.31 & 45.90 & 59.44 \\
&& AWQ       & 78.94 & 59.47 & 42.42 & 13.41 & 79.36 & 72.22 & 77.15 & 44.80 & 58.47 \\
&& GPTQ      & 79.38 & 59.66 & 40.77 & 14.02 & 78.47 & 71.27 & 76.39 & 45.39 & 58.17 \\
&& GPTAQ     & 77.91 & 59.45 & 41.19 & 10.98 & 79.39 & 71.43 & 76.89 & 46.08 & 57.91 \\
&& \cellcolor{identityrow}SARQC-GS(Identity)   & \cellcolor{identityrow}79.00 & \cellcolor{identityrow}59.57 & \cellcolor{identityrow}42.87 & \cellcolor{identityrow}14.63 & \cellcolor{identityrow}78.99 & \cellcolor{identityrow}71.74 & \cellcolor{identityrow}76.60 & \cellcolor{identityrow}45.05 & \cellcolor{identityrow}58.56 \\
&& \cellcolor{saliencyrow}SARQC-GS(Saliency)   & \cellcolor{saliencyrow}79.22 & \cellcolor{saliencyrow}59.94 & \cellcolor{saliencyrow}42.93 & \cellcolor{saliencyrow}15.85 & \cellcolor{saliencyrow}79.94 & \cellcolor{saliencyrow}72.85 & \cellcolor{saliencyrow}77.69 & \cellcolor{saliencyrow}45.82 & \cellcolor{saliencyrow}\textbf{59.28} \\
&& \cellcolor{identityrow}SARQC-GBS(Identity)  & \cellcolor{identityrow}79.38 & \cellcolor{identityrow}59.26 & \cellcolor{identityrow}42.23 & \cellcolor{identityrow}14.63 & \cellcolor{identityrow}78.50 & \cellcolor{identityrow}71.67 & \cellcolor{identityrow}77.31 & \cellcolor{identityrow}45.99 & \cellcolor{identityrow}58.62 \\
&& \cellcolor{saliencyrow}SARQC-GBS(Saliency)  & \cellcolor{saliencyrow}79.92 & \cellcolor{saliencyrow}59.46 & \cellcolor{saliencyrow}42.52 & \cellcolor{saliencyrow}14.63 & \cellcolor{saliencyrow}79.20 & \cellcolor{saliencyrow}72.53 & \cellcolor{saliencyrow}77.31 & \cellcolor{saliencyrow}45.65 & \cellcolor{saliencyrow}58.90 \\
\cmidrule(lr){3-12}
& \multirow{5}{*}{W3A16}
& FP16      & 79.65 & 60.04 & 44.02 & -- & 79.82 & 73.56 & 77.31 & 45.90 & 65.76 \\
&& GPTQ      & 78.07 & 57.79 & 38.19 & -- & 76.02 & 69.69 & 75.00 & 43.09 & 62.55 \\
&& GPTAQ     & 77.64 & 54.29 & 30.49 & -- & 70.89 & 70.09 & 71.34 & 40.44 & 59.31 \\
&& \cellcolor{identityrow}SARQC-GBS(Identity)  & \cellcolor{identityrow}78.02 & \cellcolor{identityrow}56.19 & \cellcolor{identityrow}37.04 & \cellcolor{identityrow}-- & \cellcolor{identityrow}76.27 & \cellcolor{identityrow}71.43 & \cellcolor{identityrow}74.83 & \cellcolor{identityrow}43.52 & \cellcolor{identityrow}62.47 \\
&& \cellcolor{saliencyrow}SARQC-GBS(Saliency)  & \cellcolor{saliencyrow}78.02 & \cellcolor{saliencyrow}57.55 & \cellcolor{saliencyrow}38.96 & \cellcolor{saliencyrow}-- & \cellcolor{saliencyrow}78.26 & \cellcolor{saliencyrow}70.88 & \cellcolor{saliencyrow}75.51 & \cellcolor{saliencyrow}43.60 & \cellcolor{saliencyrow}\textbf{63.25} \\
\midrule

\multirow{13}{*}{LLaMA-30B}
& \multirow{8}{*}{W4A16}
& FP16      & 80.69 & 63.38 & 54.45 & 20.73 & 83.79 & 75.93 & 80.64 & 51.54 & 63.89 \\
&& AWQ       & 80.69 & 62.93 & 54.12 & 20.73 & 83.70 & 76.40 & 80.26 & 50.94 & 63.72 \\
&& GPTQ      & 80.14 & 62.78 & 53.11 & 18.90 & 83.18 & 76.24 & 80.05 & 50.09 & 63.06 \\
&& GPTAQ      & 80.25 & 62.72 & 53.51 & 21.34 & 83.36 & 76.56 & 80.39 & 50.17 & 63.54 \\
&& \cellcolor{identityrow}SARQC-GS(Identity)   & \cellcolor{identityrow}80.79 & \cellcolor{identityrow}62.94 & \cellcolor{identityrow}54.00 & \cellcolor{identityrow}20.73 & \cellcolor{identityrow}83.49 & \cellcolor{identityrow}75.93 & \cellcolor{identityrow}80.22 & \cellcolor{identityrow}50.94 & \cellcolor{identityrow}63.63 \\
&& \cellcolor{saliencyrow}SARQC-GS(Saliency)  & \cellcolor{saliencyrow}80.69 & \cellcolor{saliencyrow}63.17 & \cellcolor{saliencyrow}54.12 & \cellcolor{saliencyrow}20.73 & \cellcolor{saliencyrow}84.07 & \cellcolor{saliencyrow}76.24 & \cellcolor{saliencyrow}80.60 & \cellcolor{saliencyrow}51.02 & \cellcolor{saliencyrow}\textbf{63.83} \\
&& \cellcolor{identityrow}SARQC-GBS(Identity)  & \cellcolor{identityrow}80.58 & \cellcolor{identityrow}62.97 & \cellcolor{identityrow}53.71 & \cellcolor{identityrow}18.90 & \cellcolor{identityrow}83.12 & \cellcolor{identityrow}75.37 & \cellcolor{identityrow}79.80 & \cellcolor{identityrow}50.09 & \cellcolor{identityrow}63.07 \\
&& \cellcolor{saliencyrow}SARQC-GBS(Saliency)  & \cellcolor{saliencyrow}80.96 & \cellcolor{saliencyrow}62.90 & \cellcolor{saliencyrow}53.55 & \cellcolor{saliencyrow}20.73 & \cellcolor{saliencyrow}83.70 & \cellcolor{saliencyrow}75.45 & \cellcolor{saliencyrow}80.09 & \cellcolor{saliencyrow}50.26 & \cellcolor{saliencyrow}63.45 \\
\cmidrule(lr){3-12}
& \multirow{5}{*}{W3A16}
& FP16      & 80.69 & 63.38 & 54.45 & -- & 83.79 & 75.93 & 80.64 & 51.54 & 70.06 \\
&& GPTQ      & 78.07 & 60.41 & 50.05 & -- & 80.40 & 73.09 & 78.49 & 46.67 & 66.74 \\
&& GPTAQ     & 78.84 & 60.60 & 49.39 & -- & 80.70 & 73.88 & 78.70 & 47.35 & 67.07 \\
&& \cellcolor{identityrow}SARQC-GBS(Identity)  & \cellcolor{identityrow}80.03 & \cellcolor{identityrow}61.27 & \cellcolor{identityrow}51.03 & \cellcolor{identityrow}-- & \cellcolor{identityrow}81.22 & \cellcolor{identityrow}75.06 & \cellcolor{identityrow}79.00 & \cellcolor{identityrow}47.44 & \cellcolor{identityrow}67.86 \\
&& \cellcolor{saliencyrow}SARQC-GBS(Saliency)  & \cellcolor{saliencyrow}80.14 & \cellcolor{saliencyrow}62.77 & \cellcolor{saliencyrow}51.52 & \cellcolor{saliencyrow}-- & \cellcolor{saliencyrow}82.45 & \cellcolor{saliencyrow}74.74 & \cellcolor{saliencyrow}79.46 & \cellcolor{saliencyrow}47.61 & \cellcolor{saliencyrow}\textbf{68.38} \\

\bottomrule
\end{tabular}
}
\end{table*}

\subsection{Ablation Study on Calibration Size}
\Cref{tab:sens to calibration data size} examines the sensitivity of SARQC to the calibration set size. Overall, \texttt{SARQC-GBS(Saliency)} remains stable across different calibration sizes and consistently achieves strong zero-shot performance, even when only a small number of calibration samples are available. This suggests that the proposed method is relatively robust to limited calibration data.

\definecolor{identityrow}{RGB}{234, 239, 247}
\definecolor{saliencyrow}{RGB}{210, 222, 239}

\subsection{Ablation Study on Calibration Corpus}

\Cref{tab:sens to calibration data} studies the sensitivity of the proposed methods to the choice of calibration corpus for LLaMA2-7B under W4A16. We report zero-shot accuracy on eight downstream benchmarks using two calibration datasets, C4 and Pile. Overall, the SARQC variants consistently outperform their corresponding baselines across both calibration corpora. Comparing identity-based and saliency-aware regularization, the saliency-aware variants are generally better in all four SARQC settings reported in the table. This supports the view that saliency-aware weighting helps preserve more important channels and improves generalization across different calibration data distributions.

\definecolor{identityrow}{RGB}{234, 239, 247}
\definecolor{saliencyrow}{RGB}{210, 222, 239}

\begin{table*}[h]
\centering
\caption{Sensitivity to the sample size of the calibration set: zero-shot accuracy (\%) on multiple benchmarks for LLaMA2-13B.}
\setlength{\tabcolsep}{3.6pt}
\renewcommand{\arraystretch}{1.10}
\resizebox{\textwidth}{!}{%
\begin{tabular}{ll l c c c c c c c c}
\toprule
Precision & Calib. Size & Method
& \textsc{PIQA} & \textsc{HellaSwag} & \textsc{MMLU} & \textsc{BoolQ}
& \textsc{WinoGrande} & \textsc{ARC-E} & \textsc{ARC-C} & Avg. (\%) \\
\midrule
FP16 &  &FP16     & 79.49 & 60.21 & 52.41 & 82.11 & 72.53 & 78.96 & 47.35 & 67.58 \\
\midrule
\multirow{8}{*}{W2A16}
& \multirow{2}{*}{16}
& GPTQ
& 51.41 & 25.91 & 25.20 & 37.89 & 49.09 & 25.72 & 21.50 & 33.82 \\
&
& \cellcolor{saliencyrow}SARQC-GBS(Saliency)
& \cellcolor{saliencyrow}52.88 & \cellcolor{saliencyrow}26.61 & \cellcolor{saliencyrow}26.29 & \cellcolor{saliencyrow}39.63 & \cellcolor{saliencyrow}51.46 & \cellcolor{saliencyrow}26.81 & \cellcolor{saliencyrow}23.21 & \cellcolor{saliencyrow}\textbf{35.27} \\
\cmidrule(lr){2-11}

& \multirow{2}{*}{32}
& GPTQ
& 52.50 & 25.73 & 24.85 & 41.68 & 50.67 & 25.67 & 21.16 & 34.61 \\
&
& \cellcolor{saliencyrow}SARQC-GBS(Saliency)
& \cellcolor{saliencyrow}52.77 & \cellcolor{saliencyrow}26.39 & \cellcolor{saliencyrow}26.55 & \cellcolor{saliencyrow}43.30 & \cellcolor{saliencyrow}51.70 & \cellcolor{saliencyrow}26.52 & \cellcolor{saliencyrow}23.46 & \cellcolor{saliencyrow}\textbf{35.81} \\
\cmidrule(lr){2-11}

& \multirow{2}{*}{64}
& GPTQ
& 52.12 & 26.10 & 24.85 & 45.84 & 49.72 & 25.93 & 20.82 & 35.05 \\
&
& \cellcolor{saliencyrow}SARQC-GBS(Saliency)
& \cellcolor{saliencyrow}52.50 & \cellcolor{saliencyrow}26.73 & \cellcolor{saliencyrow}26.74 & \cellcolor{saliencyrow}48.07 & \cellcolor{saliencyrow}50.91 & \cellcolor{saliencyrow}27.10 & \cellcolor{saliencyrow}21.67 & \cellcolor{saliencyrow}\textbf{36.25} \\
\cmidrule(lr){2-11}

& \multirow{2}{*}{128}
& GPTQ
& 51.80 & 26.17 & 24.99 & 45.54 & 49.33 & 25.59 & 20.56 & 34.85 \\
&
& \cellcolor{saliencyrow}SARQC-GBS(Saliency)
& \cellcolor{saliencyrow}52.29 & \cellcolor{saliencyrow}26.78 & \cellcolor{saliencyrow}26.79 & \cellcolor{saliencyrow}48.26 & \cellcolor{saliencyrow}51.22 & \cellcolor{saliencyrow}27.19 & \cellcolor{saliencyrow}21.93 & \cellcolor{saliencyrow}\textbf{36.35} \\
\midrule

\multirow{8}{*}{W3A16}
& \multirow{2}{*}{16}
& GPTQ
& 73.34 & 51.79 & 36.17 & 74.77 & 61.25 & 69.53 & 37.97 & 57.83 \\
&
& \cellcolor{saliencyrow}SARQC-GBS(Saliency)
& \cellcolor{saliencyrow}75.03 & \cellcolor{saliencyrow}52.56 & \cellcolor{saliencyrow}37.27 & \cellcolor{saliencyrow}76.64 & \cellcolor{saliencyrow}63.54 & \cellcolor{saliencyrow}72.52 & \cellcolor{saliencyrow}40.02 & \cellcolor{saliencyrow}\textbf{59.65} \\
\cmidrule(lr){2-11}

& \multirow{2}{*}{32}
& GPTQ
& 72.52 & 52.05 & 38.97 & 73.67 & 61.48 & 71.68 & 38.99 & 58.48 \\
&
& \cellcolor{saliencyrow}SARQC-GBS(Saliency)
& \cellcolor{saliencyrow}73.83 & \cellcolor{saliencyrow}52.30 & \cellcolor{saliencyrow}40.06 & \cellcolor{saliencyrow}76.33 & \cellcolor{saliencyrow}63.61 & \cellcolor{saliencyrow}73.82 & \cellcolor{saliencyrow}40.87 & \cellcolor{saliencyrow}\textbf{60.12} \\
\cmidrule(lr){2-11}

& \multirow{2}{*}{64}
& GPTQ
& 73.94 & 51.57 & 37.53 & 76.61 & 62.67 & 72.43 & 38.91 & 59.09 \\
&
& \cellcolor{saliencyrow}SARQC-GBS(Saliency)
& \cellcolor{saliencyrow}75.08 & \cellcolor{saliencyrow}53.39 & \cellcolor{saliencyrow}37.55 & \cellcolor{saliencyrow}78.99 & \cellcolor{saliencyrow}64.64 & \cellcolor{saliencyrow}73.06 & \cellcolor{saliencyrow}40.10 & \cellcolor{saliencyrow}\textbf{60.40} \\
\cmidrule(lr){2-11}

& \multirow{2}{*}{128}
& GPTQ & 73.88 & 52.54 & 37.85 & 75.84 & 62.27 & 70.37 & 38.40 & 58.74 \\
&
& \cellcolor[RGB]{210, 222, 239}SARQC-GBS(Saliency) & \cellcolor[RGB]{210, 222, 239}74.97 & \cellcolor[RGB]{210, 222, 239}53.35 & \cellcolor[RGB]{210, 222, 239}40.24 & \cellcolor[RGB]{210, 222, 239}76.02 & \cellcolor[RGB]{210, 222, 239}64.48 & \cellcolor[RGB]{210, 222, 239}72.18 & \cellcolor[RGB]{210, 222, 239}39.76 & \cellcolor[RGB]{210, 222, 239}\textbf{60.14} \\
\bottomrule
\end{tabular}
}
\label{tab:sens to calibration data size}
\end{table*}

\begin{table*}[h]
\centering
\caption{Sensitivity to the choice of the calibration set: zero-shot accuracy (\%) on multiple benchmarks under W4A16 for LLaMA2-7B models.}
\label{tab:sens to calibration data}
\setlength{\tabcolsep}{5.0pt}
\renewcommand{\arraystretch}{1.06}

\resizebox{0.9\textwidth}{!}{%
\begin{tabular}{l ll c c c c c c c c c}
\toprule
Calibration & Precision & Method & PIQA & HellaSwag & MMLU & HumanEval & BoolQ & WinoGrande & ARC-E & ARC-C & Avg. \\
\midrule

& FP16 & FP16      & 78.13 & 57.12 & 41.80 & 12.80 & 79.27 & 69.46 & 75.46 & 43.00 & 57.13 \\
\midrule

\multirow{7}{*}{C4}
 & \multirow{7}{*}{W4A16}
 & AWQ       & 77.69 & 56.70 & 40.92 & 12.20 & 79.02 & 68.90 & 75.42 & 41.89 & 56.59 \\
 && GPTQ      & 77.58 & 56.20 & 39.05 & 13.41 & 78.87 & 69.14 & 74.37 & 41.55 & 56.27 \\
 && GPTAQ     & 76.99 & 56.05 & 39.46 & 10.98 & 77.31 & 68.43 & 74.66 & 40.87 & 55.59 \\
 && \cellcolor{identityrow}SARQC-GS(Identity)  
    & \cellcolor{identityrow}78.07 & \cellcolor{identityrow}56.62 & \cellcolor{identityrow}41.70
    & \cellcolor{identityrow}15.85 & \cellcolor{identityrow}79.11 & \cellcolor{identityrow}68.43
    & \cellcolor{identityrow}74.96 & \cellcolor{identityrow}42.15 & \cellcolor{identityrow}57.11 \\
 && \cellcolor{saliencyrow}SARQC-GS(Saliency)  
    & \cellcolor{saliencyrow}78.13 & \cellcolor{saliencyrow}56.83 & \cellcolor{saliencyrow}41.26
    & \cellcolor{saliencyrow}16.46 & \cellcolor{saliencyrow}79.17 & \cellcolor{saliencyrow}69.30
    & \cellcolor{saliencyrow}75.29 & \cellcolor{saliencyrow}42.92 & \cellcolor{saliencyrow}\textbf{57.42} \\
 && \cellcolor{identityrow}SARQC-GBS(Identity)  
    & \cellcolor{identityrow}77.91 & \cellcolor{identityrow}56.28 & \cellcolor{identityrow}39.72
    & \cellcolor{identityrow}14.02 & \cellcolor{identityrow}77.37 & \cellcolor{identityrow}69.22
    & \cellcolor{identityrow}75.21 & \cellcolor{identityrow}42.24 & \cellcolor{identityrow}56.50 \\
 && \cellcolor{saliencyrow}SARQC-GBS(Saliency)
    & \cellcolor{saliencyrow}78.02 & \cellcolor{saliencyrow}56.57 & \cellcolor{saliencyrow}40.51
    & \cellcolor{saliencyrow}15.85 & \cellcolor{saliencyrow}79.20 & \cellcolor{saliencyrow}69.38
    & \cellcolor{saliencyrow}75.17 & \cellcolor{saliencyrow}42.92 & \cellcolor{saliencyrow}57.20 \\
\midrule

\multirow{7}{*}{Pile}
 & \multirow{7}{*}{W4A16}
 & AWQ       & 77.48 & 56.41 & 40.61 & 11.59 & 78.99 & 68.35 & 75.08 & 41.72 & 56.28 \\
 && GPTQ      & 77.31 & 55.98 & 38.76 & 12.20 & 78.26 & 69.22 & 74.03 & 41.13 & 55.86 \\
 && GPTAQ     & 76.88 & 56.00 & 39.24 & 10.37 & 77.00 & 68.98 & 74.24 & 40.53 & 55.41 \\
 && \cellcolor{identityrow}SARQC-GS(Identity)  
    & \cellcolor{identityrow}77.80 & \cellcolor{identityrow}56.35 & \cellcolor{identityrow}41.20
    & \cellcolor{identityrow}14.63 & \cellcolor{identityrow}79.11 & \cellcolor{identityrow}68.75
    & \cellcolor{identityrow}74.75 & \cellcolor{identityrow}41.98 & \cellcolor{identityrow}56.82 \\
 && \cellcolor{saliencyrow}SARQC-GS(Saliency)  
    & \cellcolor{saliencyrow}78.02 & \cellcolor{saliencyrow}56.48 & \cellcolor{saliencyrow}40.98
    & \cellcolor{saliencyrow}15.24 & \cellcolor{saliencyrow}79.14 & \cellcolor{saliencyrow}69.22
    & \cellcolor{saliencyrow}75.21 & \cellcolor{saliencyrow}42.58 & \cellcolor{saliencyrow}\textbf{57.11} \\
 && \cellcolor{identityrow}SARQC-GBS(Identity)
    & \cellcolor{identityrow}77.69 & \cellcolor{identityrow}56.07 & \cellcolor{identityrow}39.53
    & \cellcolor{identityrow}13.41 & \cellcolor{identityrow}77.61 & \cellcolor{identityrow}68.98
    & \cellcolor{identityrow}74.96 & \cellcolor{identityrow}42.06 & \cellcolor{identityrow}56.29 \\
 && \cellcolor{saliencyrow}SARQC-GBS(Saliency)
    & \cellcolor{saliencyrow}77.97 & \cellcolor{saliencyrow}56.35 & \cellcolor{saliencyrow}40.27
    & \cellcolor{saliencyrow}14.63 & \cellcolor{saliencyrow}78.99 & \cellcolor{saliencyrow}69.30
    & \cellcolor{saliencyrow}74.87 & \cellcolor{saliencyrow}42.83 & \cellcolor{saliencyrow}56.90 \\
\bottomrule
\end{tabular}
}
\end{table*}

\newpage

\subsection{Speedup}
\label{appendix: speedup}

To evaluate the practical efficiency of different quantization methods, we measure both prefill and decoding speedup on representative dense and MoE language models. All experiments are conducted on a single NVIDIA A100 80GB GPU. We report speedup relative to the corresponding FP16/BF16 baseline under the same setup. Specifically, both prefill and decoding speedup are computed as the throughput of the quantized model divided by that of the FP model, where larger values indicate better efficiency. As shown in \Cref{tab:speed_comparison}, \texttt{SARQC-GS} and \texttt{SARQC-GBS} deliver speedups that are broadly comparable to those of AWQ and GPTQ, with only minor variation across methods. This suggests that SARQC preserves the practical efficiency benefits of weight-only quantization, while improving accuracy without incurring additional runtime overhead, as the equivalent inference efficiency is guaranteed by design.

\begin{table*}[hbtp]
\centering
\caption{Prefill and decoding speedup of quantized models on a sequence length of 2048.}
\label{tab:speed_comparison}
\setlength{\tabcolsep}{4.2pt}
\renewcommand{\arraystretch}{1.12}
\definecolor{sarqcrow}{RGB}{220,236,233}

\resizebox{\textwidth}{!}{%
\begin{tabular}{lcccccccc}
\toprule
\multirow{2}{*}{Method}
& \multicolumn{2}{c}{LLaMA2-7B}
& \multicolumn{2}{c}{LLaMA2-13B}
& \multicolumn{2}{c}{Qwen3-MoE-30B}
& \multicolumn{2}{c}{Mixtral-8x7B} \\
\cmidrule(lr){2-3}
\cmidrule(lr){4-5}
\cmidrule(lr){6-7}
\cmidrule(lr){8-9}
& Prefill & Decoding
& Prefill & Decoding
& Prefill & Decoding
& Prefill & Decoding \\
\midrule
AWQ & 2.11 & 1.92 & 1.86 & 1.73 & 1.47 & 1.25 & 2.09 & 2.01 \\
GPTQ & 1.98 & 1.81 & 1.79 & 1.63 & 1.35 & 1.19 & 1.95 & 1.87 \\
SARQC-GS(Saliency) & 2.12 & 1.95 & 1.88 & 1.71 & 1.45 & 1.23 & 2.05 & 1.94 \\
SARQC-GBS(Saliency) & 2.00 & 1.84 & 1.82 & 1.68 & 1.36 & 1.18 & 1.93 & 1.85 \\
\bottomrule
\end{tabular}%
}
\end{table*}

\end{document}